\documentclass{article}
\usepackage{iclr2027_conference,times}

\usepackage{amsmath,amsfonts,bm}

\def\eqref#1{equation~\ref{#1}}
\def\1{\bm{1}}

\DeclareMathAlphabet{\mathsfit}{\encodingdefault}{\sfdefault}{m}{sl}
\SetMathAlphabet{\mathsfit}{bold}{\encodingdefault}{\sfdefault}{bx}{n}

\newcommand{\pdata}{p_{\rm{data}}}
\newcommand{\E}{\mathbb{E}}

\newcommand{\R}{\mathbb{R}}

\usepackage{amsmath,amssymb,mathtools}
\usepackage{enumitem}
\usepackage{booktabs}
\usepackage{array}
\usepackage{graphicx}
\graphicspath{{figures/}}
\usepackage{tikz}
\usetikzlibrary{positioning,arrows.meta,calc}
\usepackage{float}
\usepackage{titletoc}
\usepackage{hyperref}
\usepackage{url}
\usepackage{xcolor}

\definecolor{dataclr}{HTML}{3e689c}  % true/forward-process operators (muted steel blue)
\definecolor{modelclr}{HTML}{8e3f46}  % learned model operators (muted brick)

\newcommand{\dataC}[1]{{\color{dataclr}#1}}
\makeatletter
\let\ctmc@theta\theta
\newcommand{\modelC}[1]{{\color{modelclr}\let\theta\ctmc@theta#1}}
\makeatother

\makeatletter
\newcommand{\inlinechoice}[2]{\if@display#1\else\mathchoice{#1}{#2}{#1}{#1}\fi}
\makeatother
\newcommand{\colorletter}[3]{\NewDocumentCommand{#1}{e{^_}}{#2{#3\IfValueT{##1}{^{##1}}\IfValueT{##2}{_{##2}}}}}
\newcommand{\cleanq}{\dataC{q_0}}
\renewcommand{\pdata}{p_{\mathrm{data}}}
\colorletter{\q}{\dataC}{q}               % laws and kernels: \q_t, \q_{t\given s}, \q(z_0\given z_t)
\colorletter{\pistar}{\dataC}{\pi^\star}  % oracle denoiser, cavity, and score laws
\colorletter{\mustar}{\dataC}{\mu^\star}
\colorletter{\sstar}{\dataC}{s^\star}
\newcommand{\fQ}[1]{\dataC{Q_{#1}}}                          % forward generator (sequence level)
\newcommand{\fR}[2][]{\dataC{R\ifx\relax#1\relax\else^{#1}\fi_{#2}}}  % forward generator (token level)
\newcommand{\RevQ}[1]{\dataC{\inlinechoice{\widehat{Q}_{#1}}{\smash[t]{\widehat{Q}_{#1}}\vphantom{Q}}}}  % true reverse rate (marginal and clean-conditioned)
\newcommand{\Pstar}[1]{\dataC{\inlinechoice{P^{\star}_{#1}}{P^{\star}_{\smash{#1}}}}}  % noising path law on a window
\RenewDocumentCommand{\theta}{}{{\modelC{\theta}}}
\NewDocumentCommand{\thetastar}{}{{\modelC{\theta_\star}}}
\colorletter{\p}{\modelC}{p}              % model laws: \p^\theta_0, \p^\theta_{0\given t_1}
\colorletter{\pith}{\modelC}{\pi^\theta}  % denoiser, cavity, and score heads
\colorletter{\muth}{\modelC}{\mu^\theta}
\colorletter{\sth}{\modelC}{s^\theta}
\colorletter{\mth}{\modelC}{m^\theta}     % noisy-token law induced by the cavity head
\newcommand{\RevQth}[1]{\modelC{\inlinechoice{\widehat{Q}^{\theta}_{#1}}{\smash[t]{\widehat{Q}^{\theta}_{#1}}\vphantom{Q^{\theta}}}}}  % model reverse rate
\newcommand{\Ptheta}[1]{\modelC{\inlinechoice{P^{\theta}_{#1}}{P^{\theta}_{\smash{#1}}}}}  % model path law on a window

\renewcommand{\P}{\mathbb{P}}
\makeatletter
\newcommand{\given@script}[2]{\mathrel{\raisebox{#2}{\scalebox{1}[0.75]{$\m@th#1\mid$}}}}
\newcommand{\given}{\mathchoice{\mid}{\mid}{\given@script\scriptstyle{0.6pt}}{\given@script\scriptscriptstyle{0.43pt}}}
\makeatother
\providecommand{\one}{\mathbf{1}}
\providecommand{\X}{\mathcal{X}}
\providecommand{\Y}{\mathcal{Y}}
\providecommand{\Zs}{\mathcal{Z}}
\providecommand{\J}{\mathcal{J}}
\DeclareMathOperator{\ELBO}{ELBO}
\DeclareMathOperator{\NELBO}{NELBO}
\DeclareMathOperator{\supp}{supp}
\newcommand{\Law}{\mathrm{Law}}
\newcommand{\mask}{\mathfrak{m}}
\newcommand{\KLdiv}[2]{\mathrm{KL}\!\left(#1\,\middle\|\,#2\right)}
\newcommand{\ISdiv}[2]{D_{\mathrm{IS}}\!\left(#1\,\middle\|\,#2\right)}
\newcommand{\Exp}[2]{\E_{#1}\!\left[#2\right]}
\newcommand{\sboxed}[1]{\boxed{\;#1\;}}
\makeatletter
\DeclareRobustCommand{\eqref}[1]{\textup{\tagform@{\ref{#1}}}}
\makeatother

\usepackage{amsthm}
\theoremstyle{plain}
\newtheorem{theorem}{Theorem}
\newtheorem{lemma}{Lemma}
\newtheorem{proposition}{Proposition}
\newtheorem{corollary}{Corollary}
\newtheorem{assumption}{Assumption}
\newtheorem{definition}{Definition}
\theoremstyle{remark}
\newtheorem{remark}{Remark}
\theoremstyle{plain}
\newtheorem*{restated}{\restatedname}
\newcommand{\restatedname}{}
\newenvironment{thmrestate}[2]{\renewcommand{\restatedname}{#1}\begin{restated}[#2]}{\end{restated}}

\definecolor{insightclr}{HTML}{ad650e}  % amber accent for skimmable takeaways
\definecolor{insightbg}{HTML}{FBF4E4}   % subtle background tint
\newcommand{\insightC}[1]{\textcolor{insightclr}{#1}}
\newcounter{keyinsight}
\newcommand{\keyinsight}[1]{%
  \par\smallskip\noindent
  \refstepcounter{keyinsight}%
  {\setlength{\fboxsep}{6pt}%
    \fcolorbox{insightclr}{insightbg}{%
      \parbox{\dimexpr\linewidth-2\fboxsep-2\fboxrule\relax}{%
        \small\textbf{\insightC{Takeaway \thekeyinsight.}}\ #1}}}%
  \par\smallskip
}

\definecolor{tbdclr}{HTML}{eb06fb}

\author{%
\makebox[\dimexpr\textwidth-2\tabcolsep\relax][c]{\normalsize
Rodrigo Casado Noguerales\textsuperscript{1,2}%
\quad Bernhard Schölkopf\textsuperscript{1,2,3}%
\quad Thomas Hofmann\textsuperscript{1}%
\quad Aran Raoufi\textsuperscript{1}%
}%
}

\DeclareUnicodeCharacter{2033}{''}
\DeclareUnicodeCharacter{FFFD}{?}

\title{What Does a Discrete Diffusion Model Learn?\\
One Reverse Process in Three Coordinates}

\iclrfinalcopy % Uncomment for the camera-ready and arXiv versions.

\begin{document}

\maketitle
\lhead{Preprint} % Uncomment for arXiv.

\ificlrfinal
  \begingroup
  \renewcommand{\thefootnote}{}
  \footnotetext{%
    \textsuperscript{1}Data Analytics Lab, Department of Computer Science, ETH Zurich.
    \textsuperscript{2}Department of Empirical Inference, Max Planck Institute for Intelligent Systems, Tübingen.
    \textsuperscript{3}ELLIS Institute Tübingen.
    Correspondence to: Rodrigo Casado Noguerales
    \href{mailto:rcasado@ethz.ch}{\texttt{rcasado@ethz.ch}}.}
  \endgroup
\fi

\vspace{-10pt}
\begin{abstract}
  What does a discrete diffusion model actually learn?
  The literature trains denoisers, bridge plug-ins, and concrete scores, each deriving its own ELBO objective under different conventions.
  This fragmentation makes the object being learned unclear and ELBO values difficult to interpret and compare.
  All three methods learn the same object in different coordinates: the exact time-reversal of the noising process.
  We first prove the Oracle Distance identity: for any noising process, the negative ELBO is not just a bound on the likelihood; it equals the data entropy plus the path KL divergence from the reversed noising process to the learned one.
  Next, for token-factorizing noise, we show that denoiser, bridge plug-in, and concrete score are three coordinates of this reverse process, derive exact conversions between them, and identify the probability distribution learned by each.
  In particular, contrary to its standard interpretation, the bridge plug-in learns not a denoiser but the cavity distribution: the clean token given the noisy sequence excluding its own position.
  We confirm the theory by training all three coordinates with masked, uniform, and GIDD kernels on OpenWebText under a shared recipe, and by training uniform-diffusion solvers on four logic tasks.
  The learned coordinates agree after conversion, while confusing the cavity with the denoiser can severely degrade sampling quality.
  Together with exact boundary-term formulas, closed-form NELBO calibrations at initialization, and a clock interpretation of time schedules, these results provide a conceptual framework and a mechanical recipe for deriving, implementing, and comparing discrete diffusion processes.  % \footnote{Code will be released upon publication.}
\end{abstract}

\begin{figure}[H]
  \centering
  \input{figures/hero.tex}
  \caption{The true reverse rate in three coordinates, each learned at the ELBO optimum (\S\ref{sec:coordinates}).
    The denoiser keeps position~$i$ of $z_t$, the cavity leaves it out, and the score's numerator swaps in~$y^i$.}
  \label{fig:hero}
\end{figure}

\section{Introduction}\label{sec:intro}

Discrete diffusion provides a flexible route to model categorical data: slowly corrupt each position of a sequence, then learn to run the corruption backward \citep{hoogeboom2021argmax,austin2021d3pm}.
This unlocks a new test-time axis on which to scale computation, and any-order generation enables token-parallel decoding, self-correction, and constrained generation for language and reasoning tasks \citep{sahoo2024mdlm,schiff2025udlm,gong2025diffullama}, capabilities that have been tested at the scale of large language models \citep{nie2025llada,ye2025dream}.
Mathematically, the corruption and denoising processes are modeled as continuous-time Markov chains (CTMCs) \citep{campbell2022continuous}: jump processes determined infinitesimally by their jump rates, that is, the rates of instantaneous change between states.
Compared to discrete-time formulations, CTMCs are mathematically simpler and can be sampled using arbitrary time grids that are not fixed in advance.

Discrete diffusion modeling therefore requires mapping the model's output to the denoising process through its jump rates.
The literature does so in one of three seemingly different ways: the model can be a denoiser, a bridge plug-in, or a concrete score \citep{sahoo2024mdlm,schiff2025udlm,lou2024sedd}.
Each method then derives its own likelihood bound (ELBO) objective and sampling formulas, following different conventions about the boundary terms and interpreting the model differently.
This is error-prone and makes the ELBO hard to interpret and compare.
A better ELBO value may reflect a better generative model or a tighter variational bound, and the standard convention of dropping the boundary terms can silently invalidate ELBO comparisons.
In short, a central question remains open:
{\setlength{\topsep}{0pt}%
  \begin{center}
    \emph{What does a discrete diffusion model actually learn?}
  \end{center}
}

We answer this question in full generality: a discrete diffusion model learns the exact reversal of the noising process.
For token-factorizing noise, each of the three parameterizations of the literature is a different coordinate of this same process.
Our contributions are as follows:

\textbf{(i) The Oracle Distance.} For any noising process, the negative ELBO (NELBO) equals the data entropy plus the KL divergence from the oracle reverse process, the exact time-reversal of the noising process, to the learned one (Theorem~\ref{thm:oracle}).
The NELBO is therefore a KL between two distributions over \emph{entire trajectories}, a ``path'' KL.
The unique jump rates that minimize it among rates depending only on the noisy sequence, not on the clean data, are a projection: the clean-conditioned reverse jump rates projected onto the noisy information, analogous to flow matching in continuous diffusion.

\textbf{(ii) Three coordinates of the reverse process.} For token-factorizing noise, the reverse process has three coordinates, denoiser, plug-in, and score, with exact conversions between them (Theorem~\ref{thm:coords}).
Combined, they recover the ELBO and sampling formulas of the literature and identify the distribution each method learns.
Contrary to its standard description, the bridge plug-in does not train a denoiser but the \emph{cavity} distribution: the clean-token distribution given the noisy sequence excluding its own position.

\textbf{(iii) Empirical validation.} We verify the theory by training all three coordinates with masked, uniform, and GIDD (generalized interpolating discrete diffusion; \citealp{vonruette2025gidd}) kernels on OpenWebText under one shared recipe, and by training uniform-diffusion solvers on four logic tasks (\S\ref{sec:experiments}): the three coordinates learn the same object, ELBO training curves start at the predicted value, and confusing the cavity with the denoiser can severely degrade sampling quality.

Together with formulas for the boundary terms and an interpretation of time schedules and training weights as a curriculum over noise levels, these results give a conceptual framework and a mechanical recipe for deriving and comparing discrete diffusion processes.

\section{Continuous-time discrete diffusion}\label{sec:setup}

A diffusion model defines a forward noising process $(Z_t)_{t\in[0,T]}$ on a finite space $\X$ whose marginals $\q_t$, the distribution of $Z_t$, interpolate between the ``clean''-data distribution $\cleanq=\pdata$ and a terminal distribution $\q_T$ from which sampling is easy, and then learns to run the process backward.\footnote{Throughout, we color quantities related to the true process in \dataC{blue} and those to the learned model in \modelC{red}.}
In discrete diffusion the process corrupts data directly on its alphabet.
This is described by a continuous-time Markov chain (CTMC), which is determined by its infinitesimal generator or jump rate $\fQ{t}$, such that
\begin{equation*}
  \P\!\left(Z_{t+h}=y\given Z_t=x\right) = \one\{x=y\} + h\,\fQ{t}(x,y)+o(h),
  \qquad \fQ{t}(x,x)=-\sum_{y\ne x}\fQ{t}(x,y),
\end{equation*}
with $\fQ{t}(x,y)\ge0$ for $y\neq x$.
Continuous time removes the need for a predefined sampling schedule, lets training times be drawn continuously, and admits a simpler formulation.
We write $\q_{t\given s}(y\given x)=\P\!\left(Z_{t}=y\given Z_s=x\right)$ for the forward kernel, and similarly for other conditional distributions such as the reverse bridge $\q_{s\given t,0}$.
Whenever $\q_t(x)>0$, the time reversal of a CTMC is again a CTMC.
Its rate, which we call the \emph{reverse} rate, and its clean-conditioned version obtained by conditioning on $Z_0=z_0$, are
\begin{equation}
  \RevQ{t}(z_t,y)=\fQ{t}(y,z_t)\,\frac{\q_t(y)}{\q_t(z_t)},
  \qquad
  \RevQ{t}(z_t,y\given z_0)=\fQ{t}(y,z_t)\,\frac{\q_{t\given0}(y\given z_0)}{\q_{t\given0}(z_t\given z_0)},
  \qquad y\ne z_t.
  \label{eq:reverse-rates}
\end{equation}
The rate $\RevQ{t}(z_t,y)$ is the target of a denoising model, which has no access to clean data at generation time, while the clean-conditioned rate is the tractable training target.
Implementations typically truncate the interval to $[t_1,t_2]\subset(0,T)$, since the clean-conditioned ratios are singular or high-variance near the endpoints.
A denoising model thus consists of model reverse rates $\RevQth{t}$ matching $\RevQ{t}$, a reconstruction decoder $\p^\theta_{0\given t_1}$ aiming at $\q_{0\given t_1}$, and a terminal prior $\p^\theta_{t_2}$ from which generation starts.
Likelihood-based training then supervises the model rates via the continuous-time ELBO~\citep{campbell2022continuous,benton2024denoising}, which relies on the local \emph{rate divergence}:
\begin{equation}
  \Phi(a,b):=a\log\frac ab-a+b,\qquad a,b\ge0,
  \label{eq:phi}
\end{equation}
with $\Phi(0,b)=b$ and $\Phi(a,0)=+\infty$ for $a>0$: the Bregman divergence of $F(u)=u\log u-u$ and the KL rate between two Poisson jump intensities.
\begin{proposition}[Finite-window CTMC ELBO]
  \label{prop:elbo}
  Let $[t_1,t_2]\subset(0,T)$ and assume the conditions of Theorem~\ref{thm:elbo-full}.
  Then $\log \p^\theta_0(z_0)\ge\ELBO_{[t_1,t_2]}(\theta;z_0)$, where
  \begin{equation}
    \begin{aligned}
      \ELBO_{[t_1,t_2]}(\theta;z_0)
       & = \Exp{\q_{t_1\given0}(z_{t_1}\given z_0)}{\log \p^\theta_{0\given t_1}(z_0\given z_{t_1})}
      - \KLdiv{\q_{t_2\given0}(\cdot\given z_0)}{\p^\theta_{t_2}}
      \\[1pt]
       & \quad - \int_{t_1}^{t_2}\Exp{\q_{t\given0}(z_t\given z_0)}
                                 {\textstyle\sum_{y\ne z_t}\Phi\big(\RevQ{t}(z_t,y\given z_0),\RevQth{t}(z_t,y)\big)}\,dt.
    \end{aligned}
    \label{eq:elbo}
  \end{equation}
\end{proposition}
App.~\ref{app:elbo} gives the Monte Carlo estimator of the path integral, and App.~\ref{app:processes} the boundary terms of each kernel.
Models are trained by gradient descent on the negative ELBO, $\NELBO_{[t_1,t_2]}(\theta;z_0)$, and we write $\ELBO_{[t_1,t_2]}(\theta):=\E_{\cleanq}[\ELBO_{[t_1,t_2]}(\theta;z_0)]$ and $\NELBO_{[t_1,t_2]}(\theta)$ for the data averages.

\keyinsight{The ELBO of any noising process follows a common recipe: (i) choose the corruption kernel $\q_{t\given s}$, (ii) differentiate it to obtain the forward rates $\fQ{t}$, (iii) form the clean-conditioned reverse rates ${\RevQ{t}(\cdot\given z_0)}$ as the quotient in~\eqref{eq:reverse-rates}, (iv) map the model's output to the reverse rates $\RevQth{t}$, (v) sum ${\Phi(\RevQ{t}(\cdot\given z_0),\RevQth{t})}$ over the possible jumps $z_t\to y$, (vi) integrate in $t,\q_{t\given0}$, (vii) add the boundary terms.}

The boundary terms at $t_1$ and $t_2$ are usually discarded as model-independent constants, even though only the prior term is independent of the model, and dropping either hides the identities of \S\ref{sec:oracle} and makes reported ELBOs incomparable across processes.
In the standard masked and uniform setups and the usual windows, both boundary terms are small (Table~\ref{tab:boundary}, App.~\ref{app:metric-details}), but nothing guarantees this in general: endpoint divergences can concentrate much of the objective in the clipped region.

\section{The Oracle Distance identity}\label{sec:oracle}

Averaging \eqref{eq:elbo} over $\cleanq$ gives $\KLdiv{\cleanq}{\p^\theta_0} \le \NELBO_{[t_1,t_2]}(\theta)-H(\cleanq)$, with $H(\cleanq)$ the data entropy, further justifying NELBO minimization.
However, a likelihood bound is hard to interpret and leaves open whether a better value reflects a better generative distribution or a tighter bound, which matters where direct sample metrics are hard to define, such as for text and music.
A more satisfying interpretation requires looking not at marginals or jump rates but at the \emph{path distribution}, that is, the distribution over entire trajectories.
The two path distributions involved are the noising distribution $\Pstar{[0,T]}$, read in reverse time\footnote{A CTMC path distribution is equivalently characterized by reading it in either direction, that is, either from its starting distribution and forward rates, or its terminal distribution and backward rates; see App.~\ref{app:primer}.} as started at $\q_T$ and run backward with the marginal reverse rates $\RevQ{t}$, and the model distribution $\Ptheta{[0,T]}$, started at the terminal prior and run backward with $\RevQth{t}$.
Throughout this section we assume the regularity of Assumption~\ref{ass:regularity-fixed-support} and the conditions of Theorem~\ref{thm:oracle-window}, with its two boundary mismatch terms vanishing as $t_1\downarrow0$ and $t_2\uparrow T$.

\begin{theorem}[Oracle Distance]\label{thm:oracle}
  For any noising process and model,
  \begingroup\thickmuskip=4mu\setlength{\fboxsep}{2pt}
  \begin{equation}
    \scalebox{0.96}{$\displaystyle\boxed{
          \,\NELBO_{[0,T]}(\theta)-H(\cleanq)
          = \KLdiv{\Pstar{[0,T]}}{\Ptheta{[0,T]}}
          = \int_{0}^{T}\Exp{\q_t}{\sum_{y\ne z_t}\Phi(\RevQ{t}(z_t,y),\RevQth{t}(z_t,y))}dt.
        }\,$}
    \label{eq:oracle-distance}
  \end{equation}
  \endgroup
\end{theorem}
On a clipped time window the identity holds with the corresponding boundary mismatch terms, similar to those of Proposition~\ref{prop:elbo}.
Theorem~\ref{thm:oracle-window} gives the full statement and proves it in two alternative ways.
Since $\Phi$ is a Bregman divergence, the classical Pythagorean identity (Lemma~\ref{lem:bregman-pythagoras}) gives the unique minimizer of the NELBO over $Z_t$-measurable rates as a projection:

\begin{corollary}[Reverse-rate projection]\label{cor:projection}
  Among rates measurable in the noisy state $Z_t$ alone, $\NELBO_{[t_1,t_2]}(\theta)$ is minimized, uniquely up to null sets, by
  \begin{equation}
    \modelC{\widehat{Q}{}^{\theta_\star}_t}(z_t,y)
    = \Exp{}{\RevQ{t}(Z_t,y\given Z_0)\;\middle|\;Z_t=z_t}
    = \RevQ{t}(z_t,y),
    \qquad y\ne z_t.
    \label{eq:projection}
  \end{equation}
\end{corollary}
Thus, training on the clean-conditioned rates without access to the clean data amounts to a projection onto the noisy information available at time $t$, analogous to flow matching in continuous diffusion.

Denote by $I(Z_0;Z_t)=H(Z_0)-H(Z_0\given Z_t)$ the mutual information between clean and noisy data, that is, the information about $Z_0$ that survives in $Z_t$, with $H(Z_0\given Z_t)$ the conditional entropy.
Even at the oracle, the path term of the ELBO does not vanish: it equals the information the forward process destroys, in a CTMC analogue of the I-MMSE relation \citep{guo2005mmse}:
\begin{proposition}[Irreducible cost is information loss]\label{prop:infoloss}
  The oracle integrand satisfies
  \begin{equation}
    \J^\star_t=\frac{d}{dt}H(Z_0\given Z_t)=-\frac{d}{dt}I(Z_0;Z_t)\;\ge\;0,
    \qquad
    \int_{t_1}^{t_2}\J^\star_t\,dt = I(Z_0;Z_{t_1})-I(Z_0;Z_{t_2}),
    \label{eq:infoloss}
  \end{equation}
  where $\J^\star_t:=\E_{Z_0,Z_t}\big[\sum_{y\neq Z_t}\Phi(\RevQ{t}(Z_t,y\given Z_0),\RevQ{t}(Z_t,y))\big]$.
\end{proposition}
The Pythagorean identity for Bregman divergences then extends this result to arbitrary model rates:
\begin{equation}
  \Exp{Z_0,Z_t}{\sum_{y\ne Z_t}\Phi\big(\RevQ{t}(Z_t,y\given Z_0),\RevQth{t}(Z_t,y)\big)}
  = \J^\star_t
  + \Exp{Z_t}{\sum_{y\ne Z_t}\Phi\big(\RevQ{t}(Z_t,y),\RevQth{t}(Z_t,y)\big)}.
  \label{eq:pythagoras}
\end{equation}
Integrated over time, $\J^\star_t$ telescopes into conditional entropies by \eqref{eq:infoloss} and the mismatch part accumulates into the path KL, thus proving Theorem~\ref{thm:oracle}.

Jensen's and Gibbs' inequalities give $\NELBO_{[0,T]}(\theta)\ge\E_{\cleanq}[-\log \p^\theta_0(z_0)]\ge H(\cleanq)$, and this floor is always attainable at the oracle reverse rates:
\begin{corollary}[Universal attainable NELBO floor]\label{cor:floor}
  Suppose some $\thetastar$ realizes $\modelC{\widehat{Q}^{\theta_\star}_t}=\RevQ{t}$ on $(0,T)$.
  Then $\p^{\theta_\star}_0=\cleanq$, the ELBO is tight: $\log \p^{\theta_\star}_0(z_0)=\ELBO_{[0,T]}(\thetastar;z_0)$, and
  \begin{equation}
    \min_\theta\NELBO_{[0,T]}(\theta)=H(\cleanq).
    \label{eq:floor}
  \end{equation}
\end{corollary}
Every noising process thus shares the same attainable NELBO floor given enough capacity: the data entropy.
As a consequence, the observed gap between masked and uniform diffusion NELBOs is not intrinsic to the noising processes but instead comes from finite modeling capacity and trainability.

\keyinsight{The NELBO equals the data entropy plus the path KL from the oracle reverse process to the learned one, and its unique minimizer is the projection of the clean-conditioned reverse rate onto the information present in the noisy state.
  This explains the looseness of the ELBO: it measures the KL divergence from one specific distribution over reverse trajectories, that of the oracle reverse process, and in doing so penalizes deviations even among trajectories whose clean endpoint also matches the data distribution.}

\paragraph{Time schedules.}
Importance-sampling the training times from a density $p_{\mathrm{IS}}(t)$ amounts to running the process at a different speed, going slower where $p_{\mathrm{IS}}$ is high, and the ELBO is the same integral at any speed, by~\eqref{eq:clock-invariance} in App.~\ref{app:clocks}.
Thus, importance sampling and time schedules leave the ELBO unchanged and only set the resolution of training, that is, how densely each noise level is sampled.
Training weights other than the ELBO's do change the objective \citep{kingma2023understanding,shi2025demystifying}, but the reverse-rate projection of Corollary~\ref{cor:projection} is solved separately at each time, so the reweighted objective keeps the NELBO minimizer.
Schedules, importance densities, and training weights can thus all be seen as a choice of curriculum over noise levels, and one would focus it where the learning signal is highest \citep{hong2026understanding}.
Generation benefits from different properties of the time schedule, hence the common use of different schedules for training and sampling (linear against cosine, for instance).

\keyinsight{Importance sampling and time schedules change the speed of the process but not the ELBO; other training weights change the objective but not its per-time optimum. Thus, all of them can be seen as a choice of curriculum over noise levels.}

\section{Three coordinates for token-factorizing discrete diffusion}\label{sec:coordinates}

Sequential data are elements of a product space $\X=\X_1\times\cdots\times\X_L$.
The standard noising processes corrupt positions independently, that is, the forward noising kernel factors over positions: ${\q_{t\given s}(y\given x)}={\prod_{i}\q^i_{t\given s}(y^i\given x^i)}$, where we write $x^i$ for the $i$-th token of a sequence $x$ and $x^{-i}$ for the rest of the sequence excluding $x^i$.
The jump rates then decompose as a sum of per-position rates, $\fQ{t}=\sum_{i}(\fQ{t})_i$, each position running an independent process with rate $\fR[i]{t}$ on $\X_i$, so that $(\fQ{t})_i(x,y)=\one\{x^{-i}=y^{-i}\}\,\fR[i]{t}(x^i,y^i)$ (Theorem~\ref{thm:product-ctmc}).
Nothing forces the per-position rates $\fR[i]{t}$ to be equal, though they usually are.
The standard kernels corrupt each token with probability $\beta_t$, rising from $0$ to $1$ over $[0,T]$, and replace it by a mask state $\mask$ (masked diffusion), by a token drawn uniformly from the vocabulary of size $V$ (uniform diffusion), or by a draw from their mixture with weight $\lambda$ (GIDD; App.~\ref{app:processes}).

Evaluating the NELBO minimizer from~\eqref{eq:projection} on a one-token jump $z_t\to y$ where position $i$ jumps $z_t^i\to y^i$, $z_t^i\ne y^i$, while the rest remains unchanged, $y^{-i}=z_t^{-i}$, gives per-position reverse rates
\begin{equation}
  (\modelC{\widehat{Q}{}^{\theta_\star}_t})_i(z_t,y)
  = (\RevQ{t})_i(z_t,y)
  = \fR[i]{t}(y^i,z_t^i)\,
  \Exp{\q(z_0^i\given z_t)}{\frac{\q^i_{t\given0}(y^i\given z_0^i)}{\q^i_{t\given0}(z_t^i\given z_0^i)}}.
  \label{eq:token-rates}
\end{equation}
Importantly, the clean data appears only through the $i$-th token marginal of the posterior distribution, $\q(z_0^i\given z_t)$.
This is why token-level heads suffice, and why each coordinate below fully specifies the reverse CTMC: substituting it fixes the rates, which determine the process by Kolmogorov uniqueness.
The posterior average in~\eqref{eq:token-rates} can be expressed through the following three quantities:
\begin{equation}
  \pistar_i(z_0^i\given z_t):=\q(z_0^i\given z_t),
  \;\;
  \mustar_i(z_0^i\given z_t^{-i}):=\q(z_0^i\given z_t^{-i}),
  \;\;
  \sstar_i(z_t^i,y^i\given z_t^{-i}):=\frac{\q_t(y^i,z_t^{-i})}{\q_t(z_t)},
  \label{eq:three-laws}
\end{equation}
which we respectively call: \emph{denoiser}, the usual clean-token predictor given the noisy sequence; \emph{cavity}, the distribution of the clean token given the noisy context but excluding its own noisy position; and \emph{score}, a ratio of time-$t$ marginals, where $(y^i,z_t^{-i})$ is $z_t$ with its $i$-th token replaced by $y^i$.
The denoiser is the clean-token predictor of the masked diffusion line of work (\citealp{campbell2022continuous}; MDLM, \citealp{sahoo2024mdlm}; MD4, \citealp{shi2024simplified}; LLaDA, \citealp{nie2025llada}), the cavity is the distribution that the bridge plug-in optimizes in uniform and GIDD diffusion (UDLM, \citealp{schiff2025udlm}; GIDD, \citealp{vonruette2025gidd}; Duo, \citealp{sahoo2025diffusionduality}), and the score is the model target of score methods (\citealp{meng2022concrete}; SEDD, \citealp{lou2024sedd}; RADD, \citealp{ou2025radd}; TCSM, \citealp{zhang2025tcsm}) (Figure~\ref{fig:hero}).
Crucially, prior work describes the bridge plug-in as learning a denoiser, but it is in fact learning a cavity distribution.
In short:
\begin{theorem}[Three coordinates]\label{thm:coords}
  For a token-factorizing noising process, the NELBO is uniquely minimized by the reverse rate of Corollary~\ref{cor:projection}, and on a one-token jump $z_t\to y=(y^i,z_t^{-i})$ the posterior average in \eqref{eq:token-rates} has the three equal forms
  \begin{equation}
    \sboxed{
      \Exp{\pistar_i}
      {\frac{\q^i_{t\given0}(y^i\given z_0^i)}{\q^i_{t\given0}(z_t^i\given z_0^i)}}
      = \frac{\Exp{\mustar_i}{\q^i_{t\given0}(y^i\given z_0^i)}}
      {\Exp{\mustar_i}{\q^i_{t\given0}(z_t^i\given z_0^i)}}
      = \sstar_i(z_t^i,y^i\given z_t^{-i}).
    }
    \label{eq:coords}
  \end{equation}
\end{theorem}
Substituting a network head $\pith_i$, $\muth_i$, or $\sth_i$ for its distribution in~\eqref{eq:coords} yields three head-to-rate maps which, with Proposition~\ref{prop:elbo}, give the three standard objectives: denoiser, bridge plug-in, and concrete score, each optimized at its oracle distribution from~\eqref{eq:three-laws}, uniquely up to the identifiability of the map.
Networks are trained with access to the full sequence $z_t$, but we keep the conditioning notation of the distribution each head is trained toward, e.g.\ $z_t^{-i}$ for the cavity, and test empirically in \S\ref{subsec:agreement} whether the network learns to ignore $z_t^i$.

\keyinsight{Denoiser, cavity, and score are three coordinates of the same object, the oracle reverse rate of any token-factorizing noising process.
  They subsume the published CTMC training objectives, each with an exact distribution that minimizes the NELBO, unique up to the identifiability of the head-to-rate map.
  In particular the bridge plug-in trains a cavity distribution rather than a denoiser.}

\begin{proposition}[Coordinate conversions]\label{prop:conversions}
  Fix $(i,z_t,t)$ and write the head-induced distribution of the noisy token: $\mth_i(y\given z_t^{-i}):=\sum_{z_0^i}\muth_i(z_0^i\given z_t^{-i})\,\q^i_{t\given0}(y\given z_0^i)$.
  Then
  \begin{gather}
    \pith_i(z_0^i\given z_t)\propto\muth_i(z_0^i\given z_t^{-i})\,\q^i_{t\given0}(z_t^i\given z_0^i);
    \qquad
    \muth_i(z_0^i\given z_t^{-i})\propto\frac{\pith_i(z_0^i\given z_t)}{\q^i_{t\given0}(z_t^i\given z_0^i)};
    \label{eq:bayes-tie}
    \\
    \sth_i(z_t^i,y^i\given z_t^{-i})=\frac{\mth_i(y^i\given z_t^{-i})}{\mth_i(z_t^i\given z_t^{-i})};
    \qquad
    \mth_i=\muth_i\,\q^i_{t\given0},\ \text{with}\ \mth_i(\cdot\given z_t^{-i})\propto\sth_i(z_t^i,\cdot\given z_t^{-i}).
    \label{eq:score-conversions}
  \end{gather}
\end{proposition}
The last linear system, required to go from score to cavity, has a solution if and only if $\mth_i$ lies in the image of $\q^i_{t\given0}$ arranged as a row-stochastic matrix and a unique one if and only if $\q^i_{t\given0}$ has full row rank, while the solution $\muth_i$ is a probability vector if and only if it is nonnegative.

The cavity-to-denoiser conversion in~\eqref{eq:bayes-tie} is a \emph{Bayes reweighting}: the forward kernel at the observed token, which we call the \emph{Bayes factor}, is the evidence the local observation adds on top of the context.
It thus splits the information that $z_t$ carries about $z_0^i$ into two parts: what the context $z_t^{-i}$ reveals, the cavity, which must be learned, and what $z_t^i$ adds, given analytically by the kernel.
If however the noisy token carries no information about the original token, the distinction disappears (see App.~\ref{app:processes} for further details on masked, uniform, and GIDD diffusion):

\begin{corollary}[Denoiser equals cavity in masked diffusion]\label{cor:masked}
  If the Bayes factor $\q^i_{t\given0}(z_t^i\given z_0^i)$ is constant in $z_0^i$ on its support, then the Bayes reweighting is restriction to that support followed by renormalization.
  In masked diffusion, $\q^i_{t\given0}(\mask\given z_0^i)=\beta_t$ at a masked position for every $z_0^i$, so $\pistar_i(\cdot\given z_t)=\mustar_i(\cdot\given z_t^{-i})$.
\end{corollary}

\paragraph{Ancestral sampling.}
The finite-step sampling formulas share the same token-factorized ansatz.
In the denoiser coordinate, substituting $\pith_i$ into the exact backward jump distribution from $z_t$ to $z_s$, $s<t$, gives the standard ancestral formula
\begingroup\thickmuskip=4mu
\begin{equation}
  \scalebox{0.95}{$
      \displaystyle
      \p^\theta_{s\given t}(z_s\given z_t)=\prod_{i=1}^L \p^\theta(z_s^i\given z_t),
      \quad
      \p^\theta(Z_s^i{=}y\given z_t)
      =\q^i_{t\given s}(z_t^i\given y)\sum_{z_0^i} \pith_i(z_0^i\given z_t)\,
      \frac{\q^i_{s\given0}(y\given z_0^i)}{\q^i_{t\given0}(z_t^i\given z_0^i)}.$}
  \label{eq:ancestral}
\end{equation}
\endgroup
Substituting the cavity-to-denoiser conversion from~\eqref{eq:bayes-tie} gives the usual bridge plug-in sampler (App.~\ref{app:sequences}).

At the oracle, each factor of~\eqref{eq:ancestral} is the exact marginal $\q(z_s^i\given z_t)$, but the positional marginals give no access to the joint $\q(z_s\given z_t)$, and their product ignores the correlations between sequence positions.
This is the error committed during discrete-time ancestral sampling: a factorization error, from sampling the marginals independently, present even at the ELBO oracle.
Equivalently, the product of marginals is exact for the CTMC process in which, during the step from $t$ to $s$, each position's denoiser keeps conditioning on the other positions as they were at $t$, so the error comes from not updating the denoiser within the step.
This is why masked diffusion, whose denoiser changes only when a token is unmasked, incurs no such error when unmasking tokens one at a time.

We conclude the section with useful NELBO calibration identities for uninformative heads:
\begin{proposition}[Calibration at initialization]\label{prop:calibration}
  Fix $[t_1,t_2]\subset(0,T)$ and the uniform-forward boundary distributions of App.~\ref{app:sequences}.
  Then, per token,
  \begin{itemize}[leftmargin=1.4em,topsep=2pt,itemsep=1pt]
    \item a cavity head $\muth_i\equiv1/V$ has $\NELBO_{[t_1,t_2]}=\log V$, for every token-factorizing process (and a score head at zero logits under the folded parameterization of App.~\ref{app:processes} gives the same value);
    \item a denoiser head $\pith_i\equiv1/V$ under uniform diffusion has $\NELBO_{[t_1,t_2]}=\tfrac{V-1}{V}\log\tfrac1{\beta_{t_1}}+O(1)$.
  \end{itemize}
\end{proposition}
These identities are a useful test for NELBO implementations, because they simultaneously check boundary terms, importance sampling schedules, and the head-to-rate map.

\keyinsight{Denoiser and cavity coincide where the Bayes factor is constant, as in masked diffusion.
  Under general noise, the denoiser must also learn the analytic Bayes factor, and its NELBO diverges as $\log(1/\beta_{t_1})$ at initialization, delaying its training (\S\ref{sec:experiments}). Cavity and score instead have a per-token NELBO of exactly $\log V$ at initialization with both boundary terms, providing an implementation check.}

\section{Experiments}\label{sec:experiments}

For our text experiments, we train a grid of $\{$masked, GIDD, uniform$\}\times\{$denoiser, cavity, score$\}$ on OpenWebText under a shared recipe: a time-conditioned diffusion transformer \citep{sahoo2024mdlm} at GPT-2 Small size trained on the path term of the NELBO of Proposition~\ref{prop:elbo} (full experimental details in App.~\ref{app:details}).
Each configuration is trained with a single seed, which suffices for our conclusions.
Generation uses the ancestral sampler of~\eqref{eq:ancestral} at 256 network function evaluations (NFEs) unless otherwise stated.
Since sampling first converts to the denoiser, temperature is applied after the conversion, so that it tempers the contextual evidence (the cavity) and the local evidence (the Bayes factor) alike.
We use GPT-2 Large to score samples via generative perplexity (GenPPL) and MAUVE~\citep{pillutla2021mauve}; all reported NELBOs include both boundary terms.

\subsection{Reading heads in the wrong coordinate}\label{subsec:misread}

To assess the effect of omitting the conversion, we read each of the cavity and denoiser heads as the other, which drops the Bayes reweighting of Proposition~\ref{prop:conversions} or applies it twice.
The Bayes reweighting biases the distribution toward the observed noisy token, and omitting it when reading a cavity as a denoiser raises the NELBO by about $7$ bits/token and collapses sample quality on every non-masked kernel (Table~\ref{tab:misread}), while applying it twice by reading a denoiser as a cavity is less severe (Table~\ref{tab:misread-full}, App.~\ref{app:further}).
Along the quality--entropy frontier traced by the NFE sweep, every misread head of a non-masked kernel is dominated by its native counterpart (Figure~\ref{fig:frontier}, App.~\ref{app:further}).

\begin{table}[h]
  \caption{Each cavity head on OpenWebText, natively and read as a denoiser (NELBO in bits/token).}
  \label{tab:misread}
  \centering\small
  \setlength{\tabcolsep}{4.5pt}
  \begin{tabular}{@{}l cc cc cc cc@{}}
    \toprule
          & \multicolumn{2}{c}{Masked} & \multicolumn{2}{c}{GIDD $\lambda{=}0.1$} & \multicolumn{2}{c}{GIDD $\lambda{=}0.2$} & \multicolumn{2}{c}{Uniform}                                       \\
    \cmidrule(lr){2-3}\cmidrule(lr){4-5}\cmidrule(lr){6-7}\cmidrule(l){8-9}
          & Native                     & Misread                                  & Native                                   & Misread                     & Native & Misread & Native & Misread \\
    \midrule
    NELBO & 4.900                      & =native                                  & 4.990                                    & 12.416                      & 5.017  & 12.289  & 5.070  & 11.788  \\
    MAUVE & 0.869                      & =native                                  & 0.884                                    & 0.008                       & 0.896  & 0.007   & 0.873  & 0.011   \\
    \bottomrule
  \end{tabular}
\end{table}

The distinction is also relevant on four logic tasks, Sudoku, Boxoban, maze solving, and graph coloring, on which we train uniform-diffusion cavity models with three seeds each (App.~\ref{app:puzzles}): misread as denoisers, that is, without the Bayes factor, their test NELBO rises on every task (Table~\ref{tab:puzzles-nelbo}).
The increase is smaller than on text because here the context alone nearly determines each token, making the Bayes factor less important.

\begin{table}[h]
  \caption{Test NELBO ($10^{-3}$ bits/token) of the cavity models, mean (standard error) over seeds.}
  \label{tab:puzzles-nelbo}
  \centering\small
  \setlength{\tabcolsep}{8pt}
  \begin{tabular}{@{}l cccc@{}}
    \toprule
            & Sudoku         & Boxoban        & Maze           & Graph coloring  \\
    \midrule
    Native  & $23.3$ ($1.1$) & $32.9$ ($0.3$) & $18.5$ ($0.2$) & $123.7$ ($0.1$) \\
    Misread & $24.4$ ($1.2$) & $42.2$ ($0.4$) & $20.8$ ($0.2$) & $153.2$ ($0.1$) \\
    \bottomrule
  \end{tabular}
\end{table}

Solve rate, however, is not what the ELBO measures, and there the effect of the conversion depends on both the sampler and the task (Figure~\ref{fig:puzzles}).
The Bayes factor commits each position to its current token.
The \emph{re-corrupt} sampler, which at every step predicts a fully clean sequence and noises it back to the next noise level (App.~\ref{app:samplers}), otherwise redraws every position from its context alone: without the factor, its solve rates collapse on graph coloring and Boxoban, where color permutations and interchangeable moves leave the context undecided between many solutions.
The ancestral sampler of~\eqref{eq:ancestral} retains a commitment through its reverse bridge, which favors the current token, and its gap closes as the NFE grows.
The gap is small on maze, where each move is pinned by its neighbors along the path, and reverses on Sudoku, plausibly because its constraints let the cavity determine the unique solution from the context alone, so that committing to an imperfect current guess hurts slightly more than it helps; consistently, the relative NELBO increase is smallest on these two tasks.
Samplers can thus respond very differently to what the model has actually learned, pointing toward a promising research direction in sampler design.
\begin{figure}[h]
  \centering
  \includegraphics[width=1.0\linewidth]{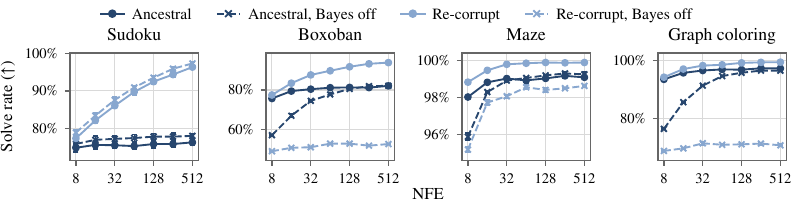}
  \caption{Single-attempt solve rate against NFE on the logic tasks, for the uniform-diffusion cavity models under the ancestral and re-corrupt samplers, with the Bayes reweighting and without it.}
  \label{fig:puzzles}
\end{figure}

\subsection{Head agreement and own-token sensitivity}\label{subsec:agreement}

Table~\ref{tab:agreement} compares trained heads through the Jensen--Shannon divergence (App.~\ref{app:metric-details}): across coordinates, and before and after their own noisy token is redrawn.
After conversion to the denoiser, each pair of independently trained heads agrees to within $0.05$ nats/token, while a cavity read as a denoiser differs from its own conversion $3.5$ to $7$ times more, so the coordinates indeed learn the same object; under masked diffusion, the two coincide (Corollary~\ref{cor:masked}).
Since the standard transformer exposes each position's own noisy token, nothing forces a cavity head to ignore it, yet trained cavity heads depend on it far less: they move a quarter to a third as much as the denoisers when the token is redrawn.
The Bayes reweighting accounts for most of this difference, as applying it to the cavity reproduces the denoiser's own-token dependence, while removing it from the denoiser leaves about twice the cavity's.
Figures~\ref{fig:agreement} and~\ref{fig:sensitivity} (App.~\ref{app:further}) resolve both statistics in diffusion time.

\begin{table}[h]
  \caption{Jensen--Shannon divergence (nats/token) between trained heads. Agreement: pairs of heads converted to the denoiser, and a cavity compared to its own denoiser conversion. Own-token sensitivity: a head before and after its noisy token is redrawn, in its own coordinate or converted ($\to$).}
  \label{tab:agreement}
  \centering\footnotesize
  \setlength{\tabcolsep}{2.85pt}
  \begin{tabular}{@{}l cccc cccc@{}}
    \toprule
                         & \multicolumn{4}{c}{Agreement} & \multicolumn{4}{c}{Own-token sensitivity}                                                                                                       \\
    \cmidrule(lr){2-5}\cmidrule(l){6-9}
                         & den.$\leftrightarrow$cav.     & den.$\leftrightarrow$score                & cav.$\leftrightarrow$score & cav.\ read as den. & cav.    & den.$\to$cav. & den.    & cav.$\to$den. \\
    \midrule
    Masked               & ---                           & ---                                       & $0.046$                    & $0$                & ---     & ---           & ---     & ---           \\
    GIDD $\lambda{=}0.1$ & $0.026$                       & $0.030$                                   & $0.031$                    & $0.184$            & $0.107$ & $0.200$       & $0.340$ & $0.338$       \\
    GIDD $\lambda{=}0.2$ & $0.027$                       & $0.034$                                   & $0.035$                    & $0.182$            & $0.083$ & $0.157$       & $0.298$ & $0.296$       \\
    Uniform              & $0.033$                       & $0.050$                                   & $0.048$                    & $0.173$            & $0.047$ & $0.081$       & $0.206$ & $0.203$       \\
    \bottomrule
  \end{tabular}
\end{table}

\subsection{Trainability and calibration of the three coordinates}\label{subsec:trainability}

The three coordinates all train the same object and share the NELBO floor (Theorem~\ref{thm:coords}), so any difference between them must be a difference in trainability.
Figure~\ref{fig:training-temperature} (top) shows the NELBO over training for every run, with the calibration identities of Proposition~\ref{prop:calibration} visible at step 0: cavity and score start at $\log V$ for every kernel, while the denoiser starts far higher, finite on the clipped window but divergent as $t_1\downarrow0$.
The score coordinate lags behind on every kernel in both NELBO and sample quality, plausibly because it must learn the normalization that the other two get by construction.
Its quality--entropy frontier, traced by the sampling temperature, sits below the other two, worst on uniform diffusion (Figure~\ref{fig:training-temperature}, bottom), consistent with the gap between SEDD and MDLM \citep{lou2024sedd,sahoo2024mdlm}.\footnote{Score samples also pass through the clipped inverse of App.~\ref{app:impl}, which on test data moves $4$--$5\%$ of the noisy-token probability mass on average.
  It cannot explain the lag alone, which also appears in the NELBO, computed from the raw score, and under masked diffusion, where the inverse is exact.}
Despite its higher initial NELBO, the denoiser recovers over training, and it is indistinguishable from the cavity at our resolution, both in NELBO and along the frontier (further results, such as the GenPPL frontier in Figure~\ref{fig:temperature-genppl}, in App.~\ref{app:further}).

\begin{figure}[h]
  \centering
  \includegraphics[width=1.0\linewidth]{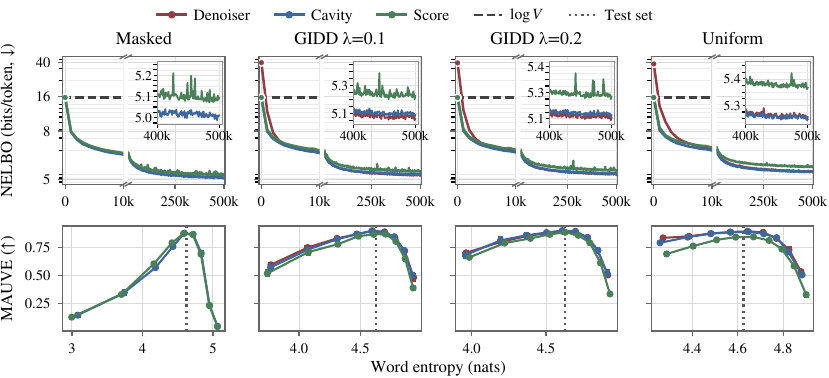}
  \caption{Top: online validation NELBO over training, where cavity and score start at $\log V$ on every kernel and the denoiser far higher (Proposition~\ref{prop:calibration}).
    Bottom: quality--entropy frontiers traced by the sampling temperature, MAUVE scored with GPT-2 Large against word entropy (App.~\ref{app:metric-details}).}
  \label{fig:training-temperature}
\end{figure}

\section{Related work}\label{sec:related}

The CTMC bound is due to \citet{campbell2022continuous}, with path-KL readings in the SDE, general-Markov, and CTMC lines \citep{song2021maximum,benton2024denoising,ren2025stochastic}, which absorb the boundary terms into constants or send them to zero.
The Oracle Distance identity is instead exact for arbitrary model rates on clipped windows, keeps both boundary terms (Proposition~\ref{prop:elbo} and Theorem~\ref{thm:oracle-window}), and splits by~\eqref{eq:pythagoras} into the information the noising process destroys and the model's distance to the oracle.
The method lineages follow the three coordinates: the score line \citep{lou2024sedd,ou2025radd,zhang2025tcsm} and the plug-in line \citep{austin2021d3pm,schiff2025udlm,vonruette2025gidd} are the score and cavity coordinates, and the masked cross-entropy objectives \citep{sahoo2024mdlm,shi2024simplified} are the collapsed case of Corollary~\ref{cor:masked}.
Proposition~\ref{prop:infoloss} with \eqref{eq:pythagoras} extends the discrete estimation--information identities \citep{jeon2025infodiscrete} to arbitrary model rates, finite windows, and every coordinate, and the invariance of the ELBO to time schedules (\S\ref{sec:oracle}) is the CTMC counterpart of \citet{kingma2023understanding}.
\citet{sun2022score} introduce the cavity for uniform diffusion, with its plug-in sampler, but train it by ratio matching rather than on the ELBO.
Apart from concurrent work, the literature had not recognized that the ELBO of the bridge plug-in is optimized by the cavity rather than by the clean-data predictor: in D3PM \citep{austin2021d3pm}, for instance, the auxiliary cross-entropy trains the denoiser while the plug-in loss and sampler assume the cavity; they report that this term helps their absorbing models and hurts their uniform ones (further details in App.~\ref{app:related}).
\citet{li2026neuralctmc} read the ELBO as a path KL up to an unspecified constant and list the three coordinates, but without the boundary terms, the Oracle Distance identity and its corollaries, the optimality of each coordinate, or the cavity finding.
Concurrently and independently, \citet{gourevitch2026udm} make the same observation for uniform diffusion and derive its conversions between denoiser, cavity, and score.
We obtain the cavity minimizer and the conversions for every token-factorizing process from one projection principle, and add the Oracle Distance identity with its corollaries, the calibration identities, and an evaluation of all three coordinates at GPT-2 scale across the standard noising processes.

\section{Conclusion}\label{sec:limitations}

We have answered the question from our title: a discrete diffusion model learns the exact time-reversal of the noising process.
Denoiser, cavity, and concrete score are three coordinates of this same reverse process, related by exact conversions; in particular, the bridge plug-in learns the cavity rather than a denoiser.
The Oracle Distance identity gives the corresponding global picture: the NELBO is exactly the data entropy plus the path KL from the oracle reverse process to the learned one, and its minimizer is the projection of the clean-conditioned reverse rates onto the noisy state.
Together with explicit boundary terms and calibration identities at initialization, this yields a conceptual framework and a mechanical recipe for deriving, implementing, and comparing discrete diffusion processes.

Our coordinate theory currently assumes token-factorizing corruption; extending it to forward processes with joint multi-token transitions is a natural direction.
Within the factorized setting, our framework unifies the forward noising process and the coordinate in which its reverse is represented and learned, while the sampler remains a separate approximation layer.
Bridging the learned continuous-time reverse process to the different finite-step sampling schemes used in practice is therefore an important remaining problem.
In this sense, the learned reverse process is unified; its realization by practical samplers is not.

\subsubsection*{Ethics statement}
This work is theoretical and methodological: it characterizes the training objective of discrete diffusion models and validates the analysis with small research models trained on OpenWebText, a public web corpus with known biases.
We release no models or systems intended for deployment; the usual considerations for language-model research apply.

\subsubsection*{Reproducibility statement}
App.~\ref{app:details} specifies the training recipe, the data pipeline and splits, and every evaluation procedure and operating point.
App.~\ref{app:processes} gives the per-kernel loss, conversion, and sampler formulas in closed form.
Proofs of every stated result are in Apps.~\ref{app:elbo}--\ref{app:processes}.
The code used for the experiments will be released with the paper.

\subsubsection*{AI use statement}
We used large language models to aid and polish writing (editing and reviewing drafts), for retrieval and discovery (searching for and summarizing related work), and for research ideation and execution (discussing ideas and derivations, and implementing, running, and analyzing experiments).
The authors reviewed all AI-assisted work, verified every technical claim and result, and take full responsibility for the content of this paper.

\bibliography{ref}
\bibliographystyle{iclr2027_conference}

\clearpage
\appendix

\numberwithin{equation}{section}
\numberwithin{theorem}{section}
\numberwithin{lemma}{section}
\numberwithin{proposition}{section}
\numberwithin{corollary}{section}
\numberwithin{definition}{section}
\numberwithin{remark}{section}
\numberwithin{assumption}{section}
\setcounter{theorem}{0}\setcounter{lemma}{0}\setcounter{proposition}{0}
\setcounter{corollary}{0}\setcounter{definition}{0}\setcounter{remark}{0}\setcounter{assumption}{0}

\section*{Appendix contents}
\startcontents[appendix]
\printcontents[appendix]{l}{1}{\setcounter{tocdepth}{2}}

\section{Continuous-time Markov chain discrete diffusion}\label{app:primer}

This appendix collects the variational and Markov-chain background the paper rests on; see \citet{norris1997markov,bremaud1981point,jacod2003limit} for textbook references.

\subsection{Variational formulation of diffusion models}\label{app:variational}

We begin with the variational identity underlying VAEs \citep{kingma2014autoencoding,rezende2014stochastic}. Let $\p^\theta(z_0,z_1)$ be a generative model for a finite data distribution $z_0\sim\cleanq$ on $\X$ and a latent variable $z_1$, and let $\q(z_1\given z_0)$ be an encoder satisfying $\q(z_1\given z_0)\ll \p^\theta(z_0,\cdot)$ for every $z_0$. Jensen's inequality for the concave $\log$ gives
\begin{equation}
  \log \p^\theta_0(z_0) = \log \Exp{\q(z_1\given z_0)}{\frac{\p^\theta(z_0,z_1)}{\q(z_1\given z_0)}}
  \geq \Exp{\q(z_1\given z_0)}{\log \frac{\p^\theta(z_0,z_1)}{\q(z_1\given z_0)}} =: \ELBO(\theta;\,z_0),
  \label{eq:VAE-ELBO-def}
\end{equation}
the evidence lower bound, which splits as
\begin{equation*}
  \ELBO(\theta;\,z_0)
  = \Exp{\q(z_1\given z_0)}{\log \p^\theta(z_0\given z_1)} - \KLdiv{\q(\cdot\given z_0)}{\p^\theta_{z_1}},
\end{equation*}
a reconstruction term and a prior-matching term. The gap of the inequality is exactly a KL divergence between the encoding and model posterior laws (throughout the appendix, \emph{law} is a synonym of \emph{distribution}, and $\Law(X)$ denotes the distribution of $X$):
\begin{equation}
  \log \p^\theta_0(z_0)-\ELBO(\theta;\,z_0) = \KLdiv{\q(\cdot\given z_0)}{\p^\theta(\cdot\given z_0)} \geq 0.
  \label{eq:VAE-gap}
\end{equation}

Diffusion models are VAEs whose latent variable is the noising trajectory. Fix $z_0\in\X$ and let $z_{1:N}$ be a discrete-time Markov trajectory with $N\geq1$, initial law $\cleanq$, and forward kernels $\q_{t\given t-1}$, so that $\q_{1:N\given0}(z_{1:N}\given z_0) = \prod_{t=1}^N \q_{t\given t-1}(z_t\given z_{t-1})$. The learned reverse-time model factors as $\p^\theta(z_{0:N}) = \p^\theta_N(z_N) \prod_{t=1}^N \p^\theta_{t-1\given t}(z_{t-1}\given z_t)$ and is also Markov in forward time. In the VAE reading, $z_{1:N}$ is the latent variable and $\p^\theta_N$ its prior.

\begin{proposition}[Discrete-time diffusion ELBO]\label{prop:discrete-time-diffusion-elbo}
  Assume the KL terms below are well defined. Then
  \begin{align}
    \ELBO(\theta;\,z_0)
     & = \E_{\q_{1\given0}(z_1\given z_0)} \left[ \log \p^\theta_{0\given 1}(z_0\given z_1) \right] - \KLdiv{\q_{N\given0}(\cdot\given z_0)}{\p^\theta_N} \nonumber       \\
     & \qquad - \sum_{t=2}^N \E_{\q_{t\given0}(z_t\given z_0)} \left[ \KLdiv{\q_{t-1\given t,0}(\cdot\given z_t,z_0)} {\p^\theta_{t-1\given t}(\cdot\given z_t)} \right],
    \label{eq:discrete-time-diffusion-elbo}
  \end{align}
  and $\log \p^\theta_0(z_0)-\ELBO(\theta;\,z_0) = \KLdiv{\q_{1:N\given0}(\cdot\given z_0)}{\p^\theta_{1:N\given0}(\cdot\given z_0)} \geq 0$.
\end{proposition}

\begin{proof}
  Start from the VAE ELBO equality with latent $z_{1:N}$, and use the reverse factorization $\q_{1:N\given0}(z_{1:N}\given z_0) = \q_{N\given0}(z_N\given z_0) \prod_{t=2}^N \q_{t-1\given t,0}(z_{t-1}\given z_t,z_0)$ together with the model factorization above:
  \begin{equation*}
    \ELBO(\theta;\,z_0)
    = \E_{\q_{1:N\given0}(z_{1:N}\given z_0)} \left[ \log \frac{ \p^\theta_N(z_N) \prod_{t=1}^{N}\p^\theta_{t-1\given t}(z_{t-1}\given z_t) }{ \q_{N\given0}(z_N\given z_0) \prod_{t=2}^{N}\q_{t-1\given t,0}(z_{t-1}\given z_t,z_0) } \right].
  \end{equation*}
  Splitting the logarithm into boundary and transition terms gives~\eqref{eq:discrete-time-diffusion-elbo}. The gap identity is the VAE identity~\eqref{eq:VAE-gap} with latent $z_{1:N}$.
\end{proof}

The discrete-time ELBO already has the structure that persists in continuous time: a reconstruction term at the left endpoint, a prior term at the right endpoint, and a sum of local discrepancies between the clean-conditioned reverse kernels and the learned reverse kernels.

\subsection{CTMC formulation}\label{app:ctmc}

\paragraph{Notation.}
A single letter $q$ or $p$ denotes the finite-dimensional marginals, kernels, and bridges of a process, with subscripts recording the relevant times: $\q_t$ is the time-$t$ marginal, $\q_{t\given s}$ the transition kernel for $s<t$, $\q_{t\given0}(\cdot\given z_0)$ the forward kernel from clean data, and $\q_{s\given t,0}(\cdot\given z_t,z_0)$ the bridge. We identify a random variable with its value when clear, writing $\q(z_0\given z_t)$ for $\q(Z_0=z_0\given Z_t=z_t)$. Jump rates take $Q,R$ (and $a,b$ for generic rates); a hat denotes time reversal. We abbreviate ``almost every(where)'' as a.e.\ and ``almost surely'' as a.s. For laws, $\mu\ll\nu$ denotes absolute continuity (on finite spaces, $\supp\mu\subseteq\supp\nu$), with relative density $d\mu/d\nu$. $H(Z)$ is the entropy of a discrete variable and $I(X;Y)=H(X)-H(X\given Y)$ the mutual information.

\paragraph{CTMC processes and their generators.}
Let $\X$ be finite and $0<T\le\infty$. A CTMC $(Z_t)_{t\in[0,T]}$ on $\X$ is fixed by a time-dependent rate matrix (generator) $\fQ{t}$ with $\fQ{t}(x,y)\geq 0$ for $x\neq y$ and $\fQ{t}(x,x) = -\sum_{y\neq x}\fQ{t}(x,y)$, through
\begin{equation}
  \q_{t+h\given t}(y\given x) = \one\{x=y\} + h\,\fQ{t}(x,y) + o(h), \qquad h\downarrow0.
\end{equation}
Sample paths are right-continuous and piecewise constant.

\paragraph{Kolmogorov equations.}
Arranged as matrices with normalized rows, the kernels obey Chapman--Kolmogorov $\q_{u\given s}=\q_{t\given s}\q_{u\given t}$ and the forward and backward linear ODEs
\begin{equation*}
  \partial_t \q_{t\given s}=\q_{t\given s}\fQ{t},\qquad \partial_s \q_{t\given s}=-\fQ{s} \q_{t\given s},\qquad \q_{s\given s}=I.
\end{equation*}
With $\cleanq$ given, the marginal solves $\partial_t \q_t=\q_t \fQ{t}$, and by uniqueness of solutions of linear ODEs the law of the chain is determined by $\cleanq$ and $\fQ{t}$.

\paragraph{Path laws and trajectory likelihoods.}\label{par:path-laws}
Three standard facts describe the process at the path level. The first is well-posedness: a law at the time from which the process is run (the left endpoint for a forward chain, the right endpoint for a reverse-time chain), together with bounded, measurable rates on the window $[s,u]$, determines a unique CTMC \emph{path law}, a probability measure on the space of right-continuous, piecewise-constant trajectories, just as $\cleanq$ and $\fQ{t}$ determine the marginals through the Kolmogorov equations. Every process in this paper is specified this way, by an endpoint law and a jump-rate family. Path laws take an uppercase $P$, subscripted by the window as in $P_{[s,u]}$, and integration against them is written $P(dz_{[s,u]})$; finite-space factors keep plain pmf notation, so a joint law over a data point and a trajectory reads $P(z_0,dz_{[s,u]})$.

The second fact is the trajectory likelihood. On a finite state space with bounded rates on a compact window, the chain makes finitely many jumps a.s., so a path is described by its jump coordinates: the initial state $x_0$, the number of jumps $n$, the ordered jump times $s<r_1<\cdots<r_n\le u$, and the post-jump states $x_1,\ldots,x_n$. With respect to the reference measure on these jump coordinates, a CTMC with initial law $p_s$ and generator $\fQ{t}$ has density
\begin{equation}
  p_s(x_0)\, \exp\!\Big(\int_s^{u}\fQ{t}(Z_t,Z_t)\,dt\Big) \prod_{k=1}^{n} \fQ{r_k}(x_{k-1},x_k):
  \label{eq:path-density}
\end{equation}
the initial mass, the no-jump survival factor accumulated from the exit rate, and the rate of each jump. This density is relative to the jump-coordinate reference measure; only ratios of such densities enter below, for which the reference measure cancels.

Third, the compensated jump sum $M_t^{xy} := N_t^{xy}-\Lambda_t^{xy}$, with $N_t^{xy}$ the number of $x\to y$ jumps in $[s,t]$ and $\Lambda_t^{xy} := \int_s^t \fQ{r}(Z_r,y)\one\{Z_r=x\}\,dr$ its compensator, is a mean-zero martingale. In particular, for any $f:\X\times\X\to\R$,
\begin{equation}
  \Exp{}{\sum_{s<r\le u} f(Z_{r^-},Z_r)\,\one\{Z_{r^-}\neq Z_r\}}
  =\Exp{}{\int_s^u \sum_{y\neq Z_t} \fQ{t}(Z_t,y)\,f(Z_t,y)\,dt}.
  \label{eq:jump-compensator}
\end{equation}

\paragraph{Time reversal.}\label{par:time-reversal}
The time-reversed process is again a CTMC. When $\q_t(x)>0$ its generator is $\RevQ{t}(x,y)=\fQ{t}(y,x)\,\q_t(y)/\q_t(x)$ for $x\neq y$, by the Markov property and Bayes' rule; conditioning the reversal on the endpoint $z_0$ replaces $\q_t$ by $\q_{t\given0}(\cdot\given z_0)$ and gives the clean-conditioned reverse rate of~\eqref{eq:reverse-rates}. In terms of small-time kernels,
\begin{equation*}
  \q_{t-h\given t,0}(y\given z_t,z_0) = \one\{z_t=y\} + h\,\RevQ{t}(z_t,y\given z_0) + o(h),
  \qquad h\downarrow 0.
\end{equation*}
Direction is thus a property of the chosen factorization, an endpoint law and rates, not of the law itself: a process Markov in forward time is Markov in reverse time, with different rates.

\begin{definition}[Noising and denoising path laws]\label{def:noising-path-laws}
  Fix a window $[t_1,t_2]\subset(0,T)$ and let $(Z_t)_{t\in\{0\}\cup[t_1,t_2]}$ be the noising process with $Z_0\sim\cleanq$. The \emph{clean-conditioned noising path law} is
  \begin{equation*}
    \Pstar{[t_1,t_2]\given0}(\cdot\given z_0)
    :=\Law\big((Z_t)_{t\in[t_1,t_2]}\given Z_0=z_0\big),
  \end{equation*}
  a probability measure on the path space over $[t_1,t_2]$.\footnote{Formally, the Skorokhod space $D([t_1,t_2],\X)$ of c\`adl\`ag paths with the $\sigma$-field generated by the coordinate maps. This $\sigma$-field is generated by cylinder events, so a path law is determined by its finite-dimensional marginals, which are in turn determined by an endpoint law and the jump rates through the Kolmogorov equations; existence follows from the construction of the chain from its jump times and jump chain. This is the rigorous content of the well-posedness fact above.} Prefixing the data law gives the \emph{joint} law
  \begin{equation*}
    \Pstar{{0},[t_1,t_2]}(z_0,dz_{[t_1,t_2]}) = \cleanq(z_0)\,\Pstar{[t_1,t_2]\given0}(dz_{[t_1,t_2]}\given z_0),
  \end{equation*}
  and marginalizing out $z_0$ gives the \emph{marginal noising path law} $\Pstar{[t_1,t_2]}:=\Law\big((Z_t)_{t\in[t_1,t_2]}\big)$.

  The \emph{model (denoising) path law} $\Ptheta{[t_1,t_2]}$ is initialized at $\p^\theta_{t_2}$ and run backward with $\RevQth{t}(z_t,y)$. Attaching the data through a reconstruction kernel extends it to the model joint
  \begin{equation*}
    \Ptheta{{0},[t_1,t_2]}(z_0,dz_{[t_1,t_2]}) = \Ptheta{[t_1,t_2]}(dz_{[t_1,t_2]})\,\p^\theta_{0\given t_1}(z_0\given z_{t_1}).
  \end{equation*}
\end{definition}

All of these are well defined by well-posedness and, by Kolmogorov uniqueness, each noising law is equivalently characterized by running the corresponding rates from either endpoint of the interval, conditionally on $Z_0$ or not, forward or backward. For instance, the clean-conditioned law $\Pstar{[t_1,t_2]\given 0}$ run backward is the CTMC started at $\q_{t_2\given0}(\cdot\given z_0)$ with the clean-conditioned reverse rates, and the joint law $\Pstar{{0},[t_1,t_2]}$ read entirely in reverse time starts at $\q_{t_2}$, runs backward with the marginal reverse rates $\RevQ{t}$, and attaches the data through $\q_{0\given t_1}(\cdot\given z_{t_1})$, using that $Z_0$ is conditionally independent of $Z_{[t_1,t_2]}$ given $Z_{t_1}$. The joint model law is determined by the triple $(\p^\theta_{t_2},\RevQth{t},\p^\theta_{0\given t_1})$, where $\RevQth{t}$ may depend on the noisy state but not on the clean data.

The practical interest of diffusion rests on the terminal mixing assumption below; most results in this paper do not need it.
\begin{assumption}[Terminal mixing]\label{ass:terminal-mixing}
  As $t\uparrow T$ (including $T=\infty$), the forward process forgets the data: $\q_{t\given0}(\cdot\given z_0)\to \q_T$ for every $z_0\in\supp \cleanq$, where $\q_T$ does not depend on $z_0$. In particular the joint law of $(Z_0,Z_t)$ converges to the product $\cleanq\otimes \q_T$ and $I(Z_0;Z_t)\to0$.
\end{assumption}

\section{The finite-window CTMC ELBO}\label{app:elbo}

We state the ELBO on an interior compact window $[t_1,t_2]\subset(0,T)$, avoiding potentially singular endpoints as is done in practice. The local rate divergence $\Phi$ is defined in~\eqref{eq:phi}, with the conventions $\Phi(0,b)=b$ and $\Phi(a,0)=+\infty$ for $a>0$; it is nonnegative and vanishes iff $a=b$.

\begin{assumption}[Rate-support condition]\label{ass:support}
  The learned reverse rates satisfy $\RevQ{t}(z_t,\cdot\given z_0)\ll\RevQth{t}(z_t,\cdot)$ for $dt\,dq(z_t,z_0)$-a.e.\ $(t,z_t,z_0)$, that is, the learned process can jump wherever the true reverse process can.
\end{assumption}

\begin{theorem}[CTMC ELBO, complete form]\label{thm:elbo-full}
  Let $\X$ be finite and let the true and learned reverse CTMCs have bounded, measurable rates $\RevQ{t}(\cdot\given z_0)$ and $\RevQth{t}$ on $[t_1,t_2]\subset(0,T)$. Define the per-time jump divergence
  \begin{equation}
    \J^\theta_t(z_0,z_t) := \sum_{y\neq z_t} \Phi\big(\RevQ{t}(z_t,y\given z_0), \RevQth{t}(z_t,y)\big).
    \label{eq:jump-divergence}
  \end{equation}
  Assume the rate-support condition (Assumption~\ref{ass:support}), the terminal absolute continuity $\q_{t_2\given0}(\cdot\given z_0)\ll \p^\theta_{t_2}$, integrability $\int_{t_1}^{t_2} \Exp{\q_{t\given0}(z_t\given z_0)}{\J^\theta_t(z_0,z_t)}\,dt<\infty$, and that the reconstruction term below is well defined. Let $\Ptheta{{0},[t_1,t_2]}(z_0,\cdot)$ be the fixed-$z_0$ section of the model joint law (Definition~\ref{def:noising-path-laws}), a finite path measure of total mass $\p^\theta_0(z_0)$. With $z_{[t_1,t_2]}$ as latent variable, the ELBO defined at the path-measure level by
  \begin{equation}
    \ELBO_{[t_1,t_2]}(\theta;\,z_0)
    := -\KLdiv{\Pstar{[t_1,t_2]\given0}(\cdot\given z_0)}{\Ptheta{{0},[t_1,t_2]}(z_0,\cdot)}
    \label{eq:ctmc-elbo-path-definition}
  \end{equation}
  admits the expanded form
  \begin{equation}
    \begin{split}
      \ELBO_{[t_1,t_2]}(\theta;\,z_0)
       & = \underbrace{\Exp{\q_{t_1\given0}(z_{t_1}\given z_0)}{\log \p^\theta_{0\given t_1}(z_0\given z_{t_1})}}_{\text{reconstruction}}
      - \underbrace{\KLdiv{\q_{t_2\given0}(\cdot\given z_0)}{\p^\theta_{t_2}}}_{\text{terminal prior}}                                    \\
       & \qquad - \underbrace{\int_{t_1}^{t_2} \Exp{\q_{t\given0}(z_t\given z_0)}{\J^\theta_t(z_0,z_t)}\,dt}_{\text{path integral}}.
    \end{split}
    \label{eq:ctmc-elbo-full}
  \end{equation}
  Moreover the ELBO gap is the relative entropy between the clean-conditioned path laws,
  \begin{equation*}
    \log \p^\theta_0(z_0)-\ELBO_{[t_1,t_2]}(\theta;\,z_0)
    = \KLdiv{\Pstar{[t_1,t_2]\given 0}(\cdot\given z_0)} {\Ptheta{[t_1,t_2]\given 0}(\cdot\given z_0)}\geq0,
  \end{equation*}
  with $\Ptheta{[t_1,t_2]\given0}(\cdot\given z_0):=\Ptheta{{0},[t_1,t_2]}(z_0,\cdot)/\p^\theta_0(z_0)$ the normalized model posterior path law.
\end{theorem}
In~\eqref{eq:ctmc-elbo-path-definition} the KL symbol is a mean log-ratio against a finite measure of mass $\p^\theta_0(z_0)$, well defined under the stated absolute-continuity hypotheses, the Radon--Nikodym derivative being a ratio of densities of the form~\eqref{eq:path-density}; this is what makes~\eqref{eq:ctmc-elbo-path-definition} a lower bound on $\log \p^\theta_0(z_0)$ rather than an ordinary divergence. Proposition~\ref{prop:elbo} in the main text is the expanded form together with the sign of the gap.

\subsection{First derivation: the infinitesimal-KL limit}\label{app:elbo-proof-infinitesimal}

The first derivation rests on an infinitesimal expansion of one-step KLs between jump kernels.

\begin{lemma}[$\Phi$ as the differential KL of jump kernels]\label{lem:diffkl}
  Let $a_t,b_t$ be piecewise continuous CTMC rates on a finite space $\X$ and a compact interval $[t,t+\Delta]$, $\Delta>0$, with transition kernels $q^a_{t,t+\Delta}$, $p^b_{t,t+\Delta}$. Assume the jump graph of $a$ is a fixed set $E$ and, for constants $0<c<C<\infty$ and every $t$,
  \begin{equation}
    c\leq a_t(x,y),b_t(x,y)\leq C\ \ \forall(x,y)\in E;
    \qquad
    a_t(x,y)=0\ \ \forall(x,y)\notin E;
    \qquad
    0\leq b_t(x,y)\leq C.
    \label{eq:diffkl-rate-assumptions}
  \end{equation}
  Then, uniformly in $t$ and $x$,
  \begin{equation*}
    \KLdiv{q^a_{t,t+\Delta}(\cdot\given x)}{p^b_{t,t+\Delta}(\cdot\given x)}
    =\Delta\sum_{y\neq x}\Phi\big(a_t(x,y),b_t(x,y)\big)+o(\Delta).
  \end{equation*}
\end{lemma}
\begin{proof}
  We expand the KL by destination, according to whether $(x,y)\in E$.

  For jumps $x\neq y$ with $(x,y)\in E$, integrating the uniformly continuous rates gives $q^a_{t,t+\Delta}(y\given x)=\Delta a_t(x,y)+o(\Delta)$ and $p^b_{t,t+\Delta}(y\given x)=\Delta b_t(x,y)+o(\Delta)$. Because both rates are bounded below and above on $E$,
  \begin{equation}
    q^a_{t,t+\Delta}(y\given x)
    \log\frac{q^a_{t,t+\Delta}(y\given x)}{p^b_{t,t+\Delta}(y\given x)}
    =\Delta a_t(x,y)\log\frac{a_t(x,y)}{b_t(x,y)}+o(\Delta),
    \qquad (x,y)\in E.
    \label{eq:diffkl-one-jump-contribution}
  \end{equation}

  For $x\neq y$ with $(x,y)\notin E$: if $y$ is unreachable from $x$ in the graph $E$ then $q^a_{t,t+\Delta}(y\given x)=0$. Otherwise let $d\geq2$ be the length of a shortest $E$-path from $x$ to $y$. The uniform rate bounds and the finiteness of $\X$ bound the total exit rate from any state, so the number of jumps $N_\Delta$ in $[t,t+\Delta]$ is dominated by a Poisson variable and $\P(N_\Delta\geq k)=O(\Delta^k)$ uniformly. Reaching $y$ needs at least $d$ jumps, so $q^a_{t,t+\Delta}(y\given x)=O(\Delta^d)$; following one shortest path and making no other jumps has probability at least a constant times $\Delta^d$ under both processes, while any single transition probability is $O(\Delta)$, so $|\log(q^a_{t,t+\Delta}/p^b_{t,t+\Delta})|=O(|\log\Delta|)$ wherever $q^a_{t,t+\Delta}(y\given x)>0$. Hence each such term is $O(\Delta^d|\log\Delta|)=O(\Delta^2|\log\Delta|)$ and
  \begin{equation}
    \sum_{y\neq x,\,(x,y)\notin E}q^a_{t,t+\Delta}(y\given x)
    \log\frac{q^a_{t,t+\Delta}(y\given x)}{p^b_{t,t+\Delta}(y\given x)}
    =o(\Delta).
    \label{eq:diffkl-multi-jump-contribution}
  \end{equation}

  For the self-transition, returning to $x$ after a jump needs another jump, of probability $O(\Delta^2)$, so with $A_t(x):=\sum_{y\neq x} a_t(x,y)$ and $B_t(x):=\sum_{y\neq x} b_t(x,y)$, $q^a_{t,t+\Delta}(x\given x)=1-\Delta A_t(x)+o(\Delta)$ and likewise for $p^b$, giving the contribution
  \begin{equation}
    (1 - \Delta A_t(x) + o(\Delta)) \log \frac{1 - \Delta A_t(x) + o(\Delta)}{1 - \Delta B_t(x) + o(\Delta)}
    = \Delta(B_t(x)-A_t(x))+o(\Delta).
    \label{eq:diffkl-self-transition-contribution}
  \end{equation}
  Adding \eqref{eq:diffkl-self-transition-contribution}, \eqref{eq:diffkl-one-jump-contribution}, and \eqref{eq:diffkl-multi-jump-contribution}, and collecting through $\Phi(a,b)=b-a+a\log(a/b)$, proves the claim, uniformly in $t$ and $x$.
\end{proof}
The additional hypotheses of Lemma~\ref{lem:diffkl} (fixed jump graph, piecewise continuity, uniform bounds) are needed only for this first derivation; they hold for the masked, uniform, and fixed-support GIDD processes of App.~\ref{app:processes}.

\begin{proof}[Proof 1 of Theorem~\ref{thm:elbo-full} (infinitesimal KL)]
  Impose the conditions of Lemma~\ref{lem:diffkl} on each continuity piece. Apply the discrete-time ELBO (Proposition~\ref{prop:discrete-time-diffusion-elbo}) to a finite partition $t_1=s_0<s_1<\cdots<s_N=t_2$ with $\Delta_k:=s_{k+1}-s_k$ and $|\Delta|:=\max_k\Delta_k$:
  \begin{align*}
    \ELBO_{(s_k)_k}(\theta;\,z_0)
     & = \Exp{\q_{t_1\given0}(z_{t_1}\given z_0)}{\log \p^\theta_{0\given t_1}(z_0\given z_{t_1})}
    - \KLdiv{\q_{t_2\given0}(\cdot\given z_0)}{\p^\theta_{t_2}}
    \\
     & \qquad - \sum_{k=0}^{N-1} \Exp{\q_{s_{k+1}\given0}}
                                 {\KLdiv{\q_{s_k\given s_{k+1},0}(\cdot\given z_{s_{k+1}},z_0)}
                                   {\p^\theta_{s_k\given s_{k+1}}(\cdot\given z_{s_{k+1}})}}.
  \end{align*}
  The two kernels in each summand are the exact transitions of the true and learned reverse CTMCs over an interval of length $\Delta_k$. Lemma~\ref{lem:diffkl}, applied in the backward orientation, turns each into $\Delta_k\sum_{y\neq x}\Phi(\RevQ{s_{k+1}}(x,y\given z_0),\RevQth{s_{k+1}}(x,y))+r_k(x)$ with uniform remainders $|r_k(x)|\le\Delta_k\,\eta(\Delta_k)$, $\eta(\delta)\to0$. Taking expectations and summing, the main sum is a Riemann sum of the piecewise continuous integrand $t\mapsto \Exp{\q_{t\given0}}{\J^\theta_t(z_0,z_t)}$ and converges to the path integral, while $|\sum_k\E[r_k]|\le(t_2-t_1)\sup_{\delta\le|\Delta|}\eta(\delta)\to0$. This yields~\eqref{eq:ctmc-elbo-full} as the vanishing-mesh limit of the discrete-time ELBOs; to identify the limit with the path-level definition~\eqref{eq:ctmc-elbo-path-definition}, take nested partitions whose union is dense in $[t_1,t_2]$: the skeletons generate the path $\sigma$-field and the skeleton KLs increase to the path KL \citep[Lemma~5.4.1]{gray2011entropy}. The same skeleton limit applied to the discrete-time gap identity yields the path-level gap.
\end{proof}

\subsection{Second derivation: CTMC Girsanov formula}\label{app:elbo-proof-girsanov}

The second derivation rests on the change-of-measure formula for finite-state jump processes, stated for two generic CTMCs.

\begin{theorem}[Girsanov formula and relative entropy for finite-state CTMCs]\label{thm:girsanov}
  Let $\X$ be finite and $[s,u]$ compact. Let $P,P'$ be the path laws of two CTMCs specified by initial laws $p_s,p'_s$ at time $s$ and bounded, measurable rates $a_t,b_t$, with $p_s\ll p'_s$ and $a_t(x,\cdot)\ll b_t(x,\cdot)$ for every $x$ and a.e.\ $t$. Then $P\ll P'$ and
  \begin{equation}
    \KLdiv{P}{P'}=\KLdiv{p_{s}}{p'_{s}}
    +\int_{s}^{u}\Exp{Z_t\sim P}{\sum_{y\neq Z_t}\Phi\big(a_t(Z_t,y),b_t(Z_t,y)\big)}\,dt,
    \label{eq:girsanov-path-kl}
  \end{equation}
  where both sides are finite if and only if the integral is.
\end{theorem}
A KL between two CTMC path laws is not a pure rate mismatch: it also carries the KL between the two laws at the starting time.
\begin{proof}
  Write $A_t(x)=\sum_{y\neq x}a_t(x,y)$ and $B_t(x)=\sum_{y\neq x}b_t(x,y)$ and apply the trajectory likelihood~\eqref{eq:path-density} to $P$ and $P'$. Under $a_t\ll b_t$ the jump-coordinate reference measure cancels in the density ratio, which is therefore well defined $P$-a.s., so $P\ll P'$ and
  \begin{equation*}
    \frac{dP}{dP'}
    =\frac{p_{s}(Z_{s})}{p'_{s}(Z_{s})}\,
    \exp\!\Big(\int_{s}^{u}\!\big(B_t(Z_t)-A_t(Z_t)\big)\,dt\Big)
    \prod_{s<r\le u:\,Z_{r^-}\neq Z_r}\frac{a_r(Z_{r^-},Z_r)}{b_r(Z_{r^-},Z_r)}
  \end{equation*}
  \citep{bremaud1981point,jacod2003limit}. Taking $\log$ and integrating against $P$, the jump compensator~\eqref{eq:jump-compensator} with $f=\log(a_t/b_t)$ turns the jump sum into $\int_{s}^{u}\Exp{P}{\sum_{y\neq Z_t}a_t(Z_t,y)\log\tfrac{a_t(Z_t,y)}{b_t(Z_t,y)}}\,dt$, the holding factor contributes $\int_{s}^{u}\Exp{P}{\sum_{y\neq Z_t}(b_t(Z_t,y)-a_t(Z_t,y))}\,dt$, and the initial factor contributes $\KLdiv{p_{s}}{p'_{s}}$. Collecting through $\Phi(a,b)=b-a+a\log(a/b)$ yields~\eqref{eq:girsanov-path-kl}.
\end{proof}

\begin{proof}[Proof 2 of Theorem~\ref{thm:elbo-full} (Girsanov)]
  Start from~\eqref{eq:ctmc-elbo-path-definition}. The reverse Markov property factorizes the model $z_0$-slice as $\Ptheta{{0},[t_1,t_2]}(z_0,dz_{[t_1,t_2]}) = \p^\theta_{0\given t_1}(z_0\given z_{t_1})\,\Ptheta{[t_1,t_2]}(dz_{[t_1,t_2]})$. Substituting and applying the Radon--Nikodym chain rule,
  \begin{equation*}
    -\ELBO_{[t_1,t_2]}(\theta;\,z_0)
    =
    -\E_{\q_{t_1\given0}(\cdot\given z_0)}
    \left[ \log \p^\theta_{0\given t_1}(z_0\given Z_{t_1}) \right]
    + \KLdiv{ \Pstar{[t_1,t_2]\given 0}(\cdot\given z_0)} {\Ptheta{[t_1,t_2]}},
  \end{equation*}
  since the time-$t_1$ marginal of $\Pstar{[t_1,t_2]\given0}(\cdot\given z_0)$ is $\q_{t_1\given0}(\cdot\given z_0)$. Apply Theorem~\ref{thm:girsanov} to the backward descriptions of $\Pstar{[t_1,t_2]\given0}(\cdot\given z_0)$ and $\Ptheta{[t_1,t_2]}$, started at $\q_{t_2\given0}(\cdot\given z_0)\ll \p^\theta_{t_2}$ with rates $\RevQ{t}(z_t,\cdot\given z_0)\ll\RevQth{t}(z_t,\cdot)$, to expand the KL term into the terminal prior KL plus the path integral, which gives~\eqref{eq:ctmc-elbo-full}. For the gap identity, substitute $\Ptheta{{0},[t_1,t_2]}(z_0,\cdot)=\p^\theta_0(z_0)\,\Ptheta{[t_1,t_2]\given0}(\cdot\given z_0)$ into~\eqref{eq:ctmc-elbo-path-definition}: the constant $\log \p^\theta_0(z_0)$ splits off and the KL between the clean-conditioned and model posterior path laws remains.
\end{proof}

\subsection{Importance sampling, time clocks, and the empirical ELBO}\label{app:clocks}

\paragraph{The empirical estimator.}
The bound becomes a training objective by estimating its time integral, optionally with an importance density $p_{\mathrm{IS}}(t)$ (equivalently a clock, \S\ref{sec:oracle}). Drawing $t\sim p_{\mathrm{IS}}$ on $[t_1,t_2]$ and $z_t\sim \q_{t\given0}(\cdot\given z_0)$ gives the unbiased estimator
\begin{equation}
  \widehat{\NELBO}_{[t_1,t_2]}(\theta;\,z_0)
  =\frac{\J^\theta_t(z_0,z_t)}{p_{\mathrm{IS}}(t)}
  +\KLdiv{\q_{t_2\given0}(\cdot\given z_0)}{\p^\theta_{t_2}} +\widehat{\mathrm{rec}}(\theta),
  \label{eq:empirical-elbo}
\end{equation}
where $\widehat{\mathrm{rec}}(\theta)=-\log \p^\theta_{0\given t_1}(z_0\given z_{t_1})$ at a draw $z_{t_1}\sim\q_{t_1\given0}(\cdot\given z_0)$; this term accommodates any final decoding step from $t_1$ to $0$, such as the standard collapse to clean tokens. The terminal-prior KL is computed analytically from the forward kernel and the chosen prior. Retaining both boundary terms keeps per-process estimates complete and comparable across methods.

Let $\tau$ be a strictly increasing $C^1$ change of time on $[t_1,t_2]$ with $d\tau/dt>0$ and inverse $t(\tau)$, and let $Z^{\prime}_{\tau}:=Z_{t(\tau)}$, with marginals $\q^{\prime}_{\tau}=\q_{t(\tau)}$ and kernels $\q^{\prime}_{\tau_2\given\tau_1}=\q_{t(\tau_2)\given t(\tau_1)}$. By the chain rule in the definition of the generator, the generator and the reverse rate each carry one Jacobian factor,
\begin{equation*}
  \dataC{Q^{\prime}_{\tau}}(x,y)=\frac{dt}{d\tau}\,\fQ{t(\tau)}(x,y),
  \qquad
  \dataC{\widehat{Q}^{\prime}_{\tau}}(z,y)
  =\dataC{Q^{\prime}_{\tau}}(y,z)\,\frac{\q^{\prime}_{\tau}(y)}{\q^{\prime}_{\tau}(z)}
  =\frac{dt}{d\tau}\,\RevQ{t(\tau)}(z,y),
\end{equation*}
and likewise the clean-conditioned rate $\dataC{\widehat{Q}^{\prime}_{\tau}}(z,y\given z_0)$ and the model rate $\modelC{\widehat{Q}^{\theta\prime}_{\tau}}(z,y)$. Estimating the path integral with an importance density $p_{\mathrm{IS}}$ is uniform sampling on the clock
\begin{equation*}
  \tau(t):=\tau_1+\int_{t_1}^{t}p_{\mathrm{IS}}(u)\,du,
  \qquad d\tau=p_{\mathrm{IS}}(t)\,dt,
  \qquad \frac{dt}{d\tau}=\frac{1}{p_{\mathrm{IS}}(t)}.
\end{equation*}
Since $\Phi(ca,cb)=c\,\Phi(a,b)$ for $c>0$, the Jacobian moves from the measure into both rate arguments,
\begin{equation}
  \begin{aligned}
    \int_{t_1}^{t_2}\Exp{\q_{t\given0}}{\J^\theta_t(z_0,z_t)}\,dt
     & =\int_{\tau_1}^{\tau_2}\frac{dt}{d\tau}\,
    \Exp{\q_{t\given0}}{\sum_{y\ne z_t}\Phi\big(\RevQ{t}(z_t,y\given z_0),\RevQth{t}(z_t,y)\big)}\,d\tau
    \\
     & =\int_{\tau_1}^{\tau_2}
    \Exp{\q_{t\given0}}{\sum_{y\ne z_t}\Phi\big(\dataC{\widehat{Q}^{\prime}_{\tau}}(z_t,y\given z_0),\modelC{\widehat{Q}^{\theta\prime}_{\tau}}(z_t,y)\big)}\,d\tau,
  \end{aligned}
  \label{eq:clock-invariance}
\end{equation}
with $t=t(\tau)$ throughout. So~\eqref{eq:empirical-elbo} is an unbiased estimator of the ELBO of the same pair of processes read on the clock $\tau$, and the optimizer is the same. When the generator is separable, $\fQ{t}=\tfrac{d\kappa}{dt}M$ for a strictly increasing scalar $\kappa$ and a fixed matrix $M$, the clock $\tau=\kappa(t)$ makes the forward process time-homogeneous with generator $M$.

Importance sampling leaves the estimator unbiased for the ELBO, whereas reweighting the per-time integrand without the $1/p_{\mathrm{IS}}$ correction changes the objective. Monotone reweightings define valid and possibly tighter variational bounds of their own \citep{kingma2023understanding,shi2025demystifying}, and are used to trade likelihood for sample quality.

\paragraph{The intrinsic clock.}
One clock is intrinsic to the process: the one on which it destroys information at a constant rate,
\begin{equation}
  \tau(t):=\frac{H(Z_0\given Z_t)}{H(\cleanq)},
  \qquad\text{so that}\qquad
  I(Z_0;Z_\tau)=(1-\tau)\,H(\cleanq).
  \label{eq:elbo_clock}
\end{equation}
By Proposition~\ref{prop:infoloss} the oracle ELBO integrand is exactly this information-loss rate, so $\tau$ is the clock on which the oracle cost per unit time is flat, and the per-time ELBO profile of a model is a rough estimator of the intrinsic clock density that improves with training. Under the linear schedule $\alpha_t=1-t$ the masked oracle cost is already constant in $t$ for unigram data and approximately so for uniform away from $t=1$, but other kernels can destroy information following very different schedules, which makes $\tau$ a natural reference for a new kernel, though not necessarily the best clock for training.

For a single draw $t\sim p_{\mathrm{IS}}$ the estimator variance is minimized at $p^\star_{\mathrm{IS}}(t)\propto\sqrt{\Exp{\cleanq\,\q_{t\given0}}{\J^\theta_t(z_0,z_t)^2}}$ \citep{robert2004monte}. This second-moment rule underlies adaptive timestep sampling in DDPMs, while variational diffusion models learn the noise schedule directly to reduce the variance of the bound \citep{nichol2021improved,kingma2023variational}. The intrinsic clock~\eqref{eq:elbo_clock}, $p_{\mathrm{IS}}\propto\J^\star_t$, equalizes the oracle cost per unit clock instead, and the two agree only where the root second moment of the integrand is proportional to its oracle mean; continuous-diffusion work adapts the training noise distribution from the same conditional-entropy-rate profile \citep{raya2026noise}.

\section{The Oracle Distance identity}\label{app:oracle}

We state the Oracle Distance in its finite-window form, prove it twice (by assembling the projection, Pythagorean, and information-loss identities, each of independent use, a route that needs the regularity of Assumption~\ref{ass:regularity-fixed-support}, and by a direct path-space argument that does not), and derive Theorem~\ref{thm:oracle} as its vanishing-boundary limit.

\subsection{The ELBO as a projection problem and the Pythagorean decomposition}\label{app:projection}

The model rate cannot depend on the unknown clean data, so the path term of the ELBO poses a projection problem: given a random target rate depending on $(Z_0,Z_t)$, choose the best rate that sees only $Z_t$. Its solution is governed by a Pythagorean identity for $\Phi$.

\begin{lemma}[Bregman--Pythagoras identity for $\Phi$]\label{lem:bregman-pythagoras}
  Let $A\geq0$ with $\E[A\log A]<\infty$, let $G$ be an observable, and let $B=b(G)\geq0$ be a $G$-measurable predictor for which the displayed expectations are finite. With $B^\star:=\Exp{}{A \given G}$,
  \begin{equation*}
    \E[\Phi(A,B)] = \E[\Phi(A,B^\star)] + \E[\Phi(B^\star,B)].
  \end{equation*}
\end{lemma}
\begin{proof}
  Condition on $G$ and write $m=B^\star$, $b=B$. For $m,b>0$, expanding $\Phi(a,b)=b-a+a\log\frac ab$ and using $\Exp{}{A\given G}=m$ (with $0\log0=0$),
  \begin{align*}
    \Exp{}{\Phi(A,b) \given G} & = b-m+\Exp{}{A\log A \given G}-m\log b, \\
    \Exp{}{\Phi(A,m) \given G} & = \Exp{}{A\log A \given G}-m\log m,     \\
    \Phi(m,b)                  & = b-m+m\log m-m\log b.
  \end{align*}
  Adding the last two reproduces the first, giving the conditional identity $\Exp{}{\Phi(A,B)\given G} = \Exp{}{\Phi(A,B^\star)\given G} + \Phi(B^\star,B)$; the boundary cases follow by continuity with the conventions in the definition of $\Phi$. Take expectations over $G$.
\end{proof}
Since $\E[\Phi(B^\star,B)]\geq0$ vanishes iff $B=B^\star$ a.e., the conditional mean is the unique minimizer of $\E[\Phi(A,B)]$ over nonnegative $G$-measurable predictors. (The same follows from convexity: $b\mapsto\Exp{}{\Phi(A,b)\given G=g}$ is strictly convex with derivative $1-\Exp{}{A\given G=g}/b$.)

\begin{thmrestate}{Corollary~\ref{cor:projection}}{Reverse-rate projection}
  Among rates measurable in the noisy state $Z_t$ alone, the NELBO is minimized, uniquely up to null sets, by $\modelC{\widehat{Q}{}^{\theta_\star}_t}(z_t,y)=\Exp{}{\RevQ{t}(Z_t,y\given Z_0)\given Z_t=z_t}=\RevQ{t}(z_t,y)$.
\end{thmrestate}
\begin{proof}
  Fix $t$ and a destination $y$, and apply Lemma~\ref{lem:bregman-pythagoras} with $G=Z_t$ and target $A=\RevQ{t}(Z_t,y\given Z_0)$. The optimal $Z_t$-measurable rate is the conditional mean, computed by Bayes' rule:
  \begin{align*}
    \Exp{}{\RevQ{t}(Z_t,y\given Z_0) \given Z_t=z_t}
     & = \fQ{t}(y,z_t) \sum_{z_0}
    \frac{\cleanq(z_0)\,\q_{t\given0}(z_t\given z_0)}{\q_t(z_t)}\,
    \frac{\q_{t\given0}(y\given z_0)} {\q_{t\given0}(z_t\given z_0)}                          \\
     & = \fQ{t}(y,z_t) \frac{1}{\q_t(z_t)}\sum_{z_0} \cleanq(z_0)\,\q_{t\given0}(y\given z_0) \\
     & = \fQ{t}(y,z_t)\frac{\q_t(y)}{\q_t(z_t)}
    = \RevQ{t}(z_t,y).
  \end{align*}
  Uniqueness is the equality case of the lemma. That the minimizer of the full NELBO is the same follows by summing over destinations and integrating over time, each term being minimized by the same rate.
\end{proof}

\begin{proposition}[Pythagorean decomposition of the ELBO integrand]\label{prop:pythagoras}
  For any $Z_t$-measurable model rate $\RevQth{t}$ such that the model mismatch term below is finite,
  \begin{equation}
    \Exp{\cleanq\,\q_{t\given0}}{\J^\theta_t(z_0,z_t)}
    = \Exp{\cleanq\,\q_{t\given0}}{\J^{\thetastar}_t(z_0,z_t)}
    + \Exp{\q_t}{\sum_{y\neq z_t}\Phi\big(\RevQ{t}(z_t,y),\RevQth{t}(z_t,y)\big)},
    \label{eq:pythagoras-averaged}
  \end{equation}
  where $\J^{\thetastar}_t$ is the integrand~\eqref{eq:jump-divergence} evaluated at the marginal reverse rate $\RevQ{t}$.
\end{proposition}
\begin{proof}
  Fix $t$ and a destination $y$ and apply Lemma~\ref{lem:bregman-pythagoras} with $G=Z_t$, target $A=\RevQ{t}(Z_t,y\given Z_0)$, and predictor $B=\RevQth{t}(Z_t,y)$; the conditional mean is $B^\star=\RevQ{t}(z_t,y)$ by the projection above. The mismatch expectation collapses to one over $z_t\sim \q_t$ because both arguments are $Z_t$-measurable. Sum over $y\neq z_t$.
\end{proof}

\subsection{The oracle ELBO is the rate of information loss}\label{app:infoloss}

The second ingredient evaluates the oracle cost $\J^\star_t=\Exp{\cleanq\,\q_{t\given0}}{\J^{\thetastar}_t(z_0,z_t)}$ under the following regularity, which the masked, uniform, and fixed-support GIDD processes satisfy on interior windows.

\begin{assumption}[Regularity and fixed support]\label{ass:regularity-fixed-support}
  On the compact window $[t_1,t_2]\subset(0,T)$, the rates $\fQ{t}$ are piecewise continuous and bounded, and the piecewise supports of $\q_t$ and $\q_{t\given0}(\cdot\given z_0)$ are independent of $t$ for each $z_0\in\supp\cleanq$.
\end{assumption}

\begin{thmrestate}{Proposition~\ref{prop:infoloss}}{Irreducible cost is information loss}
  Under Assumption~\ref{ass:regularity-fixed-support}, $\J^\star_t = \frac{d}{dt}H(Z_0\given Z_t)=-\frac{d}{dt}I(Z_0;Z_t)\geq0$, and $\int_{t_1}^{t_2}\J^\star_t\,dt = H(Z_0\given Z_{t_2})-H(Z_0\given Z_{t_1}) = I(Z_0;Z_{t_1})-I(Z_0;Z_{t_2})$.
\end{thmrestate}
\begin{proof}
  Terms with zero clean-conditioned rate contribute zero under $0\log0=0$; on positive terms the fixed-support assumption makes all ratios below well defined. Inside $\Phi(\RevQ{t}(\cdot\given Z_0),\RevQ{t})$ the rate quotient is, by the two reverse-rate formulas and Bayes' rule,
  \begin{equation*}
    \frac{\RevQ{t}(z,y\given z_0)}{\RevQ{t}(z,y)}
    = \frac{\q_{t\given0}(y\given z_0)}{\q_{t\given0}(z\given z_0)}\,\frac{\q_t(z)}{\q_t(y)}
    = \frac{\q_{0\given t}(z_0\given y)}{\q_{0\given t}(z_0\given z)}.
  \end{equation*}
  In the expansion $\Phi(a,b)=b-a+a\log(a/b)$ the linear terms cancel after conditioning on $Z_t$, because $\RevQ{t}(z,y)=\Exp{}{\RevQ{t}(z,y\given Z_0)\given Z_t=z}$. Hence, using $\q_{t\given0}(z\given z_0)\,\RevQ{t}(z,y\given z_0)= \q_{t\given0}(y\given z_0)\,\fQ{t}(y,z)$ and adding the diagonal generator terms, whose logarithms vanish,
  \begin{equation}
    \J^\star_t
    =
    \sum_{z_0,y,z}
    \cleanq(z_0)\,\q_{t\given0}(y\given z_0)\,\fQ{t}(y,z)
    \log\frac{\q_{0\given t}(z_0\given y)}{\q_{0\given t}(z_0\given z)}.
    \label{eq:oracle-integrand-flux-form}
  \end{equation}
  Conversely, differentiate $H(Z_0\given Z_t)=-\sum_{z_0,z} \cleanq(z_0)\,\q_{t\given0}(z\given z_0)\log \q_{0\given t}(z_0\given z)$. Differentiation passes through the finite sum; the derivative of the log terms cancels by normalization, $\sum_{z_0,z}\cleanq(z_0)\,\q_{t\given0}(z\given z_0)\,\partial_t\log \q_{0\given t}(z_0\given z) =\sum_z \q_t(z)\,\partial_t\sum_{z_0} \q_{0\given t}(z_0\given z)=0$, so inserting the Kolmogorov forward equation $\partial_t \q_{t\given0}(z\given z_0)=\sum_y \q_{t\given0}(y\given z_0)\fQ{t}(y,z)$ and adding the vanishing term $\sum_{z_0,y,z}\cleanq(z_0)\q_{t\given0}(y\given z_0)\fQ{t}(y,z)\log \q_{0\given t}(z_0\given y)=0$ (zero because $\sum_z \fQ{t}(y,z)=0$) gives
  \begin{equation*}
    \frac{d}{dt}H(Z_0\given Z_t)
    =
    \sum_{z_0,y,z}
    \cleanq(z_0)\,\q_{t\given0}(y\given z_0)\,\fQ{t}(y,z)
    \log\frac{\q_{0\given t}(z_0\given y)}{\q_{0\given t}(z_0\given z)},
  \end{equation*}
  which is~\eqref{eq:oracle-integrand-flux-form}. Nonnegativity is $\Phi\geq0$; the fixed-support assumption makes $H(Z_0\given Z_t)$ absolutely continuous on the window, so integration gives the cumulative form.
\end{proof}

\begin{corollary}[Monotonicity of retained information]\label{cor:mutual-information-monotonicity}
  The map $t\mapsto I(Z_0;Z_t)$ is nonincreasing, and wherever it is differentiable, $-\frac{d}{dt}I(Z_0;Z_t)=\J^\star_t\geq0$.
\end{corollary}
\begin{proof}
  For $s<t$ the Markov chain $Z_0\to Z_s\to Z_t$ and the data-processing inequality give $I(Z_0;Z_t)\leq I(Z_0;Z_s)$; the derivative statement is Proposition~\ref{prop:infoloss}.
\end{proof}

Figure~\ref{fig:toy-floor} shows both halves of the identity on the solvable model of App.~\ref{app:toy}: the common ceiling $H(\cleanq)$ reached along process-dependent paths, and the rate $\J^\star_t$ at which each process approaches it.

\begin{figure}[h]
  \centering
  \includegraphics[width=\linewidth]{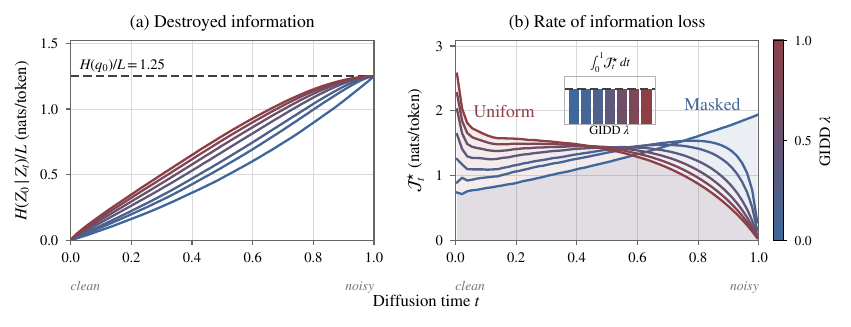}
  \caption{Every noising process destroys $H(\cleanq)$ information at a possibly different rate, computed at the exact oracle across the GIDD family. \textbf{(a)} Information destroyed by time $t$, $H(Z_0\given Z_t)/L$: different paths to the common ceiling $H(\cleanq)/L$. \textbf{(b)} Its rate $\J^\star_t$ (Proposition~\ref{prop:infoloss}): uniform destroys information earlier than masked, with equal area (inset).}
  \label{fig:toy-floor}
\end{figure}

\subsection{ELBO as distance to the oracle reverse process}\label{app:oracle-proofs}

\begin{theorem}[Oracle Distance, finite window]\label{thm:oracle-window}
  Let $[t_1,t_2]\subset(0,T)$ and let $\RevQ{t}$ and $\RevQth{t}$ be bounded and measurable. Assume the support conditions $\q_{0\given t_1}(\cdot\given z_{t_1})\ll \p^\theta_{0\given t_1}(\cdot\given z_{t_1})$, $\q_{t_2}\ll \p^\theta_{t_2}$, and $\RevQ{t}(z_t,\cdot)\ll\RevQth{t}(z_t,\cdot)$ a.e., and that the time integral in~\eqref{eq:pathkl-expanded} is finite. Then
  \begin{equation}
    \NELBO_{[t_1,t_2]}(\theta)-H(\cleanq)
    = \Exp{\q_{t_1}}{\KLdiv{\q_{0\given t_1}(\cdot\given z_{t_1})}{\p^\theta_{0\given t_1}(\cdot\given z_{t_1})}}
    + \KLdiv{\Pstar{[t_1,t_2]}}{\Ptheta{[t_1,t_2]}},
    \label{eq:oracle-distance-window}
  \end{equation}
  where
  \begin{equation}
    \KLdiv{\Pstar{[t_1,t_2]}}{\Ptheta{[t_1,t_2]}}
    = \KLdiv{\q_{t_2}}{\p^\theta_{t_2}}
    + \int_{t_1}^{t_2}\Exp{\q_t}{\sum_{y\ne z_t}\Phi\big(\RevQ{t}(z_t,y),\RevQth{t}(z_t,y)\big)}\,dt.
    \label{eq:pathkl-expanded}
  \end{equation}
\end{theorem}

\begin{proof}[Assembly proof of Theorem~\ref{thm:oracle-window}]
  Averaging the negative of~\eqref{eq:ctmc-elbo-full} over the data, $Z_0\sim\cleanq$, writes $\NELBO_{[t_1,t_2]}(\theta)$ as reconstruction, path, and terminal-prior terms, treated in turn. \emph{Reconstruction}: subtracting and adding $H(Z_0\given Z_{t_1})=\Exp{Z_0,Z_{t_1}}{-\log \q_{0\given t_1}(Z_0\given Z_{t_1})}$ gives
  \begin{equation*}
    \Exp{Z_0,Z_{t_1}}{-\log \p^\theta_{0\given t_1}(Z_0\given Z_{t_1})}
    = H(Z_0\given Z_{t_1})
    + \Exp{Z_{t_1}}{\KLdiv{\q_{0\given t_1}(\cdot\given Z_{t_1})}{\p^\theta_{0\given t_1}(\cdot\given Z_{t_1})}}.
  \end{equation*}
  \emph{Path}: Proposition~\ref{prop:pythagoras} splits the integrand into oracle plus mismatch, and Proposition~\ref{prop:infoloss} integrates the oracle part into $H(Z_0\given Z_{t_2})-H(Z_0\given Z_{t_1})$. \emph{Terminal prior}: splitting the log ratio through the marginal $\q_{t_2}$,
  \begin{align*}
    \Exp{Z_0}{\KLdiv{\q_{t_2\given0}(\cdot\given Z_0)}{\p^\theta_{t_2}}}
     & = I(Z_0;Z_{t_2}) + \KLdiv{\q_{t_2}}{\p^\theta_{t_2}}                 \\
     & = H(Z_0) - H(Z_0\given Z_{t_2}) + \KLdiv{\q_{t_2}}{\p^\theta_{t_2}}.
  \end{align*}

  Summing the three, the conditional entropies telescope into $H(Z_0)=H(\cleanq)$, leaving the reconstruction KL, the terminal KL, and the integrated mismatch, which is~\eqref{eq:oracle-distance-window} with the path KL expanded by Theorem~\ref{thm:girsanov}.
\end{proof}

\begin{proof}[Proof of Theorem~\ref{thm:oracle}]
  In Theorem~\ref{thm:oracle-window} let $t_1\downarrow0$ and $t_2\uparrow T$: the two boundary mismatches, the reconstruction KL and the terminal KL inside~\eqref{eq:pathkl-expanded}, vanish by hypothesis, and the path integrals converge to the integral in~\eqref{eq:oracle-distance} by monotone convergence, leaving the identity.
\end{proof}

Figure~\ref{fig:toy-infocore} verifies the identity and its per-time Pythagorean split on the solvable model of App.~\ref{app:toy}.

\begin{figure}[h]
  \centering
  \includegraphics[width=\linewidth]{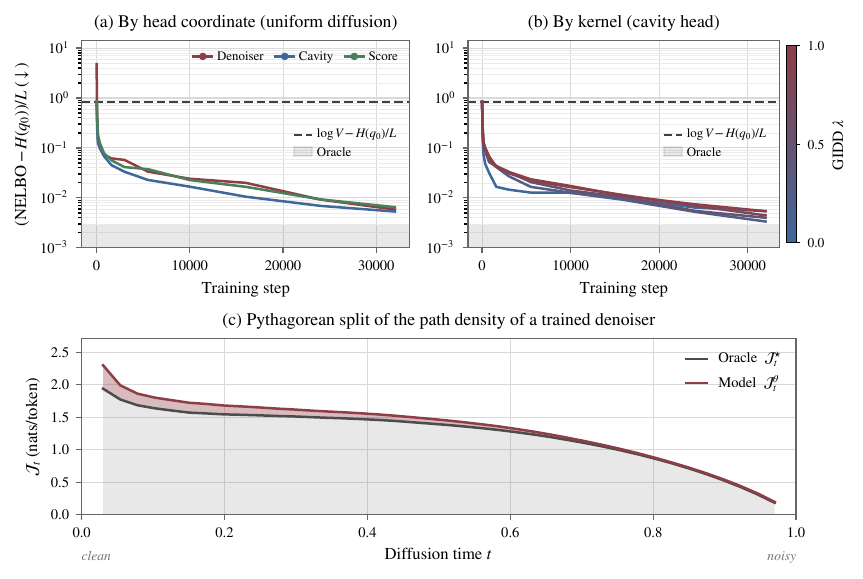}
  \caption{The NELBO is an exact distance to the oracle (Theorem~\ref{thm:oracle}). \textbf{(a, b)} Per-token excess $\big(\NELBO(\theta)-H(\cleanq)\big)/L$ over training falling to the oracle's own value on this window (shaded lower band), which is not zero but the residual boundary terms of the clipped window, for every head coordinate (a) and every noising process (b); the dashed line is the initialization value $\log V-H(\cleanq)/L$ of Proposition~\ref{prop:calibration}, where the cavity and score heads start. \textbf{(c)} The Pythagorean split~\eqref{eq:pythagoras} of the path density over diffusion time: path density $=$ oracle rate $\J^\star_t$ $+$ model--oracle divergence, here for a denoiser head trained for 300 steps.}
  \label{fig:toy-infocore}
\end{figure}

\subsection{The universal attainable NELBO floor}\label{app:floor}

\begin{lemma}[NELBO floor]\label{lem:universal-floor}
  For every noising process and model satisfying the conditions of Theorem~\ref{thm:elbo-full},
  \begin{equation}
    \NELBO_{[t_1,t_2]}(\theta)
    \ge
    \Exp{\cleanq}{-\log \p^\theta_0(z_0)}
    \ge
    H(\cleanq).
    \label{eq:universal-floor}
  \end{equation}
\end{lemma}
\begin{proof}
  The first inequality is the ELBO gap of Theorem~\ref{thm:elbo-full}; the second is Gibbs' inequality, $\Exp{\cleanq}{-\log \p^\theta_0} = H(\cleanq)+\KLdiv{\cleanq}{\p^\theta_0}$.
\end{proof}

\begin{corollary}[Universal floor, finite window]\label{cor:floor-window}
  In the setting of Theorem~\ref{thm:oracle-window}, if some $\thetastar$ realizes $\p^{\theta_\star}_{t_2}=\q_{t_2}$, $\modelC{\widehat{Q}^{\theta_\star}_t}=\RevQ{t}$ on $[t_1,t_2]$, and $\p^{\theta_\star}_{0\given t_1}=\q_{0\given t_1}$, then $\p^{\theta_\star}_0=\cleanq$, the ELBO is tight at $\thetastar$, and $\inf_\theta\NELBO_{[t_1,t_2]}(\theta)=\NELBO_{[t_1,t_2]}(\thetastar)=H(\cleanq)$.
\end{corollary}
\begin{proof}
  Evaluate Theorem~\ref{thm:oracle-window} at $\thetastar$: the reconstruction term vanishes because $\p^{\theta_\star}_{0\given t_1}=\q_{0\given t_1}$, and the path KL vanishes because $\Ptheta{[t_1,t_2]}=\Pstar{[t_1,t_2]}$, both initialized at $\q_{t_2}$ and run with $\RevQ{t}$. Hence $\NELBO_{[t_1,t_2]}(\thetastar)=H(\cleanq)$, which with Lemma~\ref{lem:universal-floor} gives the infimum. The lemma also writes $\NELBO_{[t_1,t_2]}(\thetastar)-H(\cleanq)$ as the sum of the expected ELBO gap and $\KLdiv{\cleanq}{\p^{\theta_\star}_0}$; both vanish, so $\p^{\theta_\star}_0=\cleanq$ and the ELBO is tight $\cleanq$-a.e. Under Assumption~\ref{ass:terminal-mixing} the conclusion passes to $t_2\uparrow T$, entropy and mutual information being continuous and bounded on the finite space, and with vanishing boundary mismatches this is Corollary~\ref{cor:floor}.
\end{proof}

\subsection{A direct proof of the Oracle Distance identity}\label{app:alt-proof}

The assembly proof uses Assumption~\ref{ass:regularity-fixed-support}, which the theorem itself does not need. The following proof avoids it but works with laws on the path space: the KL chain rule plays the role that the Pythagorean split played at the rate level, and since $\Phi$ is the KL rate between Poisson intensities, Proposition~\ref{prop:pythagoras} may be read as the infinitesimal, jump-by-jump version of the chain rule.

\begin{proof}[Direct proof of Theorem~\ref{thm:oracle-window}]
  The support hypotheses make every term below well defined: $\Pstar{[t_1,t_2]}\ll \Ptheta{[t_1,t_2]}$ by Theorem~\ref{thm:girsanov}, and with $\q_{0\given t_1}\ll \p^\theta_{0\given t_1}$ this gives $\Pstar{{0},[t_1,t_2]}\ll \Ptheta{{0},[t_1,t_2]}$.

  Negating~\eqref{eq:ctmc-elbo-path-definition} and averaging over $z_0\sim\cleanq$ expresses the average NELBO through the joint factorization of Definition~\ref{def:noising-path-laws}:
  \begin{equation*}
    \NELBO_{[t_1,t_2]}(\theta)
    = \int \log \frac{d\Pstar{[t_1,t_2]\given0}(\cdot\given z_0)}{d\Ptheta{{0},[t_1,t_2]}(z_0,\cdot)}(z_{[t_1,t_2]})\;
    \Pstar{{0},[t_1,t_2]}(z_0,dz_{[t_1,t_2]}).
  \end{equation*}
  Since $\cleanq$ is the $z_0$-marginal of $\Pstar{{0},[t_1,t_2]}$, subtracting $H(\cleanq)$ inserts the factor $\cleanq(z_0)$ into the numerator, and the integrand becomes the joint Radon--Nikodym derivative:
  \begin{equation*}
    \NELBO_{[t_1,t_2]}(\theta)-H(\cleanq)
    = \KLdiv{\Pstar{{0},[t_1,t_2]}}{\Ptheta{{0},[t_1,t_2]}}.
  \end{equation*}
  By the Markov property, $\Law(Z_0\given Z_{[t_1,t_2]}) = \q_{0\given t_1}(\cdot\given Z_{t_1})$, so the joint noising law also factors backward,
  \begin{equation}
    \Pstar{{0},[t_1,t_2]}(z_0,dz_{[t_1,t_2]})
    = \Pstar{[t_1,t_2]}(dz_{[t_1,t_2]})\, \q_{0\given t_1}(z_0\given z_{t_1}).
    \label{eq:noising-reverse-factorization}
  \end{equation}
  With both joint laws factored in the same order, the Radon--Nikodym derivative factorizes, and taking logs and integrating against $\Pstar{{0},[t_1,t_2]}$ splits the joint divergence into a window term and an endpoint-kernel term:
  \begin{equation*}
    \KLdiv{\Pstar{{0},[t_1,t_2]}}{\Ptheta{{0},[t_1,t_2]}}
    = \KLdiv{\Pstar{[t_1,t_2]}}{\Ptheta{[t_1,t_2]}}
    + \Exp{\q_{t_1}}{\KLdiv{\q_{0\given t_1}(\cdot\given z_{t_1})}{\p^\theta_{0\given t_1}(\cdot\given z_{t_1})}},
  \end{equation*}
  which is~\eqref{eq:oracle-distance-window}; Theorem~\ref{thm:girsanov} expands the path KL into~\eqref{eq:pathkl-expanded}.
\end{proof}

\subsection{The arrow of time and independent mechanisms}\label{app:arrow}

At the path-law level there is no asymmetry between the two time directions. In the forward direction the description factorizes into two independent mechanisms, an initial condition (the data law $\cleanq$) and a dynamical law (the generator $\fQ{t}$), chosen without knowledge of each other. Backward, this independence is destroyed: the reverse generator $\RevQ{t}(x,y)=\fQ{t}(y,x)\,\q_t(y)/\q_t(x)$ requires the marginals and hence the data law. \citet{janzing2016arrow} derive the thermodynamic arrow of time from precisely this independence of initial state and dynamics, and the same asymmetry-through-restriction underlies causal direction detection in time series \citep{peters2009detecting,peters2013causal}; see also \citet{schoelkopf2012causal}. Discrete diffusion is an instance: the forward process belongs to the small class of data-independent, position-factorized mechanisms, whereas its reversal, whenever $\cleanq$ is correlated across positions, exists only in the larger class of data-dependent generators, and the purpose of training is to supply that dependence. The framework quantifies the information the reverse mechanism must restore: the forward mechanism destroys $I(Z_0;Z_t)$ at the rate $\J^\star_t=-\frac{d}{dt}I(Z_0;Z_t)$ of Proposition~\ref{prop:infoloss}, the total amount destroyed is conserved, the floor $H(\cleanq)$ of Corollary~\ref{cor:floor}, and the choice of noising process controls only when along the trajectory it is lost.

\subsection{A Bregman-divergence and generator-matching view}\label{app:bregman}

These results are instances of standard Bregman identities. A strictly convex differentiable potential $F$ generates $D_F(a,b)=F(a)-F(b)-F'(b)(a-b)$; the divergence $\Phi$ is generated by $F(u)=u\log u-u$. For any $D_F$, the conditional mean minimizes $\E[D_F(A,b(G))]$ over $G$-measurable predictors and satisfies the Pythagorean identity \citep{banerjee2005bregman}; Lemma~\ref{lem:bregman-pythagoras} specializes this to $\Phi$, and the optimal rate is the marginal generator obtained by posterior-averaging the clean-conditioned generators, as in Generator Matching \citep{holderrieth2025generator}. It is also the jump-process analogue of flow matching \citep{gat2024discreteflow}: the regressed object is a reverse intensity and the divergence is $\Phi$ rather than the squared norm. The correspondence with continuous diffusion is exact.

\begin{remark}[The continuous-diffusion analogue]\label{rmk:continuous-analogue}
  Let the noising process be a Markov diffusion $dZ_t=f_t(Z_t)\,dt+\sigma_t(Z_t)\,dW_t$ on $\R^d$ with $Z_0\sim\cleanq$ and $a_t:=\sigma_t\sigma_t^{\mathsf T}$. Under standard smoothness and nondegeneracy assumptions the reverse-time process is again a diffusion with the same diffusion matrix and, at forward time $t$, drift $\dataC{\widehat b_t}(z) = -f_t(z) + \nabla\!\cdot a_t(z) + a_t(z)\,\nabla\log \q_t(z)$; conditioning on the data replaces the marginal by the forward kernel. The score $\nabla\log \q_t$ plays the role of $\RevQ{t}$ and $\nabla\log \q_{t\given0}$ that of the computable clean-conditioned target. For reverse processes sharing the diffusion matrix, Girsanov's theorem gives
  \begin{equation*}
    \KLdiv{\Pstar{[0,T]}}{\Ptheta{[0,T]}}
    = \KLdiv{\q_T}{\p^\theta_T}
    + \frac12\int_0^T \Exp{\q_t}{\Big(\big(\dataC{\widehat b_t}-\modelC{\widehat b^\theta_t}\big)^{\mathsf T} a_t^{-1}\,\big(\dataC{\widehat b_t}-\modelC{\widehat b^\theta_t}\big)\Big)(Z_t)}\,dt,
  \end{equation*}
  the projection becomes a weighted $L^2$ projection whose optimizer among $Z_t$-measurable drifts is the conditional mean, the drift form of Fisher's identity $\nabla\log \q_t(z)=\Exp{}{\nabla\log \q_{t\given0}(z\given Z_0)\given Z_t=z}$, and with matched endpoints
  \begin{equation*}
    \NELBO_{[0,T]}(\theta)
    = H(\cleanq)
    + \frac12\int_0^T \Exp{\q_t}{\big\|\sigma_t^{-1}\big(\dataC{\widehat b_t}-\modelC{\widehat b^\theta_t}\big)(Z_t)\big\|^2}\,dt,
  \end{equation*}
  with $H(\cleanq)$ the differential entropy. Training on a window with a terminal prior and a reconstruction decoder gives the exact analogue of Theorem~\ref{thm:oracle}, the terminal mismatch again sitting inside the path KL. Table~\ref{tab:ctmc-continuous-dictionary} collects the correspondence.
\end{remark}

\begin{table}[ht]
  \caption{The discrete/continuous correspondence: the same path structure throughout, with jump rates and the rate divergence $\Phi$ replaced by reverse drifts and the quadratic Girsanov energy.}
  \label{tab:ctmc-continuous-dictionary}
  \centering\small
  \renewcommand{\arraystretch}{1.35}
  \setlength{\tabcolsep}{6pt}
  \begin{tabular}{@{}p{3.9cm} p{4.3cm} p{4.9cm}@{}}
    \toprule
                                 & Finite-state CTMC                                                                         & Continuous diffusion                                                            \\
    \midrule
    Forward infinitesimal object & generator $\fQ{t}(x,y)$                                                                   & generator $f_t\cdot\nabla+\tfrac12 a_t{:}\nabla^2$                              \\
    True reverse object          & reverse rate $\RevQ{t}(z_t,y)$                                                            & reverse drift $\dataC{\widehat b_t}(z)$, or score $\nabla\log \q_t(z)$          \\
    Tractable conditional target & $\RevQ{t}(z_t,y\given z_0)$                                                               & $\dataC{\widehat b_t}(z\given z_0)$, or $\nabla\log \q_{t\given0}(z\given z_0)$ \\
    Path-KL divergence           & $\Phi(a,b)=a\log\tfrac ab-a+b$                                                            & $\tfrac12(u-v)^{\mathsf T}a_t^{-1}(u-v)$                                        \\
    Path-KL formula              & Theorem~\ref{thm:girsanov}                                                                & Girsanov's theorem                                                              \\
    Population optimizer         & $\Exp{}{\RevQ{t}(z_t,y\given Z_0) \given Z_t=z_t}$                                        & $\Exp{}{\dataC{\widehat b_t}(z\given Z_0) \given Z_t=z}$                        \\
    NELBO meaning                & \multicolumn{2}{c@{}}{data entropy $+$ reverse path KL, plus the two boundary mismatches}                                                                                   \\
    \bottomrule
  \end{tabular}
\end{table}

\section{Sequence modeling}\label{app:sequences}

\subsection{Token-factorizing noising processes}\label{app:product-ctmc}

In this section $\dataC{Q}$ denotes sequence-level rates and $\dataC{R}$ single-token rates. Forward noising factors over positions; correlations in the clean-data law mean that the unconditional marginals $\q_t$ do not, and only the clean-conditioned kernels factor.

\begin{theorem}[Product CTMCs]\label{thm:product-ctmc}
  Suppose position $1\le i \le L$ has generator $\fR[i]{t}$ on $\X_i$ with transition kernel $\q^i_{t\given s}$, and define the product kernel $\q_{t\given s}(y\given x):=\prod_{i=1}^L \q^i_{t\given s}(y^i\given x^i)$ on $\X=\X_1\times\cdots\times\X_L$. Then its unique infinitesimal generator is
  \begin{equation*}
    \fQ{t}(x,y) = \sum_{i=1}^L (\fQ{t})_i(x,y), \qquad (\fQ{t})_i(x,y) = \one\{x^{-i}=y^{-i}\}\,\fR[i]{t}(x^i,y^i).
  \end{equation*}
  Conversely, any process with a generator of this form has the product transition kernel, that is, it factors over positions.
\end{theorem}
\begin{proof}
  First, $\fQ{t}$ is a generator: off-diagonal entries are nonnegative, and
  \begin{equation*}
    \sum_{y\in\X}\fQ{t}(x,y)
    = \sum_{i=1}^L \sum_{y\in\X} \one\{x^{-i}=y^{-i}\}\,\fR[i]{t}(x^i,y^i)
    = \sum_{i=1}^L\sum_{u\in\X_i}\fR[i]{t}(x^i,u) = 0.
  \end{equation*}
  Second, the product kernel satisfies the forward Kolmogorov equation for $\fQ{t}$: since each position's kernel satisfies its own, $\partial_t \q^i_{t\given s}(y^i\given x^i) = \sum_{u\in\X_i}\q^i_{t\given s}(u\given x^i)\fR[i]{t}(u,y^i)$, differentiating the product gives
  \begin{equation*}
    \partial_t \q_{t\given s}(y\given x)
    = \sum_{i=1}^L \sum_{u\in\X_i} \q^i_{t\given s}(u\given x^i)\fR[i]{t}(u,y^i) \prod_{j\neq i}\q^j_{t\given s}(y^j\given x^j),
  \end{equation*}
  and expanding $\sum_{z} \q_{t\given s}(z\given x)\fQ{t}(z,y)$ with the definition of $\fQ{t}$ produces the same expression. The initial condition $\q_{s\given s}(y\given x)=\one\{x=y\}$ holds, so by uniqueness of solutions of linear ODEs the product kernel is the kernel of $\fQ{t}$. For the converse, a process with a generator of the displayed form solves the same forward equation, whose solution is unique and equal to the product kernel.
\end{proof}

The clean-conditioned reverse process still factors: for $s<t$, Bayes' rule gives
\begin{equation*}
  \q_{s\given t,0}(z_s\given z_t,z_0)
  = \frac{\q_{s\given0}(z_s\given z_0)\,\q_{t\given s}(z_t\given z_s)}{\q_{t\given0}(z_t\given z_0)}
  = \prod_{i=1}^L \q^i_{s\given t,0}(z_s^i\given z_t^i,z_0^i).
\end{equation*}
Applying Theorem~\ref{thm:product-ctmc} to this factorized backward process gives the token-wise clean-conditioned reverse rates: for a one-token jump $z_t\to y$ with $y^i\ne z_t^i$,
\begin{equation}
  \dataC{a_i}(t,z_0,z_t,y)
  := (\RevQ{t})_i(z_t, y\given z_0) = \one\{z_t^{-i}=y^{-i}\}\,\fR[i]{t}(y^i,z_t^i)\,
  \frac{\q^i_{t\given0}(y^i\given z_0^i)} {\q^i_{t\given0}(z_t^i\given z_0^i)},
  \label{eq:reverse-bridge-rate-definition}
\end{equation}
the target rates inside the token-wise ELBO.

\subsection{Three coordinates of the reverse rate}\label{app:three-laws}

Fix a position $i$ and write $e_\ell$ for the vocabulary elements. Bayes' rule expands the denoiser~\eqref{eq:three-laws} so as to isolate the Bayes factor from the rest of the sequence:
\begin{equation}
  \pistar_i(e_\ell \given z_t,t)
  = \frac{\q^i_{t\given0}(z_t^i\given e_\ell)\, \sum_{z_0^{-i}} \cleanq(e_\ell,z_0^{-i})\prod_{j\neq i} \q^j_{t\given0}(z_t^j\given z_0^j)}
  {\sum_r \q^i_{t\given0}(z_t^i\given e_r)\, \sum_{z_0^{-i}} \cleanq(e_r,z_0^{-i})\prod_{j\neq i} \q^j_{t\given0}(z_t^j\given z_0^j)},
  \label{eq:denoiser-expansion}
\end{equation}
where $(e_\ell,z_0^{-i})$ is the sequence with $i$-th token $e_\ell$. The common inner sums are the cavity law: expanding $\mustar_i$~\eqref{eq:three-laws} the same way,
\begin{equation*}
  \mustar_i(e_\ell\given z_t^{-i},t)
  =
  \frac{ \sum_{z_0^{-i}} \cleanq(e_\ell,z_0^{-i})\prod_{j\neq i}\q^j_{t\given0}(z_t^j\given z_0^j) }
  { \sum_r \sum_{z_0^{-i}} \cleanq(e_r,z_0^{-i})\prod_{j\neq i}\q^j_{t\given0}(z_t^j\given z_0^j) },
\end{equation*}
so the denoiser is the cavity reweighted by the local forward kernel,
\begin{equation}
  \pistar_i(e_\ell \given z_t,t) \propto \mustar_i(e_\ell\given z_t^{-i},t)\,\q^i_{t\given0}(z_t^i\given e_\ell):
  \label{eq:cavity-to-denoiser-propto}
\end{equation}
the local kernel is exactly the Bayes factor for observing $z_t^i$ on top of the context. Equivalently, for any $f$,
\begin{equation}
  \Exp{\pistar_i(e_\ell \given z_t,t)}{f(e_\ell)}
  = \frac{ \Exp{\mustar_i(e_\ell\given z_t^{-i},t)}{f(e_\ell)\,\q^i_{t\given0}(z_t^i\given e_\ell)} }
  { \Exp{ \mustar_i(e_\ell\given z_t^{-i},t)}{\q^i_{t\given0}(z_t^i\given e_\ell)} },
  \label{eq:cavity-to-denoiser-expectation}
\end{equation}
and the denominator is the local noisy-token likelihood given the context:
\begin{equation}
  \Exp{\mustar_i(e_\ell\given z_t^{-i},t)}{\q^i_{t\given0}(w \given e_\ell)}
  = \q(Z_t^i=w\given z_t^{-i}).
  \label{eq:cavity-marginal-identity}
\end{equation}

\begin{remark}[Support of the denoiser law]\label{rem:support}
  Here and throughout, terms with both $\pistar_i(e_\ell\given z_t,t)=0$ and $\q^i_{t\given0}(z_t^i\given e_\ell)=0$ contribute zero. Defining the reverse support $\mathcal S_i(z_t^i,t) :=\{\ell:\q^i_{t\given0}(z_t^i\given e_\ell)>0\}$, Bayes' rule gives $\supp \pistar_i(\cdot\given z_t,t)\subseteq \mathcal S_i(z_t^i,t)$ at every reachable state.
\end{remark}

\begin{thmrestate}{Theorem~\ref{thm:coords}}{Three coordinates}
  For a token-factorizing noising process the NELBO is uniquely minimized by the marginal reverse rate, whose posterior average on a one-token jump has the three equal forms of \eqref{eq:coords}.
\end{thmrestate}
\begin{proof}[Proof of Theorem~\ref{thm:coords}]
  Uniqueness of the optimizer is Corollary~\ref{cor:projection}. For the three forms, average the clean-conditioned rate~\eqref{eq:reverse-bridge-rate-definition} over $\q(z_0\given z_t)$: the product kernels cancel at every position except $i$, leaving
  \begin{equation*}
    (\RevQ{t})_i(z_t, y)
    = \one\{z_t^{-i}=y^{-i}\}\, \fR[i]{t}(y^i,z_t^i)\,
    \sum_{\ell}
    \pistar_i(e_\ell\given z_t,t)\,
    \frac{\q^i_{t\given0}(y^i\given e_\ell)}{\q^i_{t\given0}(z_t^i\given e_\ell)},
  \end{equation*}
  the denoiser form of \eqref{eq:coords}. Applying~\eqref{eq:cavity-to-denoiser-expectation} with $f(e_\ell)=\q^i_{t\given0}(y^i\given e_\ell)/\q^i_{t\given0}(z_t^i\given e_\ell)$ turns it into the cavity form, and~\eqref{eq:cavity-marginal-identity} identifies the ratio of cavity-averaged kernels with the ratio of noisy conditionals, the score form.
\end{proof}

\begin{thmrestate}{Proposition~\ref{prop:conversions}}{Coordinate conversions}
  The three heads are related by~\eqref{eq:bayes-tie} and~\eqref{eq:score-conversions}.
\end{thmrestate}
\begin{proof}[Proof of Proposition~\ref{prop:conversions}]
  The cavity-to-denoiser map is the Bayes reweighting~\eqref{eq:cavity-to-denoiser-propto}; inverting it gives $\muth_i\propto\pith_i/\q^i_{t\given0}(z_t^i\given e_\ell)$ wherever that kernel is positive.
  By the definition of the score and~\eqref{eq:cavity-marginal-identity}, the score is the ratio $\mth_i(y^i)/\mth_i(z_t^i)$ of the cavity-induced marginal $\mth_i=\muth_i\,\q^i_{t\given0}$, so either distribution head determines it.
  Conversely a score fixes $\mth_i$ up to scale, recovered as a probability vector by normalizing $\sth_i(z_t^i,\cdot)$, and $\muth_i$ solves the linear system $\mth_i=\muth_i\,\q^i_{t\given0}$, which has a solution iff $\mth_i$ lies in the image of the channel, unique iff the channel has full row rank.
  For the GIDD kernel, $\q^i_{t\given0}=\alpha_t I+\beta_t\one\omega_t$ decouples across $\ell$ (the noise floor is source-independent), giving the affine inverse $\muth_i(e_\ell)=\big(\mth_i(e_\ell)-\beta_t\omega_t(e_\ell)\big)/\alpha_t$, with full row rank for $\alpha_t>0$.
\end{proof}
The maps into the score are always available; recovering a distribution coordinate from a score needs an invertible channel and a nonnegative solution. Where $\q^i_{t\given0}(z_t^i\given e_\ell)$ vanishes for some $\ell$, the visible states of masked diffusion, the denoiser cannot recover the corresponding cavity mass. The inversion is the discrete analogue of empirical-Bayes deconvolution: the prior (cavity) is recovered by inverting the noising channel. Table~\ref{tab:coordinate-conversions} collects the generic maps; Table~\ref{tab:gidd-conversions} instantiates them for GIDD.

\paragraph{The three parameterizations, their rates, and their ancestral steps.}
Substituting a head for its law in Theorem~\ref{thm:coords} defines the three model rates; each determines the reverse CTMC in full through the Kolmogorov equation. The finite-step samplers share one template, the token-factorized ansatz $\p^\theta_{s\given t}(z_s\given z_t):=\prod_{i=1}^L \p^\theta(z_s^i\given z_t)$, whose factors are the exact per-position marginals $\q(z_s^i\given z_t)$ at the oracle laws and which is approximate only through the forced product over positions: the per-position posteriors are frozen within the step, so simultaneous jumps are sampled independently, an $O((t-s)^2)$ effect.

\emph{Denoiser} \citep{austin2021d3pm,hoogeboom2021argmax,campbell2022continuous,sahoo2024mdlm,shi2024simplified}. A head $\pith_i$ gives the rate
\begin{equation}
  \modelC{a_{i,\mathrm{den}}^\theta}(t,z_t,y)
  := \one\{z_t^{-i}=y^{-i}\}\,\fR[i]{t}(y^i,z_t^i)
  \sum_\ell \pith_i(e_\ell\given z_t,t)\,
  \frac{\q^i_{t\given0}(y^i\given e_\ell)}{\q^i_{t\given0}(z_t^i\given e_\ell)}.
  \label{eq:denoiser-parameterization}
\end{equation}
Bayes' rule on the single-token bridge gives the exact marginal $\q(z_s^i{=}y\given z_t)=\q^i_{t\given s}(z_t^i\given y)\sum_\ell\pistar_i(e_\ell\given z_t,t)\,\q^i_{s\given0}(y\given e_\ell)/\q^i_{t\given0}(z_t^i\given e_\ell)$, hence the finite-step sampler
\begin{equation}
  \p^\theta_{\mathrm{den}}(Z_s^i=y\given z_t)
  :=\q^i_{t\given s}(z_t^i\given y)
  \sum_\ell\pith_i(e_\ell\given z_t,t)\,
  \frac{\q^i_{s\given0}(y\given e_\ell)}{\q^i_{t\given0}(z_t^i\given e_\ell)},
  \label{eq:denoiser-ancestral-step}
\end{equation}
the marginalization form. For a learned head the inclusion of Remark~\ref{rem:support} is not guaranteed: a denoiser head must learn to vanish off $\mathcal S_i(z_t^i,t)$, or be masked to it, to keep the rate finite. In masked diffusion the restriction is automatic (the support at a masked position is the whole vocabulary, and a revealed position admits no reverse jump), while in uniform diffusion the support collapses only at $t=0$, so a denoiser head must concentrate on the true token exactly as $t\downarrow0$, which is where an uninformative one diverges (Proposition~\ref{prop:calibration}).

\emph{Cavity} \citep{sun2022score,vonruette2025gidd,schiff2025udlm,sahoo2025diffusionduality}. A head $\muth_i$ induces the noisy laws at times $t$ and $s$,
\begin{equation}
  \begin{gathered}
    \mth_i(y\given z_t^{-i},t)
    :=\sum_\ell\muth_i(e_\ell\given z_t^{-i},t)\,\q^i_{t\given0}(y\given e_\ell),
    \\
    \mth_{i,t\to s}(y\given z_t^{-i},t)
    :=\sum_\ell\muth_i(e_\ell\given z_t^{-i},t)\,\q^i_{s\given0}(y\given e_\ell),
  \end{gathered}
  \label{eq:cross-time-cavity-marginal}
\end{equation}
hence the rate and the plug-in step
\begin{equation}
  \begin{gathered}
    \modelC{a_{i,\mathrm{cav}}^\theta}(t,z_t,y)
    :=\one\{z_t^{-i}=y^{-i}\}\,\fR[i]{t}(y^i,z_t^i)\,
    \frac{\mth_i(y^i\given z_t^{-i},t)}{\mth_i(z_t^i\given z_t^{-i},t)},
    \\
    \p^\theta_{\mathrm{cav}}(Z_s^i=y\given z_t)
    :=\q^i_{t\given s}(z_t^i\given y)\,
    \frac{\mth_{i,t\to s}(y\given z_t^{-i},t)}{\mth_i(z_t^i\given z_t^{-i},t)}.
  \end{gathered}
  \label{eq:cavity-parameterization}
\end{equation}
The rate never divides by a per-token likelihood, and out-of-support tokens vanish on their own through the Bayes reweighting, so a cavity head is automatically confined to $\mathcal S_i$. Although often presented as a clean-token denoiser, this head is ELBO-optimized by the cavity law.

\emph{Score} \citep{meng2022concrete,sun2022score,lou2024sedd,ou2025radd,zhang2025tcsm}. A head $\sth_i$ gives the rate
\begin{equation}
  \modelC{a_{i,\mathrm{score}}^\theta}(t,z_t,y)
  :=\one\{z_t^{-i}=y^{-i}\}\,\fR[i]{t}(y^i,z_t^i)\,\sth_i(z_t^i,y^i\given z_t^{-i},t),
  \label{eq:score-parameterization}
\end{equation}
whose path loss is exactly SEDD's diffusion-weighted denoising score entropy, with the forward rate a forced rather than auxiliary weight (App.~\ref{app:processes}). For a finite step, a score head is converted to a distribution coordinate (Proposition~\ref{prop:conversions}) and sampled as one of those.

\begin{table}[ht]
  \caption{Conversions between coordinates (Proposition~\ref{prop:conversions}), with $\mth_i(y^i)=\sum_\ell\muth_i(e_\ell)\,\q^i_{t\given0}(y^i\given e_\ell)$. $^\ast$Wherever $\q^i_{t\given0}(z_t^i\given e_\ell)>0$. $^\dagger$Recover $\mth_i(\cdot)\propto \sth_i(z_t^i,\cdot)$ by normalization, then solve the linear system; the inverse is unique iff the channel $\q^i_{t\given0}$ has full row rank.}
  \label{tab:coordinate-conversions}
  \centering\small
  \renewcommand{\arraystretch}{1.5}
  \setlength{\tabcolsep}{5pt}
  \begin{tabular}{@{}l l l l@{}}
    \toprule
    From\,$\backslash$\,to & denoiser $\pith_i$
                           & cavity $\muth_i$
                           & score $\sth_i$
    \\
    \midrule
    denoiser $\pith_i$     & ---
                           & $\propto^\ast\dfrac{\pith_i(e_\ell)}{\q^i_{t\given0}(z_t^i\given e_\ell)}$
                           & $\sum_\ell\pith_i(e_\ell)\,\dfrac{\q^i_{t\given0}(y^i\given e_\ell)}{\q^i_{t\given0}(z_t^i\given e_\ell)}$
    \\
    cavity $\muth_i$       & $\propto\muth_i(e_\ell)\,\q^i_{t\given0}(z_t^i\given e_\ell)$
                           & ---
                           & $\mth_i(y^i)/\mth_i(z_t^i)$
    \\
    score $\sth_i$         & via $\muth_i$, then the cavity row
                           & solve$^\dagger$ $\mth_i=\muth_i\,{\q^i_{t\given0}}$
                           & ---
    \\
    \bottomrule
  \end{tabular}
\end{table}

\paragraph{Calibration of the cavity and score heads.}
Define the \emph{uniform-forward} boundary laws: $p^{\mathrm{unif\text{-}fwd}}_0$ is the uniform law on $\Y^L$, $p^{\mathrm{unif\text{-}fwd}}_t$ its image under $\q_{t\given0}$, and the boundary choices are $\p^\theta_{t_2}=p^{\mathrm{unif\text{-}fwd}}_{t_2}$ and $\p^\theta_{0\given t_1}=p^{\mathrm{unif\text{-}fwd}}_{0\given t_1}$. These are the support-matched laws referenced in Proposition~\ref{prop:calibration}.

\begin{thmrestate}{Proposition~\ref{prop:calibration}}{Calibration at initialization}
  With the support-matched uniform-forward boundary laws, per token: a cavity head $\equiv1/V$ has NELBO $\log V$ for every token-factorizing process and window; a score head at zero logits under the folded parameterization represents the same model and gives $\log V$; and a denoiser head $\equiv1/V$ under uniform diffusion gives $\tfrac{V-1}{V}\log(1/\beta_{t_1})+O(1)$, divergent as $t_1\downarrow0$.
\end{thmrestate}
\begin{proof}[Proof of Proposition~\ref{prop:calibration}, cavity head]
  It suffices to treat one position. The uniform cavity head induces the noisy marginal $\mth(y\given z_t^{-i},t)=\tfrac1V\sum_{\ell}\q_{t\given0}(y\given e_\ell)=p^{\mathrm{unif\text{-}fwd}}_t(y)$, so by~\eqref{eq:cavity-parameterization} the model rate equals the oracle reverse rate of the noising process started from $p^{\mathrm{unif\text{-}fwd}}_0$. With the matching boundary laws, Corollary~\ref{cor:floor-window} applied to that process makes the ELBO tight: $\NELBO_{[t_1,t_2]}(\theta;\,z_0)=-\log p^{\mathrm{unif\text{-}fwd}}_0(z_0)=L\log V$ for every $z_0\in\Y^L$. With different boundary laws the path-rate calibration still holds but the full NELBO need not equal $L\log V$; under terminal mixing the usual terminal law coincides with $p^{\mathrm{unif\text{-}fwd}}_{t_2}$ for an i.i.d.\ token-wise kernel.
\end{proof}

\begin{proof}[Proof of Proposition~\ref{prop:calibration}, score head]
  By the folded parameterizations~\eqref{eq:masked-score-head} and~\eqref{eq:gidd-score-head} of App.~\ref{app:processes}, the score head represents $\sth_i(z_t^i,y)=s^{u}_i(z_t^i,y)\,e^{\modelC{f_y^\theta}}$, where $s^u_i$ is the score induced by the uniform cavity through~\eqref{eq:cavity-parameterization} (App.~\ref{app:processes} gives $s^u_i$ per kernel). At $\modelC{f^\theta}\equiv0$ the head equals $s^u_i$ exactly, which by Theorem~\ref{thm:coords} is the score coordinate of the same reverse process as the uniform cavity head; the rates coincide, the boundary laws are shared, and the NELBO is again $L\log V$.
\end{proof}
The proof for the denoiser head is kernel-specific and given in App.~\ref{app:processes} (Proposition~\ref{prop:uniform-denoiser-diverges} and Remark~\ref{rmk:divergence-gidd}).
Figure~\ref{fig:toy-calibration} confirms the calibration identities numerically, together with the boundary-term split of the oracle NELBO.

\begin{figure}[h]
  \centering
  \includegraphics[width=\linewidth]{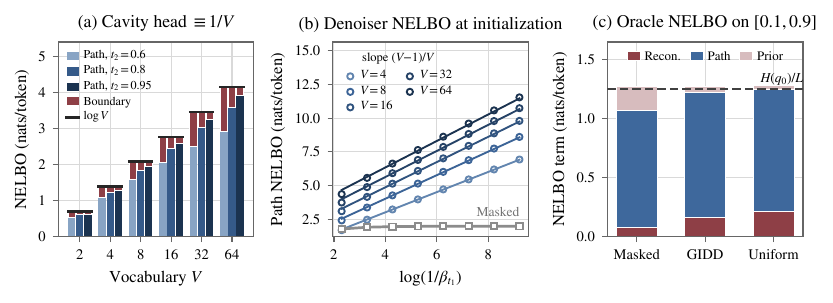}
  \caption{Calibration at initialization, exact on the solvable model. \textbf{(a)} A cavity head $\equiv1/V$ has per-token NELBO $\log V$ once boundary terms are kept, at every $V$ and every $t_2$ (Proposition~\ref{prop:calibration}). \textbf{(b)} A $1/V$ denoiser head diverges in $\log(1/\beta_{t_1})$ under uniform noising, with asymptotic slope $\tfrac{V-1}{V}$, while masked stays finite. \textbf{(c)} The oracle NELBO splits as reconstruction $+$ path $+$ terminal prior, the latter taken against the support-matched uniform-forward prior of Proposition~\ref{prop:calibration} (at masked it is exactly $\alpha_{t_2}\log V$ by~\eqref{eq:masked-uniform-forward-prior-kl}, whereas an absorbing $\delta_{\mask}$ prior would make it infinite on any window with $t_2<T$). The three terms sum to $H(\cleanq)$ for every process up to the finite-window residual (window $[0.1,0.9]$, chosen to exaggerate the boundary terms).}
  \label{fig:toy-calibration}
\end{figure}

\section{Masked, uniform, and GIDD noising processes}\label{app:processes}

The noisy alphabet $\Zs$ includes a mask state $\mask$ for masked and GIDD kernels; uniform diffusion uses the clean vocabulary $\Y$, $V=|\Y|$, shared across positions. All processes take $T=1$ and use a differentiable, decreasing signal schedule with $\alpha_0=1$, $\alpha_1=0$,
\begin{equation*}
  \beta_t:=1-\alpha_t,
  \qquad
  \rho_t:=-\frac{\alpha_t'}{\alpha_t}.
\end{equation*}

Table~\ref{tab:process-summary} collects the resulting rates, integrands, and boundary terms.

\subsection{Common CTMC form}\label{app:common-form}

All forward kernels factor over positions (Theorem~\ref{thm:product-ctmc}); we use the true reverse rate $\dataC{a_i}$ of~\eqref{eq:reverse-bridge-rate-definition} and the model rates of~\eqref{eq:denoiser-parameterization}--\eqref{eq:score-parameterization}, suppress unchanged positions, and write the candidate token $y^i$ as $y$. Leaving the boundary terms implicit, the path integrand is
\begin{equation}
  \begin{gathered}
    \J^\theta_t(z_0,z_t)
    =
    \sum_{i=1}^L
    \sum_{y\in\mathcal N_i(z_t)}
    \Phi\!\left(\dataC{a_i}(t,z_0,z_t,y),\,\modelC{a_i^\theta}(t,z_t,y)\right),
    \\
    \mathcal N_i(z):=\{y:y^{-i}=z^{-i},\ y^i\neq z^i\}.
  \end{gathered}
  \label{eq:token-space-ctmc-loss}
\end{equation}

\begin{proposition}[Cavity ELBO for source-independent jump rates]\label{prop:source-independent-cavity-elbo}
  Consider a product CTMC whose off-diagonal rates at position $i$ depend only on time and destination, $\fR[i]{t}(y,x)=\dataC{\gamma_i}(t,x)$ for $x\neq y$ (the relevant value below being at the current noisy token). Fix $z_0,z_t$ and define the true and model noisy marginals
  \begin{equation*}
    \dataC{d_i}(y):=\q^i_{t\given0}(y\given z_0^i),
    \qquad
    \mth_i(y):=\mth_i(y\given z_t^{-i},t).
  \end{equation*}
  If $\dataC{d_i}(z_t^i),\mth_i(z_t^i)>0$ and Assumption~\ref{ass:support} holds, the per-position cavity-coordinate integrand is
  \begin{equation}
    \J_{i,t}^{\mathrm{cav}}
    =
    \frac{\dataC{\gamma_i}(t,z_t^i)}{\dataC{d_i}(z_t^i)}
    \left[ \KLdiv{\dataC{d_i}}{\mth_i} + \ISdiv{\dataC{d_i}(z_t^i)}{\mth_i(z_t^i)} \right],
    \qquad
    \ISdiv{p}{q}:=\frac{p}{q}-\log\frac{p}{q}-1.
    \label{eq:source-independent-cavity-nelbo}
  \end{equation}
\end{proposition}
\begin{proof}
  Write $x:=z_t^i$, $\dataC{d}:=\dataC{d_i}$, $\mth:=\mth_i$. Equations~\eqref{eq:reverse-bridge-rate-definition} and~\eqref{eq:cavity-parameterization}, positive homogeneity of $\Phi$, and source-independence give $\J_{i,t}^{\mathrm{cav}} = \dataC{\gamma_i}(t,x) \sum_{y\neq x} \Phi\big(\dataC{d}(y)/\dataC{d}(x),\, \mth(y)/\mth(x)\big)$. Since $\dataC{d}$ and $\mth$ are probability vectors, $\sum_{y\neq x} \dataC{d}(y)=1-\dataC{d}(x)$ and likewise for $\mth$, so
  \begin{align*}
    \sum_{y\neq x} \Phi\!\left( \frac{\dataC{d}(y)}{\dataC{d}(x)}, \frac{\mth(y)}{\mth(x)} \right)
     & = \frac{1-\mth(x)}{\mth(x)} - \frac{1-\dataC{d}(x)}{\dataC{d}(x)}
    + \frac1{\dataC{d}(x)} \sum_{y\neq x} \dataC{d}(y)\log\frac{\dataC{d}(y)\,\mth(x)}{\dataC{d}(x)\,\mth(y)} \\
     & =
    \frac1{\dataC{d}(x)}
    \left[
      \sum_{y} \dataC{d}(y)\log\frac{\dataC{d}(y)}{\mth(y)}
      + \frac{\dataC{d}(x)}{\mth(x)} - \log\frac{\dataC{d}(x)}{\mth(x)} - 1
      \right]                                    \\
     & =
    \frac{\KLdiv{\dataC{d}}{\mth} + \ISdiv{\dataC{d}(x)}{\mth(x)}}{\dataC{d}(x)}.
  \end{align*}
  Multiplying by $\dataC{\gamma_i}(t,x)$ proves the claim; summing over $i$ gives the sequence integrand.
\end{proof}

\subsection{Masked diffusion}\label{app:masked}

Let $\Zs=\Y\cup\{\mask\}$ with $\mask$ absorbing. The single-position marginal, kernel, and the only forward jump are
\begin{equation}
  \begin{gathered}
    \q^i_{t\given0}(y\given z_0^i)
    =
    \alpha_t\one\{y=z_0^i\}+\beta_t\one\{y=\mask\},
    \qquad
    \fR[i]{t}(e_\ell,\mask)=\rho_t,
    \\
    \q^i_{t\given s}(y\given x)
    =
    \frac{\alpha_t}{\alpha_s}\one\{y=x\}
    +
    \big(1-\frac{\alpha_t}{\alpha_s}\big)\one\{y=\mask\}.
  \end{gathered}
  \label{eq:masked-kernel}
\end{equation}
The only nonzero true reverse jumps go from the mask to the clean token,
\begin{equation}
  \dataC{a_i}(t,z_0,z_t,e_\ell)
  =
  -\frac{\alpha_t'}{\beta_t}\one\{e_\ell=z_0^i,\ z_t^i=\mask\},
  \label{eq:masked-true-reverse-rate}
\end{equation}
and the cavity head induces $\mth_i(\mask\given z_t^{-i},t)=\beta_t$ and $\mth_i(e_\ell\given z_t^{-i},t)=\alpha_t \muth_i(e_\ell\given z_t^{-i},t)$, so
\begin{equation}
  \modelC{a_{i,\mathrm{cav}}^\theta}(t,z_t,e_\ell)
  =
  -\frac{\alpha_t'}{\beta_t}\,\muth_i(e_\ell\given z_t^{-i},t)\,\one\{z_t^i=\mask\}.
  \label{eq:masked-model-reverse-rate}
\end{equation}

\begin{thmrestate}{Corollary~\ref{cor:masked}}{Denoiser equals cavity in masked diffusion}
  Wherever the Bayes factor $\q^i_{t\given0}(z_t^i\given e_\ell)$ is constant in $\ell$ on its support, the Bayes reweighting~\eqref{eq:bayes-tie} is restriction to that support followed by renormalization; at a masked position the support is the whole vocabulary and $\pistar_i=\mustar_i$.
\end{thmrestate}
\begin{proof}[Proof of Corollary~\ref{cor:masked}]
  Where $\q^i_{t\given0}(z_t^i\given e_\ell)$ is constant in $\ell$ on its support, the Bayes reweighting~\eqref{eq:bayes-tie} multiplies $\mustar_i$ by the indicator of the support and renormalizes. At a masked position $\q_{t\given0}(\mask\given e_\ell)=\beta_t$ for every $\ell$: the support is all of $\Y$, the update is the identity, and $\pistar_i=\mustar_i$.
\end{proof}
A denoiser head gives the same masked rate with $\pith_i$ in place of $\muth_i$; at an unmasked position the denoiser is a point mass rather than the cavity law, but the masked CTMC has no reverse jump out of an unmasked token, so the distinction enters neither the ELBO nor the sampler. At a masked position the jump divergence reduces to cross-entropy: using~\eqref{eq:masked-true-reverse-rate}, \eqref{eq:masked-model-reverse-rate}, and homogeneity of $\Phi$,
\begin{equation}
  \J_{i,t}^{\mathrm{mask}}(z_0,z_t)
  = -\frac{\alpha_t'}{\beta_t}\one\{z_t^i=\mask\}\left[-\log \muth_i(z_0^i\given z_t^{-i},t)\right],
  \label{eq:masked-elbo-ce}
\end{equation}
for both distribution heads: only masked positions contribute.

\paragraph{Score coordinate: SEDD and RADD as special cases.}
Denoting the clean-conditioned score $\dataC{\sigma_\ell} := \q^i_{t\given0}(e_\ell\given z_0^i)/\q^i_{t\given0}(z_t^i\given z_0^i)$ and using homogeneity, the per-jump path cost in the score coordinate is
\begin{equation}
  \Phi\big(\dataC{a_i},\modelC{a_{i,\mathrm{score}}^\theta}\big)
  =\fR[i]{t}(y,z_t^i)\,\Phi(\dataC{\sigma_\ell},\sth_i)
  =\fR[i]{t}(y,z_t^i)\big[\sth_i-\dataC{\sigma_\ell}\log \sth_i+\dataC{\sigma_\ell}(\log\dataC{\sigma_\ell}-1)\big],
  \label{eq:sedd-score-entropy}
\end{equation}
which summed and averaged is SEDD's denoising score entropy \citep{lou2024sedd}, with the forward rate as its weight. For masked diffusion the only active channel is $\mask\to e_\ell$, and the true score factors as $\sstar_i(\mask,e_\ell\given z_t^{-i},t)=(\alpha_t/\beta_t)\,\mustar_i(e_\ell\given z_t^{-i},t)$: the only learnable part is the cavity law, time-independent for fixed visible context, which is RADD's reparameterization \citep{ou2025radd}. Our implementation folds the uniform-cavity score into the head: at a masked position the head outputs logits $\modelC{f^\theta}$ and represents
\begin{equation}
  \sth_i(\mask,e_\ell\given z_t^{-i},t)=\frac{\alpha_t}{V\beta_t}\,e^{\modelC{f_\ell^\theta}},
  \label{eq:masked-score-head}
\end{equation}
so $\modelC{f^\theta}\equiv0$ is the uniform-cavity model and the optimum is $e^{\modelC{f_\ell^\theta}}=V\mustar_i(e_\ell)$. Substituting the true score and~\eqref{eq:masked-score-head} into~\eqref{eq:sedd-score-entropy}, homogeneity pulls out $\alpha_t/\beta_t$ and the logarithms cancel, leaving
\begin{equation}
  \J_{i,t}^{\mathrm{mask,score}}
  =-\frac{\alpha_t'}{\beta_t}\,\one\{z_t^i=\mask\}
  \Big[\frac1V\sum_{\ell=1}^V e^{\modelC{f_\ell^\theta}}-1-\modelC{f_{z_0^i}^\theta}+\log V\Big].
  \label{eq:masked-score-nelbo}
\end{equation}
With the induced cavity mass $\zeta:=\tfrac1V\sum_\ell e^{\modelC{f_\ell^\theta}}$,
\begin{equation}
  \frac1V\sum_\ell e^{\modelC{f_\ell^\theta}}-1-\modelC{f_{z_0^i}^\theta}+\log V
  =-\log\frac{e^{\modelC{f_{z_0^i}^\theta}}}{\sum_\ell e^{\modelC{f_\ell^\theta}}}+\big(\zeta-1-\log \zeta\big),
  \label{eq:masked-score-vs-cavity}
\end{equation}
so the same logits read as a cavity head give the cross-entropy integrand~\eqref{eq:masked-elbo-ce}, and the score loss exceeds it by the normalization slack $\zeta-1-\log \zeta\ge0$, zero iff $\zeta=1$: a masked score head is a cavity head whose normalization is learned rather than imposed.

\paragraph{Prior term.}
The natural masked prior is the deterministic mask law $\delta_{\mask}$, whose per-token KL against $\q^i_{t_2\given0}$ is $0$ if $\alpha_{t_2}=0$ and $+\infty$ otherwise. The uniform-forward prior $p_{t_2}^{\mathrm{unif\text{-}fwd}} = \beta_{t_2}\delta_{\mask}+\tfrac{\alpha_{t_2}}{V}\sum_\ell\delta_{e_\ell}$ gives the finite truncation correction
\begin{equation}
  \KLdiv{\q^i_{t_2\given0}(\cdot\given z_0^i)}{p_{t_2}^{\mathrm{unif\text{-}fwd}}}
  = \alpha_{t_2}\log V,
  \label{eq:masked-uniform-forward-prior-kl}
\end{equation}
vanishing as $t_2\uparrow1$.

\paragraph{Sampling.}
For a step $t\to s$, a visible token stays; a masked token unmasks with the plug-in step
\begin{equation}
  \p^\theta(Z_s^i=\mask\given z_t)=\frac{\beta_s}{\beta_t},
  \qquad
  \p^\theta(Z_s^i=e_\ell\given z_t)
  =
  \frac{\alpha_s-\alpha_t}{\beta_t}\,
  \muth_i(e_\ell\given z_t^{-i},t).
  \label{eq:masked-ancestral-step}
\end{equation}
Once unmasked, a token cannot be revised. At the final step $t_1$, visible tokens are kept and remaining masks decode from $\pith_i(\cdot\given z_{t_1},t_1)$.

\subsection{Uniform diffusion}\label{app:uniform}

Uniform diffusion uses $\Zs=\Y$ and mixes the clean token with the uniform law:
\begin{equation}
  \begin{gathered}
    \q^i_{t\given0}(e_r\given e_\ell)
    =
    \alpha_t\one\{e_r=e_\ell\}+\frac{\beta_t}{V},
    \qquad
    \fR[i]{t}(e_r,e_\ell)=\frac{\rho_t}{V},
    \\
    \q^i_{t\given s}(e_r\given e_\ell)
    = \frac{\alpha_t}{\alpha_s}\one\{e_r=e_\ell\}
    + \Big(1-\frac{\alpha_t}{\alpha_s}\Big)\frac1V
  \end{gathered}
  \label{eq:uniform-kernel}
\end{equation}
for $r\ne\ell$: the process proposes uniform replacements at rate $\rho_t$. The true reverse rate for $e_\ell\ne z_t^i$ is
\begin{equation}
  \dataC{a_i}(t,z_0,z_t,e_\ell)
  =
  \frac{\rho_t}{V}\,
  \frac{ \alpha_t\one\{e_\ell=z_0^i\}+\beta_t/V }{ \alpha_t\one\{z_t^i=z_0^i\}+\beta_t/V },
  \label{eq:uniform-true-reverse-rate}
\end{equation}
and the cavity head induces $\mth_i(e_\ell\given z_t^{-i},t)= \alpha_t\muth_i(e_\ell\given z_t^{-i},t)+\beta_t/V$. Proposition~\ref{prop:source-independent-cavity-elbo} applies with $\dataC{\gamma_i}=\rho_t/V$ and
\begin{equation*}
  \dataC{d_i}(e_\ell)=\q^i_{t\given0}(e_\ell\given z_0^i)=\alpha_t\one\{e_\ell=z_0^i\}+\frac{\beta_t}{V},
\end{equation*}
so that~\eqref{eq:uniform-true-reverse-rate} is $\tfrac{\rho_t}{V}\,\dataC{d_i}(e_\ell)/\dataC{d_i}(z_t^i)$; unlike masked diffusion, every position can contribute. The three integrands are
\begin{align*}
  \J^{\mathrm{unif,cav}}_{i,t}
   & = \frac{\rho_t/V}{\dataC{d_i}(z_t^i)}
  \left[\KLdiv{\dataC{d_i}}{\mth_i}+\ISdiv{\dataC{d_i}(z_t^i)}{\mth_i(z_t^i)}\right],
  \\[2pt]
  \J_{i,t}^{\mathrm{unif,den}}
   & = \frac{\rho_t}{V} \sum_{e_\ell\neq z_t^i}
  \Phi\!\left(\frac{\dataC{d_i}(e_\ell)}{\dataC{d_i}(z_t^i)},\;
  \sum_r \pith_i(e_r\given z_t,t)\,
  \frac{\q^i_{t\given0}(e_\ell\given e_r)}{\q^i_{t\given0}(z_t^i\given e_r)}\right),
  \\[2pt]
  \J_{i,t}^{\mathrm{unif,score}}
   & = \frac{\rho_t}{V}\sum_{e_\ell\neq z_t^i}
  \Phi\!\left(\frac{\dataC{d_i}(e_\ell)}{\dataC{d_i}(z_t^i)},\;\sth_i(z_t^i,e_\ell)\right),
\end{align*}
the last being~\eqref{eq:sedd-score-entropy} with weight $\rho_t/V$. The second argument of $\J_{i,t}^{\mathrm{unif,den}}$, the denoiser's average of kernel ratios, is abbreviated $\modelC{g^\theta_i}(y):=\sum_\ell\pith_i(e_\ell\given z_t,t)\,\q^i_{t\given0}(y\given e_\ell)/\q^i_{t\given0}(z_t^i\given e_\ell)$ below. The implementation represents the score by the head logits directly, $\sth_i(z_t^i,e_\ell)=e^{\modelC{f_\ell^\theta}}$, the uniform-cavity fold with factor one: the uniform cavity induces $m^u\equiv1/V$ and score $s^u\equiv1$, so at $\modelC{f^\theta}\equiv0$ the model rate equals the forward rate, the reverse process at stationarity, as in SEDD \citep{lou2024sedd}.

\paragraph{Prior term.}
The natural prior, the stationary uniform law, coincides with the uniform-forward law, with per-token KL
\begin{equation}
  \KLdiv{\q^i_{t_2\given0}(\cdot\given z_0^i)}{\mathrm{unif}(\Y)}
  = \Big(\alpha_{t_2}+\frac{\beta_{t_2}}{V}\Big)
  \log\big(V\alpha_{t_2}+\beta_{t_2}\big)
  + \frac{(V-1)\beta_{t_2}}{V}\log\beta_{t_2}
  \;\xrightarrow[t_2\uparrow1]{}\;0.
  \label{eq:uniform-prior-kl}
\end{equation}

\paragraph{Sampling.}
Every position may be resampled. The plug-in step is
\begin{equation}
  \p^\theta(Z_s^i=e_\ell\given z_t)
  =
  \q^i_{t\given s}(z_t^i\given e_\ell)\,
  \frac{\alpha_s\,\muth_i(e_\ell\given z_t^{-i},t)+\beta_s/V}{\alpha_t\,\muth_i(z_t^i\given z_t^{-i},t)+\beta_t/V},
  \label{eq:uniform-ancestral-step}
\end{equation}
a denoiser head entering after the inverse Bayes reweighting~\eqref{eq:bayes-tie}; the final step decodes every position from $\pith_i(\cdot\given z_{t_1},t_1)$.

\begin{proposition}[Uninformative denoiser ELBO diverges for uniform diffusion]\label{prop:uniform-denoiser-diverges}
  For uniform diffusion with $\alpha_0=1$ and $V\ge2$, the head $\pith_i\equiv1/V$ has, as $t\downarrow0$,
  \begin{equation*}
    \Exp{\q_{t\given0}}{\J_{i,t}^{\mathrm{unif,den}}}
    = \frac{V-1}{V}\,\frac{-\alpha_t'}{\beta_t}\,(1+o(1)),
  \end{equation*}
  non-integrable at $t=0$, so that $\NELBO_{[t_1,t_2]}(\theta) = \frac{(V-1)L}{V}\,\log\frac1{\beta_{t_1}}+O(1)$ as $t_1\downarrow0$.
\end{proposition}
\begin{proof}
  Fix one position with clean token $z_0^i$. The true rates~\eqref{eq:uniform-true-reverse-rate} satisfy, as $\beta_t\to0$: if $z_t^i=z_0^i$, every jump away decays like $O(\beta_t)$; if $z_t^i\neq z_0^i$, the correcting jump $e_\ell=z_0^i$ has rate $(-\alpha_t'/\beta_t)(1+o(1))$ (using $\rho_t\alpha_t=-\alpha_t'$) and the others are $O(\rho_t)$. The uninformative model rate is, up to the factor $\rho_t/V$, the average $\frac1V\sum_r g$ with $e_r$ in place of $z_0^i$; since $e_\ell\ne z_t^i$, exactly one term ($e_r=e_\ell$) diverges like $\alpha_t/\beta_t$ while the rest stay bounded, so
  \begin{equation*}
    \modelC{a_{i,\mathrm{den}}^\theta}(t,z_t,e_\ell)
    =\frac{-\alpha_t'}{V\beta_t}+O(-\alpha_t')
    \qquad\text{for every } e_\ell\ne z_t^i.
  \end{equation*}
  On the event $z_t^i=z_0^i$ (probability $1-O(\beta_t)$), the ratio $\dataC{a_i}/\modelC{a_i^\theta}=O(\beta_t^2)$, so $\Phi(\dataC{a_i},\modelC{a_i^\theta})=\modelC{a_i^\theta}\,(1+O(\beta_t^2|\log\beta_t|))$ and the $V-1$ destinations contribute
  \begin{equation*}
    \Exp{\q_{t\given0}}{\J_{i,t}^{\mathrm{unif,den}}\one\{z_t^i=z_0^i\}}
    =(V-1)\frac{-\alpha_t'}{V\beta_t}+O(-\alpha_t').
  \end{equation*}
  On the complementary event, of probability $(V-1)\beta_t/V$, both rates are $O(-\alpha_t'/\beta_t)$ with the model rate of the same order, so $\Phi(\dataC{a_i},\modelC{a_i^\theta})=O(-\alpha_t'/\beta_t)$ and the event probability cancels the singularity, contributing $O(-\alpha_t')$. Adding the two and integrating with $-\alpha_t'\,dt=d\beta_t$ gives the per-position path term $\frac{V-1}{V}\log(1/\beta_{t_1})+O(1)$; the $L$ positions add.
\end{proof}
The failure mode is the jump support: the uninformative denoiser allows jumps away from clean tokens near $t=0$, and those kernel ratios diverge like $1/\beta_t$, whereas the cavity's induced marginal carries a matching $O(\beta_t)$ Bayes factor that downweights them and keeps its ELBO finite.

\subsection{Generalized interpolating discrete diffusion}\label{app:gidd}

GIDD \citep{vonruette2025gidd} chooses a differentiable mixing law $\omega_t$ on $\Zs$ (a row vector) and sets
\begin{equation}
  \q^i_{t\given0}(y\given e_\ell)
  = \alpha_t\one\{y=e_\ell\}+\beta_t\omega_t(y),
  \qquad \q^i_{t\given0}
  = \alpha_t I+\beta_t\one\omega_t.
  \label{eq:gidd-marginal-kernel}
\end{equation}
Solving $\q^i_{t\given0}=\q^i_{s\given0}\,\q^i_{t\given s}$ within the same rank-one family gives
\begin{equation}
  \q^i_{t\given s}
  =
  \alpha_{t\given s}I+\beta_{t\given s}\one\omega_{t\given s},
  \qquad
  \alpha_{t\given s}:=\frac{\alpha_t}{\alpha_s},
  \qquad
  \beta_{t\given s}\omega_{t\given s}
  :=\beta_t\omega_t-\frac{\alpha_t}{\alpha_s}\beta_s\omega_s,
  \label{eq:gidd-conditional-kernel}
\end{equation}
and expanding at $t=s+\Delta$ gives the destination-rate vector and off-diagonal generator
\begin{equation}
  h_t(y):=\beta_t\omega_t'(y)+\rho_t\omega_t(y),
  \qquad
  \fR[i]{t}(x,y)=h_t(y),\qquad y\neq x,
  \label{eq:gidd-forward-rate}
\end{equation}
with $\sum_y h_t(y)=\rho_t$; valid schedules require $h_t\geq0$ off the diagonal. Consistent with the interior-window assumptions, we take the supports of $\omega_t$ and $h_t$ fixed on $(0,1)$; strictly positive entries are then bounded away from zero on compact interior windows, which is how the clipped window fixes the supports of the family. Our experiments use the static mixture $\omega_t\equiv\omega=(1-\lambda)\delta_{\mask}+\lambda\,\mathrm{unif}(\Y)$ (\S\ref{sec:experiments}), for which $h_t=\rho_t\omega$; masked and uniform diffusion are $\omega=\delta_\mask$ and $\omega=\mathrm{unif}(\Y)$.

The true and cavity reverse rates take the form
\begin{equation}
  \dataC{a_i}(t,z_0,z_t,y)
  =
  h_t(z_t^i)\,
  \frac{\q^i_{t\given0}(y\given z_0^i)}
  {\q^i_{t\given0}(z_t^i\given z_0^i)},
  \qquad
  \modelC{a_{i,\mathrm{cav}}^\theta}(t,z_t,y)
  =
  h_t(z_t^i)\,
  \frac{\mth_i(y\given z_t^{-i},t)}
  {\mth_i(z_t^i\given z_t^{-i},t)},
  \label{eq:gidd-rates}
\end{equation}
with $\mth_i(y\given z_t^{-i},t)=\alpha_t\muth_i(y\given z_t^{-i},t)\one\{y\in\Y\}+\beta_t\omega_t(y)$, and Proposition~\ref{prop:source-independent-cavity-elbo} applies with destination-dependent $\dataC{\gamma_i}(t,y)=h_t(y)$ and $\dataC{d_i}(y)=\alpha_t\one\{y=z_0^i\}+\beta_t\omega_t(y)$:
\begin{equation}
  \J_{i,t}^{\mathrm{GIDD,cav}}
  = \frac{h_t(z_t^i)}{\dataC{d_i}(z_t^i)}
  \left[ \KLdiv{\dataC{d_i}}{\mth_i} + \ISdiv{\dataC{d_i}(z_t^i)}{\mth_i(z_t^i)}\right].
  \label{eq:gidd-closed-form-nelbo}
\end{equation}
A denoiser head gives, for $y\neq z_t^i$, the rate
\begin{equation}
  \modelC{a_{i,\mathrm{den}}^\theta}(t,z_t,y)
  =
  h_t(z_t^i)
  \sum_\ell
  \pith_i(e_\ell\given z_t,t)\,
  \frac{\alpha_t\one\{y=e_\ell\}+\beta_t\omega_t(y)}{\alpha_t\one\{z_t^i=e_\ell\}+\beta_t\omega_t(z_t^i)}.
  \label{eq:gidd-denoiser-reverse-rate}
\end{equation}
Since $\Phi$ is $1$-homogeneous and the true rate is $h_t(z_t^i)\,\dataC{d_i}(y)/\dataC{d_i}(z_t^i)$, the denoiser and score integrands are
\begin{equation}
  \J_{i,t}^{\mathrm{GIDD,den}}
  = h_t(z_t^i)\sum_{y\neq z_t^i}
  \Phi\!\left(\frac{\dataC{d_i}(y)}{\dataC{d_i}(z_t^i)},\;
  \sum_\ell \pith_i(e_\ell\given z_t,t)\,
  \frac{\q^i_{t\given0}(y\given e_\ell)}{\q^i_{t\given0}(z_t^i\given e_\ell)}\right),
  \label{eq:gidd-denoiser-nelbo}
\end{equation}
\begin{equation}
  \J_{i,t}^{\mathrm{GIDD,score}}
  = h_t(z_t^i)\sum_{y\neq z_t^i}
  \Phi\!\left(\frac{\dataC{d_i}(y)}{\dataC{d_i}(z_t^i)},\;\sth_i(z_t^i,y\given z_t^{-i},t)\right),
  \label{eq:gidd-score-nelbo}
\end{equation}
with App.~\ref{app:uniform} the case $\omega_t\equiv1/V$, $h_t\equiv\rho_t/V$. The score head folds the uniform-cavity score: with $m^u(y):=\tfrac{\alpha_t}{V}\one\{y\in\Y\}+\beta_t\omega_t(y)$ the noisy marginal induced by a uniform cavity,
\begin{equation}
  \sth_i(z_t^i,y\given z_t^{-i},t)=\frac{m^u(y)}{m^u(z_t^i)}\,e^{\modelC{f_y^\theta}},
  \label{eq:gidd-score-head}
\end{equation}
so $\modelC{f^\theta}\equiv0$ represents the uniform-cavity model on every channel; the factor is constant on each reverse channel (content destinations share $m^u(y)=(\alpha_t+\beta_t\lambda)/V$ under the static mixture, and the mask destination carries $m^u(\mask)=\beta_t\omega(\mask)$). Uniform diffusion is recovered at $\omega(\mask)=0$, where the factor is one, and masked diffusion at $\omega=\delta_\mask$, where the only channel is $\mask\to e_\ell$ and the factor is $\alpha_t/(V\beta_t)$, as in~\eqref{eq:masked-score-head}.

\paragraph{Prior term.}
The natural prior is $\omega_{t_2}$; for $\omega_{t_2}(z_0^i)>0$ the per-token KL is
\begin{equation}
  \KLdiv{\q^i_{t_2\given0}(\cdot\given z_0^i)}{\omega_{t_2}}
  =
  \left(\alpha_{t_2}+\beta_{t_2}\omega_{t_2}(z_0^i)\right)
  \log
  \frac{\alpha_{t_2}+\beta_{t_2}\omega_{t_2}(z_0^i)}{\omega_{t_2}(z_0^i)}
  +
  \beta_{t_2}(1-\omega_{t_2}(z_0^i))\log\beta_{t_2},
  \label{eq:gidd-prior-kl}
\end{equation}
infinite by support mismatch if $\omega_{t_2}(z_0^i)=0<\alpha_{t_2}$. The uniform-forward prior $p_{t_2}^{\mathrm{unif\text{-}fwd}}(y)=\tfrac{\alpha_{t_2}}{V}\one\{y\in\Y\}+\beta_{t_2}\omega_{t_2}(y)$ is finite whenever $\alpha_{t_2}>0$:
\begin{equation}
  \KLdiv{\q^i_{t_2\given0}(\cdot\given z_0^i)}{p_{t_2}^{\mathrm{unif\text{-}fwd}}}
  =
  \sum_{y\in\Y}
  \left(\alpha_{t_2}\one\{y=z_0^i\}+\beta_{t_2}\omega_{t_2}(y)\right)
  \log\frac{\alpha_{t_2}\one\{y=z_0^i\}+\beta_{t_2}\omega_{t_2}(y)}{\alpha_{t_2}/V+\beta_{t_2}\omega_{t_2}(y)},
  \label{eq:gidd-uniform-forward-prior-kl}
\end{equation}
recovering $\alpha_{t_2}\log V$ for masked and the stationary case for uniform.

\begin{corollary}[Conversions for GIDD]\label{cor:gidd-conversions}
  For the kernel~\eqref{eq:gidd-marginal-kernel} with $\alpha_t>0$, the channel $\q^i_{t\given0}$ has full row rank and the score-to-cavity map is affine:
  \begin{equation}
    \muth_i(e_\ell)=\frac{\mth_i(e_\ell)-\beta_t\,\omega_t(e_\ell)}{\alpha_t},
    \qquad
    \mth_i(y)=\frac{\sth_i(z_t^i,y)}{\sum_{y'}\sth_i(z_t^i,y')},
    \label{eq:score-to-cavity-gidd}
  \end{equation}
  a probability vector iff $\mth_i(e_\ell)\ge\beta_t\omega_t(e_\ell)$ for every $\ell$ and $\mth_i(\mask)=\beta_t\omega_t(\mask)$.
\end{corollary}
\begin{proof}
  The noise floor $\beta_t\omega_t(y)$ is source-independent, so $\mth_i(e_\ell)=\alpha_t\muth_i(e_\ell)+\beta_t\omega_t(e_\ell)$ decouples across $\ell$; solve for $\muth_i$, with $\mth_i$ recovered from a score as in Proposition~\ref{prop:conversions}. The mask carries no cavity mass, so $\muth_i$ sums to one iff the mask entry sits at its noise floor.
\end{proof}

\begin{table}[ht]
  \caption{Conversions for GIDD (Corollary~\ref{cor:gidd-conversions}), with current state $z_t^i$, Bayes factor $\q^i_{t\given0}(z_t^i\given e_\ell)=\alpha_t\one\{e_\ell=z_t^i\}+\beta_t\omega_t(z_t^i)$, and cavity-induced marginal $\mth_i(y)=\alpha_t\muth_i(y)\one\{y\in\Y\}+\beta_t\omega_t(y)$ (taken $\propto \sth_i(z_t^i,y)$ when starting from a score). Masked ($\omega_t=\delta_{\mask}$) gives $\muth_i(e_\ell)=(\beta_t/\alpha_t)\,\sth_i(\mask,e_\ell)$, which is RADD; uniform gives $\muth_i(e_\ell)=(\mth_i(e_\ell)-\beta_t/V)/\alpha_t$.}
  \label{tab:gidd-conversions}
  \centering\small
  \renewcommand{\arraystretch}{1.5}
  \setlength{\tabcolsep}{7pt}
  \begin{tabular}{@{}l l l l@{}}
    \toprule
    From\,$\backslash$\,to
             & denoiser $\pith_i(e_\ell)$
             & cavity $\muth_i(e_\ell)$
             & score $\sth_i(z_t^i,y)$
    \\
    \midrule
    denoiser & ---
             & $\propto\pith_i(e_\ell)/\q^i_{t\given0}(z_t^i\given e_\ell)$
             & $\mth_i(y)/\mth_i(z_t^i)$
    \\
    cavity   & $\propto\muth_i(e_\ell)\,\q^i_{t\given0}(z_t^i\given e_\ell)$
             & ---
             & $\mth_i(y)/\mth_i(z_t^i)$
    \\
    score
             & $\propto\muth_i(e_\ell)\,\q^i_{t\given0}(z_t^i\given e_\ell)$ (via $s\!\to\!\mu$)
             & $\dfrac{\mth_i(e_\ell)-\beta_t\omega_t(e_\ell)}{\alpha_t}$
             & ---
    \\
    \bottomrule
  \end{tabular}
\end{table}

\begin{remark}[The divergence for GIDD]\label{rmk:divergence-gidd}
  The divergence is not special to uniform diffusion. For a kernel $\q^i_{t\given0}(y\given e_\ell)=\alpha_t\one\{y=e_\ell\}+\beta_t\omega_t(y)$, fix a clean token $x$ with $\omega_t(x)>0$ near $t=0$ and consider the revealed event $z_t^i=z_0^i=x$, of probability tending to one. For any clean destination $y\neq x$, the $e_\ell=y$ term of the uninformative denoiser gives $\modelC{g^\theta_i}(y)\geq\tfrac1V\,\tfrac{\alpha_t+\beta_t\omega_t(y)}{\beta_t\omega_t(x)}$ while $\dataC{d_i}(y)/\dataC{d_i}(z_t^i)=O(\beta_t)$, so $\Phi\big(\dataC{d_i}(y)/\dataC{d_i}(z_t^i),\modelC{g^\theta_i}(y)\big)=\modelC{g^\theta_i}(y)\,(1+o(1))$, and this one destination contributes at order $h_t(x)/(\beta_t\omega_t(x))$, where $h_t$ is the forward destination rate of~\eqref{eq:gidd-forward-rate}. This has the exact primitive
  \begin{equation*}
    \frac{h_t(x)}{\beta_t\omega_t(x)}
    =\frac{d}{dt}\log\frac{\beta_t\omega_t(x)}{\alpha_t},
  \end{equation*}
  whose integral diverges as the lower endpoint tends to zero, since $\beta_t\omega_t(x)/\alpha_t\to0$. Any clean token carrying positive $\omega_t$-mass and positive data probability produces the divergence, including uniform and every full-support hybrid.
\end{remark}

\begin{table}[htbp]
  \caption{The three processes as instances of the common form, with $\dataC{d_i}(\cdot)=\q^i_{t\given0}(\cdot\given z_0^i)$, $\mth_i=\mth_i(\cdot\given z_t^{-i},t)$ (Proposition~\ref{prop:source-independent-cavity-elbo}), $\modelC{g^\theta_i}$ the denoiser's average of kernel ratios (App.~\ref{app:uniform}), and $h_t(y)=\beta_t\omega_t'(y)+\rho_t\omega_t(y)$ the GIDD destination rate of~\eqref{eq:gidd-forward-rate}; heads are written without their conditioning arguments. At masked positions the denoiser and cavity rows coincide (Corollary~\ref{cor:masked}). The score rows use the folded parameterizations~\eqref{eq:masked-score-head} and~\eqref{eq:gidd-score-head}. Both terminal priors are recorded; the uniform-forward one is the support-matched choice of Proposition~\ref{prop:calibration}.}
  \label{tab:process-summary}
  \centering
  \setlength{\tabcolsep}{3pt}
  \renewcommand{\arraystretch}{2.5}
  \begin{tabular}{@{}>{\raggedright\arraybackslash}m{2.0cm}>{\centering\arraybackslash}m{4.5cm}>{\centering\arraybackslash}m{3.3cm}>{\centering\arraybackslash}m{3.3cm}@{}}
    \toprule
     & Masked
     & Uniform
     & GIDD
    \\
    \midrule
    Forward rate
     & $\fR[i]{t}(e_y,\mask)=\rho_t$
     & $\fR[i]{t}(e_y,e_x)=\dfrac{\rho_t}{V}$
     & $\fR[i]{t}(x,y)=h_t(y)$
    \\
    Reverse rate $(\RevQth{t})_i(z_t,y)$
     & $-\dfrac{\alpha_t'}{\beta_t}\one\{z_t^i=\mask\}\muth_i(y^i)$
     & \multicolumn{2}{>{\centering\arraybackslash}m{\dimexpr 3.3cm+3.3cm+2\tabcolsep+\arrayrulewidth\relax}@{}}{$h_t(z_t^i)\,\dfrac{\alpha_t \muth_i(y^i)+\beta_t\omega_t(y^i)}{\alpha_t \muth_i(z_t^i)+\beta_t\omega_t(z_t^i)}$ \newline (uniform: $h_t\equiv\rho_t/V$, $\omega_t\equiv1/V$)}
    \\
    Sampling $\p^\theta_{s\given t}(z_s^i\given z_t)$
     & $
         \begin{cases}
           \dfrac{\alpha_s-\alpha_t}{\beta_t}\,\muth_i(z_s^i), & z_s^i\in\Y  \\[6pt]
           \dfrac{\beta_s}{\beta_t},                           & z_s^i=\mask
         \end{cases}
       $
     & \multicolumn{2}{>{\centering\arraybackslash}m{\dimexpr 3.3cm+3.3cm+2\tabcolsep+\arrayrulewidth\relax}@{}}{$\q^i_{t\given s}(z_t^i\given z_s^i)\,\dfrac{\alpha_s\,\muth_i(z_s^i)+\beta_s\omega_s(z_s^i)}{\alpha_t\,\muth_i(z_t^i)+\beta_t\omega_t(z_t^i)}$}
    \\
    $\J_{i,t}^{\mathrm{den}}$
     & $-\dfrac{\alpha_t'}{\beta_t}\,\one\{z_t^i=\mask\}\big[{-\log}\pith_i(z_0^i)\big]$
     & \multicolumn{2}{>{\centering\arraybackslash}m{\dimexpr 3.3cm+3.3cm+2\tabcolsep+\arrayrulewidth\relax}@{}}{$h_t(z_t^i)\sum_{y\neq z_t^i}\Phi\big(\dataC{d_i}(y)/\dataC{d_i}(z_t^i),\,\modelC{g^\theta_i}(y)\big)$}
    \\
    $\J_{i,t}^{\mathrm{cav}}$
     & $-\dfrac{\alpha_t'}{\beta_t}\,\one\{z_t^i=\mask\}\big[{-\log}\muth_i(z_0^i)\big]$
     & \multicolumn{2}{>{\centering\arraybackslash}m{\dimexpr 3.3cm+3.3cm+2\tabcolsep+\arrayrulewidth\relax}@{}}{$\dfrac{h_t(z_t^i)}{\dataC{d_i}(z_t^i)}\Big[\KLdiv{\dataC{d_i}}{\mth_i} + \ISdiv{\dataC{d_i}(z_t^i)}{\mth_i(z_t^i)}\Big]$}
    \\
    $\J_{i,t}^{\mathrm{score}}$
     & $-\dfrac{\alpha_t'}{\beta_t}\one\{z_t^i=\mask\}$\newline
    ${}\cdot\big[\tfrac1V\sum_\ell e^{\modelC{f_\ell^\theta}}\!-\!1\!-\!\modelC{f_{z_0^i}^\theta}\!+\!\log V\big]$
     & \multicolumn{2}{>{\centering\arraybackslash}m{\dimexpr 3.3cm+3.3cm+2\tabcolsep+\arrayrulewidth\relax}@{}}{$h_t(z_t^i)\sum_{y\neq z_t^i}\Phi\big(\dataC{d_i}(y)/\dataC{d_i}(z_t^i),\,\sth_i(z_t^i,y)\big)$}
    \\
    $1/V$ denoiser NELBO
     & $L\log V$
     & $\frac{(V-1)L}{V}\log\frac1{\beta_{t_1}}$\newline${}+O(1)$
     & diverges (Rmk.~\ref{rmk:divergence-gidd})
    \\
    $1/V$ cavity NELBO
     & \multicolumn{3}{c@{}}{$L\log V$}
    \\
    Natural prior per-token KL
     & $0$ if $\alpha_{t_2}=0$, else ${+}\infty$
     & \eqref{eq:uniform-prior-kl}
     & \eqref{eq:gidd-prior-kl}
    \\
    Unif.-fwd.\ per-token KL
     & $\alpha_{t_2}\log V$
     & \eqref{eq:uniform-prior-kl}
     & \eqref{eq:gidd-uniform-forward-prior-kl}
    \\
    \bottomrule
  \end{tabular}
\end{table}

\section{Exact verification on a solvable model}\label{app:toy}

Every identity of the paper has been verified on an exactly enumerable model, whose oracle laws, entropies, and path KLs are available in closed form. The data law is a homogeneous bigram chain on $V=8$ tokens of length $L=8$ (transition matrix drawn from a row-softmax of $\mathcal N(0,1)$ entries at temperature $0.3$), with $H(\cleanq)=10.00$ nats, that is $1.25$ per token. We instantiate masked, uniform, and GIDD noising with the linear schedule $\alpha_t=1-t$ on the window $[t_1,t_2]=[0.02,0.98]$, the GIDD member mixing through $\omega=(1-\lambda)\delta_{\mask}+\lambda\,\mathrm{unif}(\Y)$ with $\lambda=\tfrac12$ except where swept. Every floor, calibration, and conversion reported is oracle-exact and assembled by the implementation of \S\ref{sec:experiments} evaluated at the exact head; only Figure~\ref{fig:toy-infocore} involves training (a small DiT: hidden 128, 4 heads, 3 blocks; AdamW \citep{loshchilov2019adamw} at $10^{-3}$ under a cosine decay, batch 256, 32{,}000 steps, each head trained on the ELBO of its own coordinate). On a finite window the oracle's own NELBO sits $0.001$--$0.003$ nats/token above $H(\cleanq)/L$ and equals the residual boundary terms (namely, the factorization error from the final decoding step at $t_1$ plus the excess of the fixed prior $I(Z_0;Z_{t_2})$). The three figures are placed with the results they verify: Figure~\ref{fig:toy-floor} verifies the shared floor and its rate (App.~\ref{app:infoloss}), Figure~\ref{fig:toy-infocore} the Oracle Distance decay and the Pythagorean split (App.~\ref{app:oracle-proofs}), and Figure~\ref{fig:toy-calibration} the calibration identities (App.~\ref{app:three-laws}).

\section{Experimental details}\label{app:details}

\subsection{Implementation details}\label{app:impl}

All 11 runs share the configuration of Table~\ref{tab:hparams}, differing only in noise kernel and coordinate. The architecture is the diffusion transformer of \citet{sahoo2024mdlm}, the same one used by \citet{lou2024sedd}, with time conditioning by adaLN in every run; of the $169.6$M total parameters, $92.2$M are non-embedding, $7.4$M of those the time-conditioning path, and the remainder the token embedding and the output head over the GPT-2 vocabulary ($V=50{,}257$, plus the mask state). Muon \citep{jordan2024muon} is applied to the hidden matrices under the RMS-matched update scaling of \citet{liu2025muonscalable} (token embedding, output head, and sub-matrix parameters routed to AdamW; \citealp{loshchilov2019adamw}).

\begin{table}[h]
  \caption{Shared configuration of the 11 OpenWebText training runs and of every reported evaluation and sampling procedure. Kernel and coordinate are the only differences between runs.}
  \label{tab:hparams}
  \centering\small
  \setlength{\tabcolsep}{4pt}
  \begin{tabular}{@{}ll@{\hspace{1.2em}}ll@{}}
    \toprule
    \multicolumn{2}{@{}l}{Architecture}          & \multicolumn{2}{l@{}}{Optimization}
    \\
    \cmidrule(r){1-2}\cmidrule(l){3-4}
    Layers                                       & 12                                  & Objective                & ELBO, own coordinate
    \\
    Hidden width                                 & 768                                 & Optimizer                & Muon, RMS-matched
    \\
    Attention heads                              & 12                                  & Learning rate            & $10^{-3}$, constant
    \\
    MLP width                                    & 3072                                & Warmup                   & linear from 0, 1k steps
    \\
    Position embedding                           & rotary                              & Steps                    & 500k ($65$B tokens)
    \\
    Time conditioning                            & adaLN, width 128                    & Batch size               & 512
    \\
    Non-embedding params                         & $92.2$M                             & Muon momentum            & $0.95$, Nesterov
    \\
    Total params                                 & $169.6$M                            & Weight decay             & $0.01$, excl.\ norms/emb.
    \\
    Sequence length                              & 256                                 & Gradient clip            & 8
    \\
    Vocabulary                                   & $50{,}257$ $+$ mask                 & EMA decay                & $0.9999$
    \\
    Dropout                                      & 0                                   & Time sampling            & low-discrepancy uniform
    \\
                                                 &                                     & Window $[t_1,t_2]$       & $[0.001,0.9999]$
    \\
                                                 &                                     & Adam $(\beta_1,\beta_2)$ & $(0.9,0.999)$
    \\
    \midrule
    \multicolumn{2}{@{}l}{Likelihood evaluation} & \multicolumn{2}{l@{}}{Sampling}
    \\
    \cmidrule(r){1-2}\cmidrule(l){3-4}
    Weights                                      & EMA                                 & Sampler                  & ancestral, denoiser coord.
    \\
    Split                                        & test                                & Time grid                & cosine
    \\
    Batches                                      & 100 ($51{,}200$ sequences)          & NFE                      & 256, swept $\{16,\dots,1024\}$ \\
    Time bins                                    & 50, equal mass                      & Temperature, top-$p$     & 1, 1
    \\
    Units                                        & bits/token, full bound              & Samples per point        & $2{,}048$                      \\
    \bottomrule
  \end{tabular}
\end{table}
Every run trains the path term of the ELBO with the estimator of~\eqref{eq:empirical-elbo}, under the linear schedule $\alpha_t=1-t$ and uniform time sampling, and the boundary terms enter only the reported NELBOs. Score heads use the folded parameterizations~\eqref{eq:masked-score-head} and~\eqref{eq:gidd-score-head}.

All heads are stored as logits and converted in log space following Table~\ref{tab:gidd-conversions}.
A distribution head maps to the score through its noisy marginal; a score head maps back by normalizing $\sth_i(z_t^i,\cdot)$ to $\mth_i$ and applying the affine inverse~\eqref{eq:score-to-cavity-gidd}, where entries violating $\mth_i\ge\beta_t\omega_t$ lie off the image of the forward channel and are set to zero, and the rest renormalized.

All coordinates share the reconstruction term,
\begin{equation}
  -\Exp{\q_{t_1\given0}(z_{t_1}\given z_0)}{\sum_i\log \pith_i(z_0^i\given z_{t_1},t_1)},
  \label{eq:implemented-reconstruction}
\end{equation}
with the decoder obtained from the trained head by these conversions (for a score head, score $\to$ cavity $\to$ denoiser); the sampler's final decode uses the same decoder, so the reconstruction term scores the decoder actually sampled from. Because the affine inverse can place mass exactly zero on the true token, the implementation smooths the decoder, $\pith_i\mapsto(1-\epsilon)\pith_i+\epsilon/V$ with $\epsilon=10^{-8}$, bounding the reconstruction term by $\log(V/\epsilon)$ per token, and the ELBO is exact for the smoothed decoder.

\paragraph{Samplers.}\label{app:samplers}
Two samplers appear in the paper, both applied in the denoiser coordinate, so a head trained in another coordinate is converted first through Proposition~\ref{prop:conversions}. The \emph{ancestral} sampler is the factorized backward jump~\eqref{eq:ancestral}: on a grid $t_2=t_0>\dots>t_N=t_1$ it draws every position of $z_{t_{k+1}}$ from its exact conditional given $z_{t_k}$, so each step multiplies the head's law by the analytic factor $\q^i_{t\given s}(z_t^i\given y)$ and keeps the state on the noising path. The \emph{re-corrupt} sampler instead decodes a clean sequence $\hat z_0$ from the head's law at $z_{t_k}$ and forward-corrupts it to the next level, $z_{t_{k+1}}\sim\q_{t_{k+1}\given0}(\cdot\given\hat z_0)$. It never consults $\q^i_{t\given s}$, so the whole step is carried by the decoded law; in exchange every state it produces is a forward corruption of a complete, self-consistent sequence, which is what lets a solver revise a committed answer instead of only refining an unfinished one. Reading a head in the wrong coordinate, the misread of \S\ref{subsec:misread}, changes only the head-to-law map and leaves both samplers otherwise untouched.

\paragraph{Data.}
OpenWebText documents are tokenized with the GPT-2 tokenizer ($V=50{,}257$ content tokens plus a mask state) and concatenated, then cut into records of exactly 256 tokens, so no record carries padding and the token stream and its order are those of the corpus. OpenWebText has no released validation split, so we hold out one of the 73 shards the corpus is cut into, about $1.4\%$ of the documents, and train on the rest. Splitting the held-out shard's records into even and odd halves gives the validation and test sets. All paper numbers are computed on the test half, on a fixed set of 100 batches (51{,}200 sequences, $\mathrm{SE}\approx0.004$ bits/token). Only the training curves of Figure~\ref{fig:training-temperature} are validation (32 batches per point), labeled as such.

\subsection{Logic tasks}\label{app:puzzles}

Each instance is a token sequence whose problem part is a clean anchor, never corrupted and excluded from the loss, and whose solution part is diffused and checked by a verifier after decoding.

\paragraph{Sudoku.} The Sudoku-Extreme boards of \citet{wang2025hrm}, as $81$ digit tokens with the given cells fixed: 3.8M training boards and a test pool of 423k, split into orbit-disjoint validation and test sets of about 200k and 223k. An embedding added to the input marks the given cells, without which a uniformly corrupted cell is indistinguishable from a given clue and uniform diffusion does not train. A board is solved when it satisfies every Sudoku constraint.

\paragraph{Boxoban.} $10\times10$ Sokoban levels \citep{boxoban2018} with the A$^*$ solutions of \citet{garrigaalonso2024sokoban}, encoded as the $100$ board cells followed by $128$ action tokens over $\{$up, right, down, left, stay$\}$, where \emph{stay} pads the solution and is scored like a move. The 92k solved levels of the official validation pool give validation and test sets of 25k each, the rest joining the $936$k training levels. A level is solved when the simulated moves place every box on a target.

\paragraph{Maze.} $40\times40$ mazes with their shortest path, encoded as $1{,}600$ cells followed by $192$ moves in the same alphabet: $500$k generated training mazes and $25$k held-out ones. The model encodes the board with a convolutional grid encoder, which the move tokens read by cross-attention. A maze is solved when the path reaches the goal.

\paragraph{Graph coloring.} Planted $4$-colorings of $n=64$ vertices with cross-color edges present independently with probability $p=0.2$, an average degree of $9.6$ just above the Kesten--Stigum detectability threshold of $9$. The sequence is one color token per vertex, the graph enters as the attention mask of half the heads, and instances are drawn fresh from seeded generators. Any proper coloring counts as solved.

\paragraph{Training and evaluation.} Table~\ref{tab:logic-hparams} lists the configuration of every run and of its evaluation. Training augments Sudoku with digit relabeling, row and column permutations within and across bands and stacks, and transposition, and Boxoban and maze with the eight rotations and reflections of the board, the moves remapped to match. Solve rates (Figure~\ref{fig:puzzles}) and NELBOs (Table~\ref{tab:puzzles-nelbo}) are averaged over the three training seeds. Reading the head as a denoiser skips the Bayes reweighting of Proposition~\ref{prop:conversions} and leaves everything else unchanged.

\begin{table}[h]
  \caption{Configuration of the logic-task runs and of their evaluation: settings shared by the four tasks (top) and per task (bottom).}
  \label{tab:logic-hparams}
  \centering\small
  \setlength{\tabcolsep}{4pt}
  \begin{tabular}{@{}ll@{\hspace{1.2em}}ll@{}}
    \toprule
    \multicolumn{2}{@{}l}{Architecture}          & \multicolumn{2}{l@{}}{Optimization}
    \\
    \cmidrule(r){1-2}\cmidrule(l){3-4}
    Layers                                       & 12                                  & Objective                & ELBO, cavity coordinate
    \\
    Hidden width                                 & 512                                 & Kernel                   & uniform
    \\
    Attention heads                              & 8                                   & Optimizer                & AdamW
    \\
    MLP width                                    & 2048                                & Adam $(\beta_1,\beta_2)$ & $(0.9,0.999)$
    \\
    Time conditioning                            & adaLN, width 128                    & Warmup                   & linear from 0, 1k steps
    \\
    Dropout                                      & 0                                   & Batch size               & 512
    \\
                                                 &                                     & Weight decay             & $0.05$, excl.\ norms/emb.
    \\
                                                 &                                     & Gradient clip            & 1
    \\
                                                 &                                     & EMA decay                & $0.9999$
    \\
                                                 &                                     & Window $[t_1,t_2]$       & $[0.01,0.9999]$
    \\
                                                 &                                     & Training seeds           & 3
    \\
    \midrule
    \multicolumn{2}{@{}l}{Likelihood evaluation} & \multicolumn{2}{l@{}}{Sampling}
    \\
    \cmidrule(r){1-2}\cmidrule(l){3-4}
    Weights                                      & EMA                                 & Samplers                 & ancestral, re-corrupt
    \\
    Split                                        & test                                & Instances                & $4{,}096$ test, one attempt each
    \\
    Batches                                      & 250 of 256 sequences                & Time grid                & uniform
    \\
    Units                                        & bits/token, full bound              & NFE                      & $\{8,16,\dots,512\}$
    \\
                                                 &                                     & Temperature, top-$p$     & 1, 1
    \\
    \bottomrule
  \end{tabular}

  \vspace{8pt}
  \setlength{\tabcolsep}{6pt}
  \begin{tabular}{@{}l cccc@{}}
    \toprule
                            & Sudoku    & Boxoban         & Maze              & Graph coloring \\
    \midrule
    Network                 & DiT       & DiT             & DiT, grid encoder & graph DiT      \\
    Parameters              & 42.7M     & 42.7M           & 61.6M             & 42.8M          \\
    Sequence length         & 81        & $100+128$       & $1{,}600+192$     & 64             \\
    Solution vocabulary     & 9         & 5               & 5                 & 4              \\
    Training instances      & 3.8M      & 936k            & 500k              & fresh draws    \\
    Learning rate, constant & $10^{-4}$ & $3\cdot10^{-4}$ & $10^{-4}$         & $10^{-4}$      \\
    Steps                   & 1M        & 2.5M            & 400k              & 300k           \\
    \bottomrule
  \end{tabular}
\end{table}

\subsection{Metric estimation details}\label{app:metric-details}

\paragraph{Likelihood evaluation.}
Reported NELBOs are the full finite-window bound~\eqref{eq:ctmc-elbo-full}: the path term plus the reconstruction term at $t_1$ and the prior term at $t_2$, taken against the support-matched uniform-forward prior of Proposition~\ref{prop:calibration}. Table~\ref{tab:boundary} reports how much the two endpoint terms contribute. Units are bits/token throughout; $\log V=10.82$ nats $=15.62$ bits at $V=50{,}257$.
All reported NELBOs use the EMA weights, except Figure~\ref{fig:training-temperature} (top), which plots the same quantity for the online weights at every validation step, including step 0.

\begin{table}[h]
  \caption{Size of the two boundary terms of the finite-window NELBO at the training window $[0.001,0.9999]$, in units of $10^{-3}$ bits/token. The prior term is a property of the noising process and does not depend on the head; the reconstruction term is the endpoint decode and does.}
  \label{tab:boundary}
  \centering\small
  \begin{tabular}{@{}l c ccc@{}}
    \toprule
                         & Prior KL & \multicolumn{3}{c}{Reconstruction}                   \\
    \cmidrule(lr){2-2}\cmidrule(l){3-5}
    Kernel               &          & den.                               & cav.   & score  \\
    \midrule
    Masked               & $1.562$  & =cav.                              & $1.66$ & $1.73$ \\
    GIDD $\lambda{=}0.1$ & $0.435$  & $2.63$                             & $1.96$ & $2.13$ \\
    GIDD $\lambda{=}0.2$ & $0.345$  & $3.12$                             & $2.09$ & $2.33$ \\
    Uniform              & $0.166$  & $5.19$                             & $2.67$ & $3.96$ \\
    \bottomrule
  \end{tabular}
\end{table}

\paragraph{Sample quality metrics.}
Generative perplexity (GenPPL) is the exponentiated mean per-token NLL of generated samples under GPT-2 Large \citep{radford2019gpt2}. Word entropy is the histogram entropy of the whitespace-split words of a sample, averaged over samples, the word-level form of the token-histogram entropy of \citet{zheng2025timeagnostic,sahoo2025diffusionduality}; the test set's value at $L=256$ is $4.625$ nats. MAUVE \citep{pillutla2021mauve} compares the sample and test distributions in a quantized GPT-2 Large feature space, with $2{,}048$ samples against $2{,}048$ held-out test sequences and the default bucket count $n/10\approx205$, averaged over five quantization seeds. Text samples are drawn at temperature $1$ and top-$p$ $1$ unless a sweep varies them, on a cosine grid in $t$. Every cell at the reference budget, NFE 256 at any temperature, natively or misread, is drawn at three sampling seeds of $2{,}048$ samples and reported as the seed mean; the NFE grid uses one seed. At the reference operating point the spread over the three sampling seeds is about $0.01$ in MAUVE and $1\%$ in GenPPL. The temperature sweep of Figure~\ref{fig:training-temperature} (bottom) covers temperatures from $0.5$ to $1.3$ at top-$p$ $1$.

\paragraph{Head agreement comparison.}
Two heads are compared on shared corrupted batches $(z_0,z_t)$ from the test split: both are converted to the denoiser coordinate and scored by the per-token Jensen--Shannon divergence $\tfrac12\big[\KLdiv{\modelC{\pi_i^A}}{\bar\pi}+\KLdiv{\modelC{\pi_i^B}}{\bar\pi}\big]$ with $\bar\pi=\tfrac12(\modelC{\pi_i^A}+\modelC{\pi_i^B})$, over the content vocabulary, in nats/token and bounded by $\log 2$. The statistic is averaged over the positions whose denoiser is not analytically a point mass (under masked noise, the masked ones) and either over $t$ drawn per sequence from the training time proposal over the full $51{,}200$ test sequences (Table~\ref{tab:agreement}) or at fixed $t$ over $6{,}400$ sequences per point (Figure~\ref{fig:agreement}). The misread entries are the same statistic between a run's converted $\pith_i$ and its own raw head read in the other coordinate.

\paragraph{Own-token sensitivity.}
The sensitivity columns of Table~\ref{tab:agreement} and Figure~\ref{fig:sensitivity} use the batches, seeds, time proposal, and grid of the head agreement comparison. After corrupting a batch to $z_t$, one position $i$ per sequence is redrawn from the forward kernel $\q^i_{t\given0}(\cdot\given z_0^i)$ at the same $t$, giving $z_t'$ that differs from $z_t$ at most at $i$; the statistic is the Jensen--Shannon divergence at $i$ between the head's law at $z_t$ and at $z_t'$, read in either coordinate through Proposition~\ref{prop:conversions}, averaged over the positions where the redraw changed the token. Masked diffusion is excluded: its denoiser equals the cavity at masked positions (Corollary~\ref{cor:masked}) and is a point mass on the observed token at unmasked ones, so the statistic is fixed analytically.

\subsection{Further experiments}\label{app:further}

Table~\ref{tab:results} lists the endpoint numbers of every run (\S\ref{subsec:trainability}): denoiser and cavity are close to each other in NELBO and overlap in GenPPL and MAUVE, while the score trails behind and has the highest GenPPL on every kernel, all of them at a similar word entropy (within $0.03$ nats per kernel).

\begin{table}[h]
  \caption{Endpoint numbers of the 11 runs at 500k steps: test NELBO, with GenPPL, word entropy, and MAUVE at the reference operating point.}
  \label{tab:results}
  \centering\small
  \setlength{\tabcolsep}{3pt}
  \begin{tabular}{@{}l ccc ccc ccc ccc@{}}
    \toprule
                         & \multicolumn{3}{c}{NELBO (bits/token)} & \multicolumn{3}{c}{GenPPL} & \multicolumn{3}{c}{Word entropy} & \multicolumn{3}{c}{MAUVE}                                                                                            \\
    \cmidrule(lr){2-4}\cmidrule(lr){5-7}\cmidrule(lr){8-10}\cmidrule(l){11-13}
    Kernel               & den.                                   & cav.                       & score                            & den.                      & cav.           & score & den.  & cav.  & score & den.           & cav.           & score \\
    \midrule
    Masked               & =cav.                                  & \textbf{4.900}             & 4.999                            & =cav.                     & \textbf{108.4} & 115.8 & =cav. & 4.720 & 4.713 & =cav.          & \textbf{0.869} & 0.866 \\
    GIDD $\lambda{=}0.1$ & \textbf{4.976}                         & 4.990                      & 5.115                            & \textbf{78.2}             & 78.8           & 91.0  & 4.696 & 4.688 & 4.704 & \textbf{0.891} & 0.884          & 0.870 \\
    GIDD $\lambda{=}0.2$ & \textbf{4.999}                         & 5.017                      & 5.132                            & \textbf{84.3}             & 86.7           & 107.2 & 4.700 & 4.697 & 4.717 & 0.883          & \textbf{0.896} & 0.860 \\
    Uniform              & 5.081                                  & \textbf{5.070}             & 5.257                            & \textbf{101.7}            & 105.0          & 152.8 & 4.712 & 4.714 & 4.738 & \textbf{0.887} & 0.873          & 0.814 \\
    \bottomrule
  \end{tabular}
\end{table}

\paragraph{Misreading in both directions.}
Table~\ref{tab:misread-full} extends Table~\ref{tab:misread} with GenPPL and the reverse misread, a denoiser head read as a cavity, which receives the Bayes reweighting a second time during sampling.

\begin{table}[h]
  \caption{Each run evaluated natively and with its head read in the wrong coordinate.}
  \label{tab:misread-full}
  \centering\small
  \setlength{\tabcolsep}{6pt}
  \begin{tabular}{@{}l ll cc cc cc@{}}
    \toprule
                         & Trained & Read & \multicolumn{2}{c}{NELBO (bits/token)} & \multicolumn{2}{c}{GenPPL} & \multicolumn{2}{c}{MAUVE}                              \\
    \cmidrule(lr){4-5}\cmidrule(lr){6-7}\cmidrule(l){8-9}
    Kernel               & as      & as   & Native                                 & Misread                    & Native                    & Misread & Native & Misread \\
    \midrule
    Masked               & cav.    & den. & 4.900                                  & =native                    & 108.4                     & =native & 0.869  & =native \\
    GIDD $\lambda{=}0.1$ & cav.    & den. & 4.990                                  & 12.416                     & 78.8                      & 3573    & 0.884  & 0.008   \\
    GIDD $\lambda{=}0.2$ & cav.    & den. & 5.017                                  & 12.289                     & 86.7                      & 3987    & 0.896  & 0.007   \\
    Uniform              & cav.    & den. & 5.070                                  & 11.788                     & 105.0                     & 2797    & 0.873  & 0.011   \\
    \midrule
    GIDD $\lambda{=}0.1$ & den.    & cav. & 4.976                                  & 6.568                      & 78.2                      & 218     & 0.891  & 0.777   \\
    GIDD $\lambda{=}0.2$ & den.    & cav. & 4.999                                  & 6.905                      & 84.3                      & 290     & 0.883  & 0.679   \\
    Uniform              & den.    & cav. & 5.081                                  & 8.155                      & 101.7                     & 677     & 0.887  & 0.266   \\
    \bottomrule
  \end{tabular}
\end{table}

\paragraph{Agreement in time.}
Figure~\ref{fig:agreement} resolves the agreement columns of Table~\ref{tab:agreement} in diffusion time.

\begin{figure}[h]
  \centering
  \includegraphics[width=\linewidth]{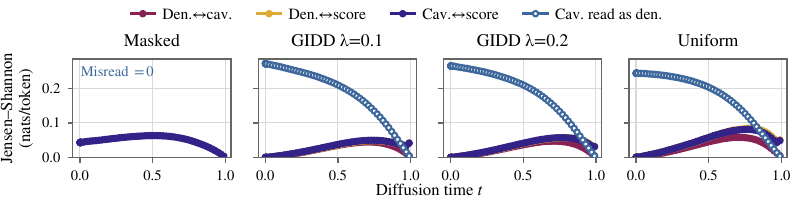}
  \caption{Pairwise Jensen--Shannon divergence between the independently trained heads of each kernel, converted to the denoiser coordinate, against diffusion time (filled markers). The open markers compare each cavity head against itself read as a denoiser without the conversion, a disagreement that grows as $t\to0$. Masked has no denoiser run.}
  \label{fig:agreement}
\end{figure}

\paragraph{Own-token sensitivity in time.}
Figure~\ref{fig:sensitivity} resolves the sensitivity columns of Table~\ref{tab:agreement} in diffusion time. In the denoiser coordinate the two heads are indistinguishable at every $t$: the denoiser's own-token dependence is the Bayes reweighting of the cavity's throughout the window. In the cavity coordinate the cavity head sits below the reweighting-corrected denoiser at every $t$, both roughly flat over the interior of the window. All four curves peak in the interior and vanish as $t\uparrow1$, where the context is noise and a redrawn token changes nothing; at the clean end a single redrawn token is outweighed by an almost clean context. What remains of a denoiser once the reweighting is divided out is about twice the cavity's own residual, so the denoiser learns the Bayes factor somewhat less exactly than the cavity learns to discard its token.

\begin{figure}[h]
  \centering
  \includegraphics[width=\linewidth]{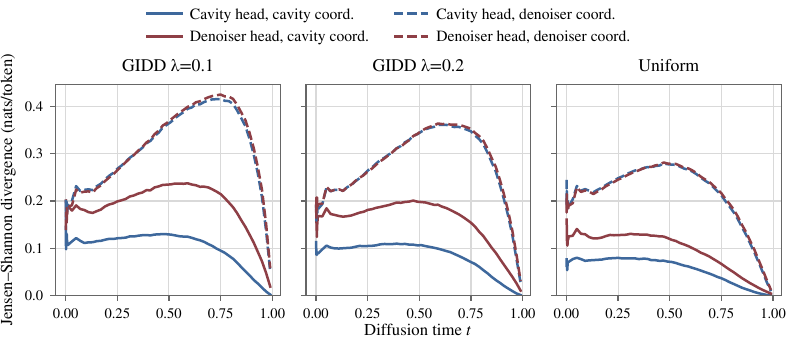}
  \caption{Own-token sensitivity against diffusion time: Jensen--Shannon divergence between a head's law before and after its noisy token is redrawn, for the cavity and denoiser heads read in the cavity coordinate (solid) and the denoiser coordinate (dashed).}
  \label{fig:sensitivity}
\end{figure}

\paragraph{Sampling sweeps.}
Figure~\ref{fig:frontier} traces the NFE sweep $\{16,\dots,1024\}$ at temperature 1 and top-$p$ 1 for every head, natively and read in the wrong coordinate: every misread series of a non-masked kernel is dominated by its native counterpart at every budget. Figure~\ref{fig:temperature-genppl} shows the temperature frontiers of Figure~\ref{fig:training-temperature} in GenPPL.

\begin{figure}[h]
  \centering
  \includegraphics[width=0.9\linewidth]{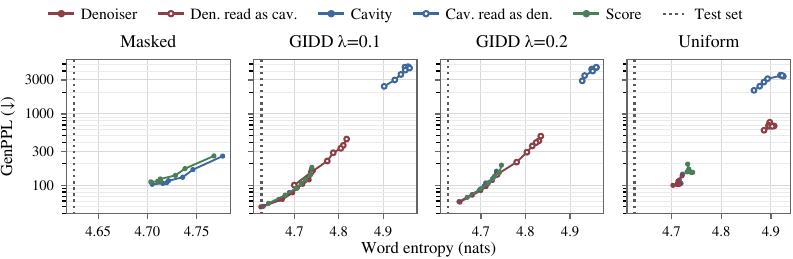}
  \caption{Quality--entropy frontiers traced by the NFE sweep $\{16,\dots,1024\}$: GenPPL under
    GPT-2 Large against word entropy, per kernel, for every head in its native coordinate (filled markers) and read in the
    wrong one (open markers); the dotted line is the test-set word entropy.
  }
  \label{fig:frontier}
\end{figure}

\begin{figure}[h]
  \centering
  \includegraphics[width=0.9\linewidth]{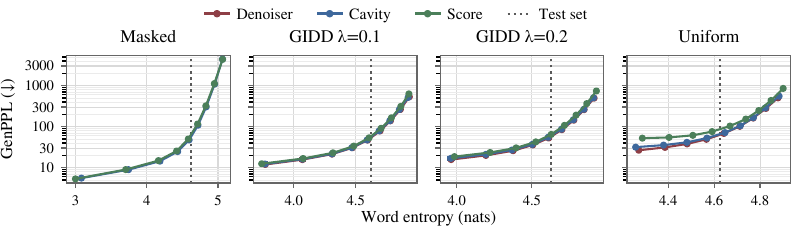}
  \caption{GenPPL counterpart of Figure~\ref{fig:training-temperature} (bottom): GenPPL under GPT-2 Large against word entropy along the temperature sweep at NFE 256, per kernel. The dotted line is the test-set word entropy.}
  \label{fig:temperature-genppl}
\end{figure}

\section{Extended related work}\label{app:related}

Discrete diffusion originates with structured categorical corruption \citep{hoogeboom2021argmax,austin2021d3pm}, the discrete analogue of Gaussian diffusion \citep{sohldickstein2015deep,ho2020denoising}, and now scales to large language models \citep{nie2025llada}.

\paragraph{Path-KL identities.}
For SDE diffusions, \citet{song2021maximum} bound the likelihood by a weighted score-matching objective whose gap is a Girsanov path KL, and \citet{huang2021variational} derive the continuous-time ELBO variationally. \citet{benton2024denoising} give a denoising framework and ELBO for general Markov processes; \citet{ren2025stochastic} develop a stochastic-integral calculus for discrete diffusion with CTMC change-of-measure theorems and error bounds for $\tau$-leaping; \citet{chen2024uniformization} prove sampling guarantees for uniformization with the score-estimation error entering as an assumed loss. Theorem~\ref{thm:oracle} uses the same $\Phi$ integrand as an exact decomposition of the ELBO, with the oracle reverse process as reference, for arbitrary noising processes, with the finite-window boundary form in Theorem~\ref{thm:oracle-window}. A tutorial treatment of diffusion in general state spaces that considers similar path-KL identities is \citet{pauline2025foundations}.

\paragraph{Weights and clocks.}
\citet{kingma2023understanding} show that common continuous-diffusion losses are ELBOs under noise-level distributions when the weighting is monotone; the clock equivalence of \S\ref{sec:oracle} is the CTMC counterpart, and separates unbiased importance densities from bound-changing reweightings \citep{nichol2021improved,shi2025demystifying,raya2026noise}. The masked cross-entropy at a uniform mask rate is ARDM's any-order objective \citep{hoogeboom2022ardm} and LLaDA's estimator \citep{nie2025llada}.

\paragraph{Masked diffusion.}
Simplified masked objectives are due to \citet{sahoo2024mdlm,shi2024simplified,zheng2023reparam}. \citet{zheng2025timeagnostic} show that masked training and sampling do not depend on the time variable. This is the flat-Bayes-factor collapse of Corollary~\ref{cor:masked} and~\eqref{eq:masked-elbo-ce}: for a fixed context the masked cavity does not depend on time. RADD \citep{ou2025radd} names the law of the clean token given the unmasked context, which does not depend on time, and derives the exact per-position reverse step from it.

\paragraph{Score-based discrete diffusion.}
Concrete scores originate with \citet{meng2022concrete,sun2022score}, and SEDD \citep{lou2024sedd} trains them by denoising score entropy, which is~\eqref{eq:sedd-score-entropy} with the forward rate as its weight. MD4's score parameterization enforces the constraint surface that a free score head can leave \citep{shi2024simplified}, which our conversions parameterize. Target Concrete Score Matching \citep{zhang2025tcsm} unifies objectives through the target score.
For uniform diffusion, \citet{sun2022score} parameterize the conditional of each noisy token given the others as the cavity pushed through the forward kernel, $\mth_i=\muth_i\,\q^i_{t\given0}$ in the notation of Proposition~\ref{prop:conversions}, read their ratios off it, and sample with the corresponding plug-in step.
They train it by ratio matching, a cross-entropy on each noisy token given the others, which a network that sees the token would minimize by copying it, so they have to exclude each position's own token with a hollow transformer.
The ELBO needs no such restriction: its minimizer is the cavity for every token-factorizing kernel, so the plug-in models, built on standard transformers, learn the cavity directly.

\paragraph{Plug-in lineage.}
D3PM \citep{austin2021d3pm} introduced the $x_0$-parameterization and the posterior-sum sampler, algebraically the plug-in step (\S\ref{sec:coordinates}); UDLM \citep{schiff2025udlm}, GIDD \citep{vonruette2025gidd}, Duo \citep{sahoo2025diffusionduality}, and the scaled non-masked models of \citet{sahoo2026scaling} train it on uniform and GIDD kernels, where its optimum is the cavity law. Both steps appear in D3PM's own formulation. Writing $\modelC{\tilde p^\theta}$ for its $x_0$-head, its reverse step is $\p^\theta(Z_s^i{=}y\given z_t)\propto\q^i_{t\given s}(z_t^i\given y)\sum_{z_0^i}\modelC{\tilde p^\theta}(z_0^i\given z_t)\,\q^i_{s\given0}(y\given z_0^i)$, an average of joints: its summand is $\q^i_{s\given0}(y\given z_0^i)\,\q^i_{t\given s}(z_t^i\given y)=\q^i_{t\given0}(z_t^i\given z_0^i)\,\q^i_{s\given t,0}(y\given z_t^i,z_0^i)$, the bridge weighted by the Bayes factor, so normalizing by $\mth_i(z_t^i\given z_t^{-i},t)$ recovers the cavity step~\eqref{eq:cavity-parameterization}; averaging the bridge alone, as their text describes, gives the denoiser step~\eqref{eq:denoiser-ancestral-step}. The two are introduced as one parameterization, yet they optimize different laws, equal only where the Bayes factor is flat (Corollary~\ref{cor:masked}), that is, on their absorbing kernels. Their auxiliary cross-entropy $\Exp{}{-\log \modelC{\tilde p^\theta}(z_0^i\given z_t)}$ therefore trains the denoiser while the plug-in loss and sampler assume the cavity, which would explain why it helps their absorbing models and reduces performance on their uniform and nearest-neighbor ones, the double application measured in \S\ref{subsec:misread}. \citet{gourevitch2026udm} are discussed in \S\ref{sec:related}.

\paragraph{Flows, generators, and other frameworks.}
Discrete flow matching \citep{gat2024discreteflow,campbell2024discreteflow} regresses probability velocities along prescribed corruption paths; \citet{shaul2025kinetic} decouple the path from the generating velocity and characterize kinetic-optimal choices, a freedom that in the likelihood picture appears as the coordinate and clock invariances of one optimal rate. Generator Matching \citep{holderrieth2025generator} regresses the posterior-averaged marginal generator with Bregman losses; \S\ref{sec:oracle} derives that average as the unique optimizer of the likelihood itself. \citet{amin2025scud} explain the empirical differences between kernels through conditioning on the jump schedule, complementary to the coordinate mechanism here and to the shared floor of Corollary~\ref{cor:floor}.

\paragraph{Information-theoretic identities.}
I-MMSE ties information to estimation in the Gaussian channel \citep{guo2005mmse}; its Poisson analogue has $\Phi$ as the natural loss \citep{atar2012poisson}; \citet{kong2023infodiffusion} write the continuous diffusion likelihood as an exact integral of MMSE, decomposed per variable by \citet{kong2024interpretable}. The discrete I-MDSE/I-MDCE identities of \citet{jeon2025infodiscrete} state the oracle case with $T\to\infty$ in the score and masked cross-entropy coordinates; Proposition~\ref{prop:infoloss} with~\eqref{eq:pythagoras} extends them to arbitrary model rates, finite windows with boundary terms, and every coordinate. Conditional-mean optimality and the Pythagorean identity for Bregman divergences are due to \citet{banerjee2005bregman}, with the exponential-family geometry in \citet{banerjee2005clustering}; \citet{vincent2011connection} is the continuous denoiser/score identity our conversions mirror.

\section{Generated samples}\label{app:samples}

One sample from each run, drawn uniformly at random from the 2{,}048 generated, from the 500k-step EMA checkpoints, sampled with 256 NFEs at temperature $0.9$.

\paragraph{Masked, cavity.}
\begin{quote}\small
  told his driver to file a report accordingly. An arrest and from a prosecutor's office in court heard Mr Mitchell could go to trial next month for commingling relating to the raid. In the programme, one revelation is that Mr Begg known and realised his wrongdoing after he was told he could not get out of his positions. "He's the one covering it up with the information," he told the court. But on 7 October a 21-year-old Begg man was arrested. The case was closed when for a time there was "shame and sorrow". There are rumours Begg lied to investigators and witnesses. "He has no influence on the case from a rented home in Cambridgeshire," said Assistant Information Commissioner Sebastian Aldatts. The prosecutors were involved in a recent High Court-flung case hearing for DIA as it investigated a British-made helicopter in September 2012. It was the first military helicopter crashed on UK-controlled territory and since last year the British air force was operating in Afghanistan. It was the first time a helicopter had been used in a military attack in New Zealand. After the 2012 raid, the British Air Force launched its
\end{quote}

\paragraph{Masked, score.}
\begin{quote}\small
  been linked to a dramatic penalty goal near QPR Park on Sunday during the campaign in which they, who miss the players in the second round, were impressed by their side When Rodgers aired his bad play on Sports 4 TV when the game was in July, he said a ban kept him from the club. But it is understood QPR are exceeding the £48 million budget while City are not. The club are concerned and Rodgers is under the orders of Chris Wilson. Finishing with his role, Rodgers responded to the media: "I have 30 days to be there in training, I have two days to train and you can't second what I have done. "What's the problem? We are not getting Tom's cut. It's his own fault. We have a terrible line that gets him goal but that is not all it is. We need to tell you that you need to do a lot of things, you've got some cap space and we learn. "I move like this so it is very important for me as a player, we learn. "If Tom is training then he will be doing well ahead of us. "The most we ever spend
\end{quote}

\paragraph{GIDD $\lambda{=}0.1$, denoiser.}
\begin{quote}\small
  or an incurable infection. The serotonin role is what sticks around to wake you and won't get to you. When you feel sick, you have limited information. That mechanism requires the release of the strong neurotransmitter serotonin stored in your prefrontal cortex to help you survive medical emergencies. You actually want it to go down on top of you, to talk to your brain. “It doesn't necessarily cause things like allergies to deteriorate and fall out,” Dodging said. But compared with the alcohol or crack cocaine you take and makes your brain nervous mate, “this doesn't jeopardize the science or anything. You don't even have a scare by going after serotonin while you are suffering. The studies that have linked it would have said they risk it. It does lock up some of the protein, too, to process this. And some of what was supposed to happen to serotonin, did not happen.” Helmet tablets, along with the new drugs, have a particular place for the internal nervous system. Things don't already work, he said. Going to use it, however, has a few limitations. First, several treatments are likely to be done before
\end{quote}

\paragraph{GIDD $\lambda{=}0.1$, cavity.}
\begin{quote}\small
  , so I worked my way around, work my way back around" --- Architectural Abstract Architecture What did you think of Airdropels and Airdropism? The concept of Airdropels is simple. In fact, they are the form of New Technologies technologies and, to a certain extent, the development of the form of government, as it is a new tool for manipulation and deception. They will rebuild and rebuild to strengthen their position, ultimately, they will change the way they vanquished causes of conflict, but they also will try to create rules, and if the rules create perpetual changes, then the they will. They say, “There is any history of time spent change and you've got the government spending trillions of dollars trying to avoid it. It happens because big corporations are creating laws. Laws that the U.S. Congress is coming up, which have nothing to do with being a government, because I believe the values encoded in the Constitution that we must do things properly, and that we are subject to principle, and history is a part of it. That makes me happy that as a civilization, Philosophy of Freedom is at the heart and is trying to lead us in the direction of this direction.
\end{quote}

\paragraph{GIDD $\lambda{=}0.1$, score.}
\begin{quote}\small
  and the more liberal crowd would be too concerned about fair/fast, clean/money-action/full investment in criminal justice reform. When Sanders made it on stage, it was clear that the tactics of his campaign weren't going to work, and he was just sitting in the bus by the state line. And rather than accusing Clinton and Kevin Spacey, Sanders's Democratic friend, which had raised complaints about him, Clinton "could have let it all down like a bird" so she would stand by Clinton's "less policies " and that "don't have the money to keep our words off the air." Still, Sanders was prepared to engage with Clinton and articulate her message. And he feels he can define it better in his speeches Saturday night. "Yes, I'm on the side of PIPA," he said, "but I represent America. " " He will have much more to offer from Clinton's campaign money and celebrity money, with another governor, donors morale and opportunity to continue the business with his political base. Look, she is a passionate charity fighter who's got the kind of support Obama had in the campaign process, working quietly just a few months a year, and she certainly has accrued to the
\end{quote}

\paragraph{GIDD $\lambda{=}0.2$, denoiser.}
\begin{quote}\small
  City. He'd secured himself a contract with LeBron and as we've seen before had dramatically increased his numbers for the Thunder. In at 8-1 against Detroit though he had missed just three straight games to keep his teammates out absent injury. Neither Westbrook nor Kyle Orton even had a factor that game provided he was not on the floor, he would have been playing better than Oklahoma City, with 18 of 25 teams on their schedule beating Westbrook's suffocating game by 18 points. I can't pick these years, every one was that or even a half the work of LeBron. What didn't really care or much really care was the Thunder team went to 2-0 and did not disappoint in 2-0, earning two wins and winning a total of 3 games. What was there? One evening after the performance, leading to the 40-point TD/flash in Houston it's 7-0 opportunity to beat the team. All in all that they got the win and Dris beat them, so much like boo-ha and then “ha” beats. When LeBron, being rubbed out at the media, told me unprompted the media he paid attention to the rest
\end{quote}

\paragraph{GIDD $\lambda{=}0.2$, cavity.}
\begin{quote}\small
  fully inspired, and a large audience like Long has more of that attention. "If it wasn't for Gears of War I can't imagine how many people say that those games will change gaming," he says. "We all watch great videogames, and one person goes, 'Well, then isn't the time.' But I think we'll never see all the heck of it." Sunday's video's rebrings on spoilers from the film, but you'll learn more about Wes' relationship from the wiki that reads the comic for him. It's just a fun video with just a pop-up in tone and lots of detail in how things are done, and how the art can be turned into a standalone MMO campaign for whatever purpose. I initially wasn't exactly excited with the idea of Wes as a JRPG...to-the-day. They were quite right, though I wasn't sure the visuals would be perfect (or at the very least, I still recall using a coworker's sports figure), but, what did you build on Wes' return as a nerd? The mix of playable animations with dialogue (I have not even played a dialogue for the sidequest game for that matter. There needs to be more and
\end{quote}

\paragraph{GIDD $\lambda{=}0.2$, score.}
\begin{quote}\small
  . Times on Friday, with concern about "the guidance put in place to it easier for regulators to apply what is benefit gains to investors in the region." It said, for instance, the ICTG is willing to offer at least \$30,000 shares to low-cost investors going back two years, and that the federal government would have an eager hand in exchange for ICTG's assurance that it could buy a franchise in a foreign country for \$18. The newspaper also said: "We demanded that this happen. We can't blame them. It could end in bad views for us. We do not believe the regulators have given it their clear instructions about the rigging. Mr. Smith, who was contacted by Mr. Baker, told The Washington Post on Saturday that ICTG plans to open more coal mines in the Middle East than in the years to March 1, so people might try the black market, rather than retail. "I think it's such a complicated issue, we have to make a change, and... again, I'm very concerned about what it is," he said. He cited The Huffington magazine's Crispin Pelt, and his work with Justice Johnson, to be the source. Despite the concerns, he told Fox News,
\end{quote}

\paragraph{Uniform, denoiser.}
\begin{quote}\small
  of every year Johnny Manziel's coaching job along with several other 49ers, Saints, and Broncos players, it's a very intricate thought process. In that one thing, Ryan and Murray even left it all too late in another direction as well: So, we are yapping today and are focusing on the topics these next ten days. Find your team on the field we're doing, and you can head back to practice field when Coach Murray posts an announcement via Facebook. Every now, as we take the schedule bashing at our best, please CONNECT! For the holiday, Being in a New Year, here is a special series on practice fields, focusing again. Without further ado, we evaluate the top three teams today. We chose the following preseason schedule: the first NFL Combine, held in Iowa, the first, but not the nation's great Combine, in May, and the initial Pacific Player of the Year in June. “Um, uh, okay? Really? Coach?!" Manziel nodded at me with a freaking grin and said, “Fuck that!” So freaking awesome! We picked the first NFL-firing Cedar Falls on July 1. The NFL
\end{quote}

\paragraph{Uniform, cavity.}
\begin{quote}\small
  's movie studio() adopted baselina analog features in addition to converting to MPEG format as direct-to-feedback, but they also built heavily on the analysis of original time frames in the 4D digital format. This took off as a wise decision, as the "green" philosophy of the time was logical enough to suppress some original episodes from the making. In an example, the 3D disc featured four ;mins of broadcasts from the viewing ramp on morning 1, while someone listened to Marshall Brownock's Morning 2 (middle door). Just what did this pattern mean? For the most part, I remember Scott: due to the escalating power of projector technology one at a time, Scott vividly took a CD, read dozens of photocopies, recorded by Joel Entrance, took the tape into his revisit cabinet, and then listened to phonos from before he passed away. The use of abeforms, to separated reads, gave these an ex post facto intellectual continuity over many clips from different shows in the movie. In 1987, TMI was adapted to a multimedia format In 2004, after discussions with Tim Nachows and Stephen North, the Edison Broadcasting Corporation (ICC) decided to adopt the MPEG/P/TTV format as a
\end{quote}

\paragraph{Uniform, score.}
\begin{quote}\small
  after a year (AC4/1199, the APSS Design) and for select users, like the Sony Xperia A10 and the LG A township Shoot, you is the number two RD battery/hot. 15 THE R3 1900 The United States R3, which starts with a Galaxybird (US for the arc phone and desktop device) cheap and also is actually is compatible with the Ace as well well as a 1028 and an SDATA/STAT port. 16 THE SARM (8 years) In case you pond, I will be reading for you to ask to see this ride back. To my knowledge, my request increased (Monica Perry declined to provide this one) to S co (Earthshake, only to be issued) at the fee. 19 Thunderbacks --- presented by Pat Colby at Sheffield University this week. Talk about therapies endorsed by the medical, scientific professions --- and just that, go far back and time to present. This week I'll follow along some of the topics of one of these four experiments hosted by the Crossroads's Independent Submarine/FCC. Several of the parameters examined are 3 days during the course of
\end{quote}

\end{document}